\documentclass[11pt,letterpaper]{article}

\PassOptionsToPackage{authoryear,round,semicolon}{natbib}
\PassOptionsToPackage{shortlabels,inline}{enumitem}
\PassOptionsToPackage{hidelinks}{hyperref}
\PassOptionsToPackage{dvipsnames}{xcolor}

\usepackage[letterpaper,margin=1in]{geometry}
\usepackage{parskip}
\usepackage{txie}

\definecolor{brilliantrose}{rgb}{1.0, 0.33, 0.64}
\definecolor{Klein}{HTML}{002FA7}
\colorlet{txblue}{RoyalBlue!70!NavyBlue}

\usepackage[utf8]{inputenc}
\usepackage[T1]{fontenc}
\usepackage{microtype}
\usepackage{booktabs}
\usepackage{url}

\allowdisplaybreaks

\hypersetup{pdfauthor={}, pdftitle={Outcome-based Offline Reinforcement Learning}}

\Crefname{theorem}{Theorem}{Theorems}
\Crefname{lemma}{Lemma}{Lemmas}
\Crefname{proposition}{Proposition}{Propositions}
\Crefname{corollary}{Corollary}{Corollaries}
\Crefname{assumption}{Assumption}{Assumptions}
\Crefformat{equation}{#2Eq.\;(#1)#3}
\Crefformat{figure}{#2Figure #1#3}
\Crefformat{assumption}{#2Assumption #1#3}

\title{When Does Trajectory-Level Supervision Permit Efficient Offline Reinforcement Learning?}

\author{%
  \begin{tabular}[t]{c}
    Xuanfei Ren \\
    University of Wisconsin-Madison \\
    Madison, WI \\
    \texttt{xuanfeiren@cs.wisc.edu}
  \end{tabular}
  \hspace{2em}
  \begin{tabular}[t]{c}
    Tengyang Xie\thanks{Corresponding author.} \\
    University of Wisconsin-Madison \\
    Madison, WI \\
    \texttt{tx@cs.wisc.edu}
  \end{tabular}
}
\date{}

\begin{document}

\maketitle

\begin{abstract}
  Offline reinforcement learning is typically analyzed under process-level reward
  supervision, yet many sequential decision datasets record only trajectory-level
  outcomes. We develop a statistical theory for offline policy optimization from
  such outcome-level supervision. We first study the canonical setting where the
  target remains the expected cumulative reward, but each offline trajectory
  provides only a scalar label whose conditional mean is the cumulative return.
  We propose OPAC, a pessimistic actor-critic algorithm that learns a latent
  reward model and optimizes a policy from trajectory-level labels. We prove a
  high-probability guarantee of order
  $\widetilde O(H^2\sqrt{C_{sa}(\pi^\star)/n})$ and a matching lower bound,
  characterizing the sharp statistical cost of replacing process-level rewards
  with one trajectory-level label. We then extend the principle to
  preference-based feedback, preserving the leading horizon and concentrability
  dependence up to preference-model constants. Finally, we study generalized
  outcome-based offline RL, where both the supervision and the objective are
  trajectory-level quantities induced by a nonlinear aggregation of latent
  per-step rewards. This problem is not learnable in general: for all-success
  objectives, any offline learner may require $\Omega(2^H)$ trajectories even
  with deterministic transitions and constant concentrability. We then identify a
  tractable regime through two structural coefficients, $\kappa_\mu(\sigma)$ and
  $\chi_\mu(\sigma)$, capturing information loss in outcome aggregation and
  generalized Bellman updates, under which generalized OPAC achieves polynomial
  sample complexity. Together, our results delineate when outcome-level
  supervision enables sample-efficient offline control and when missing
  process-level rewards create fundamental statistical barriers.
  \end{abstract}

\section{Introduction}
Offline reinforcement learning asks whether a good policy can be learned from a fixed dataset of past interactions \citep{lange2012batch,levine2020offline,jiang2025offline}, with a rich line of empirical algorithms developed in deep RL \citep[e.g.,][]{fujimoto2019off,kumar2020conservative,fujimoto2021minimalist,kostrikov2021offline}. Most theory assumes that each trajectory is annotated with process-level rewards \citep{agarwal2019reinforcement,foster2023foundations}: after every action, the learner observes the reward assigned to that state-action pair. This assumption gives direct access to Bellman targets, but it is often stronger than the supervision available in long-horizon applications. A dataset may record only whether a task eventually succeeded \citep{andrychowicz2017hindsight,lightman2023let}, what final score was achieved, whether a proof was accepted \citep{uesato2022solving}, whether a patient recovered \citep{komorowski2018artificial,gottesman2019guidelines}, or which of two completed trajectories a human preferred \citep{christiano2017deep}. The feedback is trajectory-level, while the decisions remain sequential \citep{jia2025we,chen2025outcome}.

This paper develops a statistical theory for offline policy optimization from
outcome-level supervision. The central question is not only whether such
supervision is sufficient, but also where the statistical difficulty comes from:
distribution shift in offline control \citep{xie2021bellman,cheng2022adversarially,jiang2025offline},
recovering latent per-step rewards from aggregated labels
\citep{jia2025we,chen2025outcome}, or the trajectory-level objective itself
\citep{zhang2020variational}. We show that these sources of difficulty lead to
qualitatively different regimes.

We begin with the regime closest to classical offline RL. The target objective is  the usual expected cumulative reward
  $J(\pi)=\EE_{\tau\sim\pi}   \left[\sum_{h=1}^H r^\star(s_h,a_h)\right]$,
but the offline data do not reveal the per-step rewards. Instead, each sampled trajectory is labeled by a scalar outcome whose conditional mean is the cumulative return. Thus the objective is unchanged and only the observation model is weakened. This isolates a basic statistical question:
\begin{center}
    \emph{What is the cost of replacing $H$ local reward observations by one trajectory-level label?}
\end{center}

We answer this question with matching upper and lower bounds, up to logarithmic factors. Our algorithm, OPAC, is a pessimistic actor-critic method that jointly learns a latent per-step reward and a value function. The reward model is fit by trajectory-level regression, the critic is constrained through a plug-in Bellman error, and the policy update uses pessimism to account for distribution shift in the offline data \citep{xie2021bellman}. Under standard realizability and completeness assumptions, for bounded per-step rewards and bounded trajectory outcomes, the output policy satisfies
  $J(\pi^\star)-J(\bar\pi)
  =
  \widetilde O   \left(H^2\sqrt{{C_{sa}(\pi^\star)}/{n}}\right)$,
or equivalently needs $\widetilde O(H^4C_{sa}(\pi^\star)/\varepsilon^2)$ trajectories for $\varepsilon$-optimality (\Cref{theorem:outcome-based-offline-rl}). We also prove an $\Omega(H^4/\varepsilon^2)$ trajectory lower bound at constant state-action concentrability (\Cref{thm:outcome_based_lower_bound}). The hard instance has deterministic transitions and only two actions, so the additional factor relative to process-level reward supervision is not caused by exploration or transition estimation; it comes from compressing reward information into a single scalar outcome.

We then consider a weaker but practically important form of trajectory-level
feedback: pairwise preferences. The learner observes which of two trajectories is
preferred, generated by a Bradley--Terry--Luce model
\citep{bradley1952rank,christiano2017deep}, while the optimization target
remains the same cumulative return. The same algorithmic template applies after
replacing squared trajectory-return regression with logistic preference
regression. The resulting guarantee preserves the leading
$H^2\sqrt{C_{sa}(\pi^\star)/n}$ dependence, up to preference model constants 
(\Cref{thm:pref-outcome-based-main}). Thus, for the standard cumulative-reward
objective, calibrated scalar outcomes are not essential: sufficiently
informative comparisons can support the same pessimistic offline-control
mechanism.

The classical cumulative reward, however, is only one way trajectory-level
feedback can arise. In many problems, the observed outcome is not merely a noisy
proxy for cumulative reward, but the performance criterion itself. This motivates
a generalized RL formulation in which the environment still admits latent
per-step reward primitives, while policies are evaluated through a known
trajectory-level aggregation rule. Concretely, for any candidate reward process
$r=(r_1,\ldots,r_H)$, define
$R(\tau;r)=\sigma(r_1(s_1,a_1),\ldots,r_H(s_H,a_H))$, where
$\sigma:\RR^H\to[0,V_{\max}]$ is known and may be nonlinear. The learning
objective is
\begin{align*}
  J(\pi)=\EE_{\tau\sim\pi}[
  \sigma(r^\star_1(s_1,a_1),\ldots,r^\star_H(s_H,a_H))].
\end{align*}
This formulation is already a meaningful control problem before considering
trajectory-level supervision \citep{zhang2020variational,barakat2023reinforcement}. The agent still makes sequential decisions, and the
environment admits local reward primitives, but the desired behavior may
depend on how these primitives combine over the full trajectory. Such criteria
arise naturally in long-horizon tasks, where forcing the objective into a
cumulative reward can change the quantity one intends to optimize.

The outcome-based version is especially important because many datasets record
this trajectory-level criterion directly. In this setting, the learner observes
outcomes $Y$ satisfying $\EE[Y\mid\tau]=R(\tau;r^\star)$, rather than the latent
process rewards. Thus, the trajectory outcome is both the supervision signal and
the optimization target. This makes the formulation more faithful to many real
datasets, but also statistically more delicate: the learner must optimize a
generalized objective from generalized trajectory-level supervision. This leads
to a second, sharper question:
\begin{center}
  \emph{Is efficient learning still possible when both supervision and the objective are generalized?}
\end{center}

In general, the answer is no. The clearest example is the all-success objective,
where rewards are binary and $R(\tau;r)=\prod_{h=1}^H r_h(s_h,a_h)$, so a
trajectory has value one only if every stage succeeds. We prove that, for this
objective, any offline learner may require $\Omega(2^H)$ trajectories to obtain
nontrivial performance, even with deterministic transitions and constant state-action
concentrability (\Cref{thm:main-lb-allsuccess}).

This impossibility points to two distinct bottlenecks. The first is outcome
aggregation: observing only $R(\tau;r^\star)$, a nonlinear $\sigma$ may map
different per-step reward processes to nearly indistinguishable trajectory
outcomes. We quantify this inverse problem by the Reward Process Coefficient
$\kappa_\mu(\sigma)$ (\Cref{def:kappa}), which measures how well reward-process
differences are preserved after aggregation.
The second bottleneck comes from the generalized objective itself. To obtain a
tractable theory, we restrict attention to aggregations that admit Bellman-style
dynamic programming. Even then, generalized Bellman targets may further compress
reward differences when forming one-step targets. We quantify this effect by the
Bellman Inverse Coefficient $\chi_\mu(\sigma)$ (\Cref{def:chi}), which measures
how well generalized Bellman targets preserve reward-process differences.
With these two measures, we identify a tractable regime. Under structural
assumptions, Generalized OPAC achieves  suboptimality
$\widetilde O\bigl(
V_{\max}^2L\sqrt{\kappa_\mu(\sigma)H^2C_{sa}(\pi^\star)/n}
+
V_{\max}^2L\sqrt{\chi_\mu(\sigma)H^4/n}
\bigr)$ up to lower-order terms 
(\Cref{thm:main-gen-outcome}). Thus the rate remains polynomial in the horizon,
sample size, and offline concentrability, while $\kappa_\mu(\sigma)$ and
$\chi_\mu(\sigma)$ track the two information-loss mechanisms above.

\paragraph{Our contributions}
We study offline RL from trajectory-level supervision. First, for the standard
cumulative-reward objective with scalar trajectory outcomes, we propose OPAC and
prove matching upper and lower bounds up to logarithmic factors
(\Cref{theorem:outcome-based-offline-rl,thm:outcome_based_lower_bound}),
characterizing the cost of replacing step-level rewards by trajectory-level
supervision. Second, we extend OPAC to trajectory preferences, with
guarantees matching the leading horizon and concentrability dependence of the
scalar-outcome setting (\Cref{thm:pref-outcome-based-main}). Third, to the best
of our knowledge, we are the first to study offline RL where both the supervision
and the objective are generalized trajectory-level quantities: the observed
outcome is a known aggregation of all stage-wise rewards, and the learning goal
is to optimize this same quantity. In this setting, we prove an exponential
impossibility result (\Cref{thm:main-lb-allsuccess}) and identify a learnable
regime through $\kappa_\mu(\sigma)$ and $\chi_\mu(\sigma)$, under which
Generalized OPAC achieves polynomial sample complexity
(\Cref{thm:main-gen-outcome}).

\section{Preliminary}\label{section:preliminary}

\paragraph{Finite-horizon MDPs, policies, and occupancy}
A finite-horizon Markov decision process is specified by
$M=(\cS,\cA,P,r^\star,H,s_1)$, where:
(i) $H\in\mathbb N_{\ge 1}$ is the horizon;
(ii) $\cS=\bigsqcup_{h=1}^H\cS_h$ and
$\cA=\bigsqcup_{h=1}^H\cA_h$ are layered state and action spaces;
(iii) $P=(P_h)_{h=1}^{H-1}$, with
$P_h:\cS_h\times\cA_h\to\Delta(\cS_{h+1})$, is the transition kernel;
(iv) $r^\star=(r_h^\star)_{h=1}^H$ is the unknown mean process reward, with
$r_h^\star:\cS_h\times\cA_h\to[0,1]$; and
(v) $s_1\in\cS_1$ is a fixed initial state.
A trajectory is
$\tau=(s_1,a_1,\dots,s_H,a_H)
\in\cS_1\times\cA_1\times\cdots\times\cS_H\times\cA_H$.
The reward function is part of the environment and defines the optimization
objective, but in our outcome-based setting the process-level rewards are not
observed in the offline data.
A Markov policy is $\pi=(\pi_h)_{h=1}^H$, with
$\pi_h:\cS_h\to\Delta(\cA_h)$. We fix a policy class $\Pi$ and denote the
behavior policy used to collect data by $\mu\in\Pi$. For any $\pi\in\Pi$, the
step-$h$ state--action occupancy is $d_h^\pi(s,a)
\coloneqq
\PP_{\tau\sim\pi}\bigl[(s_h,a_h)=(s,a)\bigr]$,
where $\tau\sim\pi$ denotes a rollout of $\pi$ from $s_1$ under $P$. We write
$\EE_\pi[\cdot]=\EE_{\tau\sim\pi}[\cdot]$ and
$\EE_\mu[\cdot]=\EE_{\tau\sim\mu}[\cdot]$, and denote the step-$h$ occupancy of
$\mu$ by $d_h^\mu$.

\paragraph{Cumulative rewards, values, objectives, and function classes}
For any candidate mean process reward $r=(r_h)_{h=1}^H$, define the cumulative
 reward $ R(\tau;r)
 \coloneqq
 \sum_{h=1}^H r_h(s_h,a_h)$ and the step-$h$ $Q$-function under $\pi$ by $Q_h^{\pi,r}(s,a)
 \coloneqq
 \EE_\pi\left[
   \sum_{h'=h}^H r_{h'}(s_{h'},a_{h'})
   \,\Big|\,
   (s_h,a_h)=(s,a)
 \right]$,
and let
$V_h^{\pi,r}(s)\coloneqq
\EE_{a\sim\pi_h(\cdot\mid s)}[Q_h^{\pi,r}(s,a)]$.
When $r=r^\star$, we omit it from the superscript and write
$Q_h^\pi,V_h^\pi$, and set $R^\star(\tau)\coloneqq R(\tau;r^\star)$.
The policy value is $
  J(\pi)
  \coloneqq
  \EE_\pi[R^\star(\tau)]
  =
  V_1^\pi(s_1)$,
and the in-class optimal comparator is
$\pi^\star\in\argmax_{\pi\in\Pi}J(\pi)$.
We are given a per-step value-function class
$\cF=\cF_1\times\cdots\times\cF_H$, where
$\cF_h\subseteq\{f:\cS_h\times\cA_h\to[0,V_{\max}]\}$, and a mean
process-reward class
$\cR\subseteq\{r=(r_h)_{h=1}^H:r_h:\cS_h\times\cA_h\to[0,1]\}$.
We identify $f\in\cF$ with the tuple $(f_1,\dots,f_H)$ and, since
$\cS_h\cap\cS_{h'}=\emptyset$, also write $f(s,a)$ for $f_h(s,a)$ whenever
$s\in\cS_h$.

\paragraph{Bellman operator}
For a reward $r=(r_h)_{h=1}^H$ and policy $\pi$, the Bellman operator $\cT_r^\pi$ acts on $f:\cS\times\cA\to\RR$ by
\begin{align}\label{eq:def-bellman-op}
  (\cT_r^\pi f)(s,a)   \coloneqq   r_h(s,a)+\EE_{s'\sim P_h(\cdot\mid s,a)}   \bigl[f_{h+1}(s',\pi)\bigr],\qquad (s,a)\in\cS_h\times\cA_h,
\end{align}
with the boundary convention $f_{H+1}\equiv 0$, and we use the shorthand $f_h(s,\pi)\coloneqq\EE_{a\sim\pi_h(\cdot\mid s)}[f_h(s,a)]$. When the reward subscript is omitted, $\cT^\pi\coloneqq\cT_{r^\star}^\pi$.

\paragraph{Offline data and coverage}
Unless stated otherwise, $\cD=\{(\tau^{(i)},Y^{(i)})\}_{i=1}^{n}$ denotes an
offline dataset of $n$ i.i.d.\ trajectories $\tau^{(i)}\sim\mu$, each annotated
with a trajectory-level outcome $Y^{(i)}$ satisfying
$\EE[Y^{(i)}\mid \tau^{(i)}]=R^\star(\tau^{(i)})
=\sum_{h=1}^H r_h^\star(s_h^{(i)},a_h^{(i)})$. The precise observation model for
$Y^{(i)}$ depends on the section. We write
$\EE_\cD[\phi(\tau,Y)]\coloneqq n^{-1}\sum_{i=1}^n\phi(\tau^{(i)},Y^{(i)})$ for
empirical averages.

The all-step state--action concentrability
\citep{chen2019information,xie2021bellman} of $\pi\in\Pi$ with respect to $\mu$
is
\begin{align}\label{eq:def-Csa}
  C_{sa}(\pi)
  \coloneqq
  \max_{h\in[H]}
  \sup_{(s,a)\in\cS_h\times\cA_h}
  \frac{d_h^\pi(s,a)}{d_h^\mu(s,a)},
\end{align}
and $C_{sa}$ is always computed with respect to the fixed behavior policy $\mu$.
This density-ratio condition is a standard coverage assumption in offline RL
\citep{munos2007performance,farahmand2010error,chen2019information}. It could
be relaxed to function-class-dependent coverage notions, where the density ratio
is tested only against functions in the relevant class
\citep{duan2020minimax,xie2021bellman,song2022hybrid}. Since this paper focuses
on the granularity of reward supervision rather than coverage assumptions, we
use the classical density-ratio formulation for simplicity.

\paragraph{Asymptotic notation}
We use $O(\cdot)$, $\Omega(\cdot)$, and $\Theta(\cdot)$ for standard asymptotic bounds up to universal numerical constants. The notations $\widetilde O(\cdot)$, $\widetilde\Omega(\cdot)$, and $\widetilde\Theta(\cdot)$ additionally hide factors polylogarithmic in the problem parameters, such as $n,H,|\cA|,|\Pi|,|\cF|,|\cR|$, and $1/\delta$. Finally, $a\lesssim b$ means $a\le Cb$ for a universal constant $C>0$, and $a\gtrsim b$ is defined analogously.

\section{Sample-Efficient Offline RL with Outcome Reward}\label{section:algorithms}

We work in the offline framework of \Cref{section:preliminary}, specializing to
the case where each trajectory carries a single scalar label rather than
process-level rewards, while the optimization target remains the standard
expected cumulative return under $r^\star$. Thus, the
learner observes only one unbiased cumulative-reward outcome after the trajectory
terminates; it does not observe the reward vector
$(r_1^\star(s_1,a_1),\ldots,r_H^\star(s_H,a_H))$ or any per-step reward labels.
With $r_h^\star\in[0,1]$, the latent cumulative return lies
in $[0,H]$. In this section, we further assume $Y^{(i)}\in[0,H]$ and set
$V_{\max}=H$.

\subsection{Algorithm}

We define empirical criteria for policy mismatch, Bellman consistency under a
candidate mean reward, and regression to scalar trajectory outcomes. Our proposed algorithm OPAC then
alternates pessimistic evaluation with layer-wise exponential-weights policy
updates.

\paragraph{Empirical criteria}
Define $y_h^{r,f,\pi}\coloneqq r_h(s_h,a_h)+f_{h+1}(s_{h+1},\pi)$ with the
convention $f_{H+1}=0$, and
\begin{subequations}\label{eq:all_losses_emp}
  \begin{align}
    \cL_\cD(\pi,f)
    &\coloneqq
    \sum_{h=1}^{H}
    \EE_\cD\!\left[
      f_h(s_h,\pi)-f_h(s_h,a_h)
    \right],
    \label{eq:loss_policy_emp}
    \\
    \cL_\cD^{\mathrm{BE}}(\pi,r,f)
    &\coloneqq
    \sum_{h=1}^{H}
    \left\{
    \EE_\cD\!\left[
      \bigl(f_h(s_h,a_h)-y_h^{r,f,\pi}\bigr)^2
    \right]
    -
    \min_{g\in\cF_h}
    \EE_\cD\!\left[
      \bigl(g(s_h,a_h)-y_h^{r,f,\pi}\bigr)^2
    \right]
    \right\},
    \label{eq:loss_bellman_emp}
    \\
    \cL_\cD^{\mathrm{RM}}(r)
    &\coloneqq
    \EE_\cD\!\left[
      \left(Y-\sum_{h=1}^H r_h(s_h,a_h)\right)^2
    \right].
    \label{eq:loss_reward_emp}
  \end{align}
\end{subequations}
The first term measures empirical policy mismatch, the second is a plug-in
Bellman error under the candidate reward $r$, and the third fits the latent
process reward to the observed trajectory outcomes. The minimization over $g$
is the standard double-sampling correction for Bellman regression, removing the
transition-noise variance from the squared target. Together, these criteria
enforce Bellman consistency for the critic while anchoring the candidate reward
to the outcome labels.

\paragraph{Outcome-based Pessimistic Actor-Critic (OPAC)}
Fix $\eta,\beta>0$ and $K\in\mathbb N$. Initialize
$\pi_{1,h}(\cdot\mid s)=\mathrm{Unif}(\cA_h)$. For $k=1,\ldots,K$, perform
\begin{align}
  &\textit{Pessimistic evaluation:}\quad
  (f_k,r_k)
  \in
  \argmin_{(f,r)\in\cF\times\cR}
  \cL_\cD(\pi_k,f)
  +\beta \cL_\cD^{\mathrm{BE}}(\pi_k,r,f)
  +\beta \cL_\cD^{\mathrm{RM}}(r),
  \label{eq:alg-pessimism}
  \\
  &\textit{Policy improvement:}\quad
  \pi_{k+1,h}(\cdot\mid s)
  \propto
  \pi_{k,h}(\cdot\mid s)
  \exp\bigl(f_{k,h}(s,\cdot)/\eta\bigr),
  \qquad h\in[H].
  \label{eq:alg-policy-optimization}
\end{align}
The output is the mixture policy
$\bar\pi\coloneqq\mathrm{Unif}(\pi_1,\ldots,\pi_K)$: at deployment, sample
$k\sim\mathrm{Unif}[K]$ once and execute $\pi_k$ for the entire episode.
\subsection{Theoretical Analysis}

We state the assumptions and main upper bound for OPAC. Besides standard
function-approximation conditions, the outcome-specific assumption is reward
realizability.

\begin{assumption}[Reward realizability]\label{assumption:realizable-R}
  The ground-truth mean process reward satisfies $r^\star\in\cR$.
\end{assumption}

We also assume approximate value realizability and Bellman completeness. The
latter is imposed uniformly over candidate rewards $r\in\cR$, since OPAC
estimates rewards from outcome labels and evaluates Bellman backups under these
candidates.

\begin{assumption}[Approximate realizability]\label{assumption:realizable-f}
  For all $\pi\in\Pi$,
  $\min_{f\in\cF}\sup_{h,\nu}\Vert f_h-(\cT_{r^\star}^\pi f)_h\Vert_{2,\nu}^2
  \leq \varepsilon_\cF$.
\end{assumption}

\begin{assumption}[Completeness]\label{assumption:completeness}
  For every $(\pi,r,f)\in\Pi\times\cR\times\cF$,
  $\min_{g\in\cF}\Vert g-\cT_r^\pi f\Vert_{2,\mu}^2
  \leq \varepsilon_{\cF,\cF}$.
\end{assumption}

\begin{theorem}[Upper bound for outcome-based offline RL]\label{theorem:outcome-based-offline-rl}
  Suppose \Cref{assumption:realizable-R,assumption:realizable-f,assumption:completeness} hold.
  Running OPAC for $K$ iterations with suitable parameters $\eta$ and $\beta$
  returns a policy $\bar\pi$ such that, with probability at least $1-\delta$,
  \begin{align}\label{eq:main-ub-general}
    J(\pi^\star)-J(\bar\pi)
    =
    \widetilde O\left(
    H^2\sqrt{\frac{C_{sa}(\pi^\star)}{n}}
    +\frac{H^2}{\sqrt K}
    +\sqrt{H C_{sa}(\pi^\star)\varepsilon_\cF}
    +\sqrt{H C_{sa}(\pi^\star)\varepsilon_{\cF,\cF}}
    +H\sqrt{\varepsilon_\cF}
    \right).
  \end{align}
  In particular, if $\varepsilon_\cF=\varepsilon_{\cF,\cF}=0$ and $K\ge n$, then $ J(\pi^\star)-J(\bar\pi)
  =
  \widetilde O\left(
  H^2\sqrt{{C_{sa}(\pi^\star)}/{n}}
  \right)$,
  or equivalently $n=\widetilde O(H^4C_{sa}(\pi^\star)/\varepsilon^2)$
  trajectories suffice for $\varepsilon$-optimality.
\end{theorem}

The proof of \Cref{theorem:outcome-based-offline-rl} is given in
\Cref{section:proof-additive-upper-bound}. The terms in
\eqref{eq:main-ub-general} have the usual interpretations: statistical error
under coverage $C_{sa}(\pi^\star)$, optimization error from $K$ actor--critic
iterations, and approximation errors from $\varepsilon_\cF$ and
$\varepsilon_{\cF,\cF}$. The distinctive feature is the horizon dependence:
relative to process-level supervision, outcome supervision pays one additional
factor of $H$. The lower bound in \Cref{thm:outcome_based_lower_bound} shows
that this factor is unavoidable, reflecting the statistical price of compressing
$H$ process-level reward observations into a single trajectory outcome; see
\Cref{remark:one-more-H}.
\subsection{Lower Bound}

The next theorem matches scaling of \cref{eq:main-ub-general}, so $\varepsilon$-optimality can require $\Omega(H^4/\varepsilon^2)$ trajectories.

\begin{theorem}[Lower bound for outcome-based offline RL]\label{thm:outcome_based_lower_bound}
  For all sufficiently large horizons $H$ and all $n\ge 64H^2$, there exists a class of outcome-based offline RL instances with deterministic transitions and  $C_{sa}(\pi^\star)\le2$, such that any algorithm using $n$ trajectories must incur
  \begin{align*}%
    \sup_{M}
    \EE_{M}   \left[J_{M}(\pi_{M}^\star)-J_{M}(\widehat\pi)\right]
    \ge
    c   \frac{H^2}{\sqrt n},
  \end{align*}
  for a universal constant $c>0$. Thus, even with constant state--action coverage and no transition-learning difficulty, achieving expected sub-optimality at most $\varepsilon$ requires $\Omega   \left(H^4/\varepsilon^2\right)$ trajectories.
\end{theorem}
The proof of \Cref{thm:outcome_based_lower_bound} is given in \Cref{subsec:lb-sum-reward-outcomes}.

\begin{remark}[Tightness of OPAC]\label{remark:tight}
\Cref{thm:outcome_based_lower_bound} matches
\Cref{theorem:outcome-based-offline-rl} in every parameter that appears: the
$H^2$ horizon factor, the $\sqrt{1/n}$ statistical rate, and the dependence on
$C_{sa}(\pi^\star)$. Consequently, the
$\widetilde O(H^2\sqrt{C_{sa}(\pi^\star)/n})$ guarantee achieved by
OPAC---through pessimistic actor--critic search and
trajectory-level squared regression on the outcome label---is statistically
optimal up to logarithmic factors among all algorithms that observe bounded
unbiased outcome labels. No refinement of these algorithmic ingredients can
improve the rate without strengthening the supervision model itself.
\end{remark}

\begin{remark}[Comparison with process-supervision offline RL]\label{remark:one-more-H}
Taken together,
\Cref{theorem:outcome-based-offline-rl,thm:outcome_based_lower_bound} identify
$\widetilde\Theta(H^4/\varepsilon^2)$ trajectories as the outcome-supervised
sample complexity. This is strictly more demanding than the process-level
setting, where rewards are observed at every step: \citet{xie2021policy}
establish an information-theoretic lower bound of $\Omega(H^3/\varepsilon^2)$
for that regime. The hard instance in \Cref{thm:outcome_based_lower_bound} has
essentially no transition-learning difficulty---it is a deterministic chain
with a single state per layer, two actions, and a constant single-policy
concentrability coefficient. The extra factor of $H$ therefore cannot be
attributed to exploration, transition uncertainty, or poor coverage; it is the
purely statistical price of compressing $H$ per-step reward signals into a
single trajectory-level label.
\end{remark}

\section{Offline RL from Trajectory-Level Preferences}\label{sec:preference}

We next consider a weaker form of trajectory-level supervision: instead of a
scalar outcome for each trajectory, the learner observes only pairwise
preferences between trajectories. This is the offline analogue of the
preference-based outcome-feedback model studied by \citet{chen2025outcome}, and
also a stylized abstraction of reward modeling in RLHF. Throughout this section,
we keep the normalization $r_h^\star(s,a)\in[0,1]$, so
$R^\star(\tau)\in[0,H]$. The optimization target remains the standard cumulative
return; only the observation model changes.

\paragraph{Preference observation model}
The offline dataset is
$\cD_{\mathrm{Pref}}=\{(\tau^{(i),+},\tau^{(i),-},y^{(i)})\}_{i=1}^n$, where
$\tau^{(i),+},\tau^{(i),-}\stackrel{\mathrm{i.i.d.}}{\sim}\mu$ and
$y^{(i)}\in\{0,1\}$ indicates whether the first trajectory is preferred. We
assume a Bradley--Terry--Luce (BTL) comparison model
\citep{bradley1952rank,christiano2017deep}. For any reward $r\in\cR$, define
\begin{align}\label{eq:btl-choice-prob}
  \mathrm{C}_r(\tau^+,\tau^-)
  \coloneqq
  \frac{\exp(\gamma R(\tau^+;r))}
  {\exp(\gamma R(\tau^+;r))+\exp(\gamma R(\tau^-;r))},
\end{align}
where $\gamma>0$ is fixed. Preferences are generated from the true reward:
\begin{align}\label{eq:btl-model}
  y^{(i)}
  \mid
  (\tau^{(i),+},\tau^{(i),-})
  \sim
  \mathrm{Bern}\bigl(
    \mathrm{C}_{r^\star}(\tau^{(i),+},\tau^{(i),-})
  \bigr).
\end{align}
Thus the learner observes neither $r^\star$ nor $R^\star(\tau)$ directly, but
only noisy binary comparisons.

\paragraph{Preference loss and algorithm}
For a candidate reward $r\in\cR$, define the empirical preference loss
\begin{align}\label{eq:pref-loss}
  \cL_{\cD_{\mathrm{Pref}}}^{\mathrm{Pref}}(r)
  \coloneqq
  \frac1n\sum_{i=1}^n
  \left[
    -y^{(i)}\log \mathrm{C}_r(\tau^{(i),+},\tau^{(i),-})
    -(1-y^{(i)})\log\bigl(1-\mathrm{C}_r(\tau^{(i),+},\tau^{(i),-})\bigr)
  \right].
\end{align}
This replaces the squared reward-model loss
$\cL_\cD^{\mathrm{RM}}(r)$ in \Cref{eq:loss_reward_emp}: both fit the latent
per-step reward from trajectory feedback, but preferences
identify rewards only through return differences. Uniform shifts of all
trajectory returns are therefore unidentifiable, but do not affect policy
optimization.

Let $\cD_{\mathrm{traj}}$ denote the $2n$ individual
trajectories appearing in the preference pairs. Preference-based OPAC uses the
same pessimistic actor--critic template as OPAC, with only the reward-model loss
changed. Specifically, starting from the same uniform policy $\pi_1$, for
$k=1,\ldots,K$ compute
\begin{align}\label{eq:pref-alg-pessimism}
  (f_k,r_k)
  \in
  \argmin_{(f,r)\in\cF\times\cR}
  \cL_{\cD_{\mathrm{traj}}}(\pi_k,f)
  +\beta \cL_{\cD_{\mathrm{traj}}}^{\mathrm{BE}}(\pi_k,r,f)
  +\beta \cL_{\cD_{\mathrm{Pref}}}^{\mathrm{Pref}}(r),
\end{align}
and then apply the same layer-wise exponential-weights policy update as in
\Cref{eq:alg-policy-optimization}. The output is again the mixture policy
$\bar\pi=\mathrm{Unif}(\pi_1,\ldots,\pi_K)$. Thus, the preference setting
requires no new actor--critic machinery: only the reward-fitting changes
from squared regression to logistic  regression.

\paragraph{Sample complexity}
Assume every candidate reward $r\in\cR$ satisfies $0\le r_h(s,a)\le 1$, so
$R(\tau;r)\in[0,H]$, and set $V_{\max}=H$ as in the scalar-outcome setting. We
regard the BTL model as fixed; the comparison constants
$\alpha_{\mathrm C}$ and $c_{\mathrm C}$ are defined in
\Cref{eq:pref-model-constants}.

\begin{theorem}[Upper bound for outcome-based offline RL with preferences]\label{thm:pref-outcome-based-main}
  Suppose 
  \Cref{assumption:realizable-R,assumption:realizable-f,assumption:completeness}
  hold. Running preference-based OPAC for $K$ iterations with suitable
  parameters $\eta$ and $\beta$ returns a policy $\bar\pi$ such that, with
  probability at least $1-\delta$,
  \begin{align}\label{eq:pref-main-rate}
    J(\pi^\star)-J(\bar\pi)
    =
    \widetilde O\Bigg(
     & \sqrt{1+\frac{1}{\alpha_{\mathrm C}}}
    \left[
      H^2\sqrt{\frac{C_{sa}(\pi^\star)}{n}}
      +
      \sqrt{\frac{c_{\mathrm C}H C_{sa}(\pi^\star)}{n}}
    \right] \notag \\
     & \quad +
     \sqrt{HC_{sa}(\pi^\star)\varepsilon_\cF}
     +
     \sqrt{HC_{sa}(\pi^\star)\varepsilon_{\cF,\cF}}
    +\frac{H^2}{\sqrt K}
    +H\sqrt{\varepsilon_\cF}
    \Bigg).
  \end{align}
  In particular, if $\varepsilon_\cF=\varepsilon_{\cF,\cF}=0$ and $K\ge n$, then
  \[
    J(\pi^\star)-J(\bar\pi)
    =
    \widetilde O\left(
    \sqrt{1+1/\alpha_{\mathrm C}}
    \left[
      H^2\sqrt{C_{sa}(\pi^\star)/n}
      +
      \sqrt{c_{\mathrm C}H C_{sa}(\pi^\star)/n}
    \right]
    \right).
  \]
\end{theorem}

The proof is given in \Cref{appendix:pref-upper}. The rate has the same
offline-RL structure as \Cref{theorem:outcome-based-offline-rl}: the
coverage coefficient $C_{sa}(\pi^\star)$, and the leading term preserves the
$H^2\sqrt{C_{sa}(\pi^\star)/n}$ dependence. The
additional constants $\alpha_{\mathrm C}$ and $c_{\mathrm C}$ quantify the comparison model on the bounded return range $[0,H]$.

\section{Outcome-Based RL with Generalized Objective}
\label{section:main-result-generalized}
The preceding sections fix the cumulative-reward objective and weaken only the supervision, replacing the scalar trajectory outcome by a binary preference.
In many applications, however, the trajectory-level signal is not a noisy proxy for a latent per-step sum but the performance criterion itself, and this criterion may be a nonlinear function of the per-step rewards. The natural object of study is then the trajectory outcome itself, both as the data we observe and as the quantity we wish to maximize.

This motivates the following generalized RL formulation. Fix a known stage-wise aggregation rule $\sigma:\RR^H\to[0,V_{\max}]$. For any trajectory $\tau=(s_1,a_1,\dots,s_H,a_H)$ and any candidate per-step reward $r=(r_h)_{h=1}^H\in\cR$, define the trajectory return
$
  R(\tau;r)   =   \sigma\bigl(r_1(s_1,a_1),\ldots,r_H(s_H,a_H)\bigr)
$
and the corresponding policy value\footnote{In this section we overload notations, but just extend from the classical setting to the generalized objective setting.}
\begin{align}\label{eq:main-gen-objective}
  J(\pi)   =   \EE_{\tau\sim\pi}[R(\tau;r^\star)]   =   \EE_{\tau\sim\pi}   \left[\sigma\bigl(r^\star_1(s_1,a_1),\ldots,r^\star_H(s_H,a_H)\bigr)\right].
\end{align}
Taking $\sigma$ to be the sum recovers the classical setting; other choices capture nonlinear trajectory criteria, such as the following all-success outcome.

\begin{example}[All-success outcome]\label[example]{ex:allsuccess-outcome}
  Suppose $r_h(s,a)\in\{0,1\}$, and let
  $
    \sigma(u_1,\ldots,u_H)=\prod_{h=1}^H u_h.
  $
  Then
  $
    R(\tau;r)=\prod_{h=1}^H r_h(s_h,a_h)
  =1$ if every stage succeeds, so the generalized objective
  $
    J(\pi)=\EE_{\tau\sim\pi}   \left[\prod_{h=1}^H r_h^\star(s_h,a_h)\right]
  $
  is the probability of completing the trajectory without any failed stage.
\end{example}
The offline dataset takes the same form as before, $\cD=\{(\tau^{(i)},Y^{(i)})\}_{i=1}^{n}$, with $\EE[Y^{(i)}\mid\tau^{(i)}]=R(\tau^{(i)};r^\star)$ and $Y^{(i)}\in[0,V_{\max}]$. The learner is asked to use $\cD$ to maximize~\cref{eq:main-gen-objective}. We assume that the trajectory return $R(\tau;r)$ and every value function $f_h\in\cF$ take values in $[0,V_{\max}]$.
\begin{remark}
  The generalized objective is well posed under process supervision, and the
standard machinery of value functions and Bellman optimality equations extends
to this setting (\Cref{sec:generalized_rl}). Under outcome-based supervision,
the interpretation is especially direct: maximizing
\Cref{eq:main-gen-objective} is exactly maximizing the expected trajectory
outcome recorded in the offline data.
  \end{remark}

Unfortunately, this problem is no longer learnable in general. We first give an impossibility result for \Cref{ex:allsuccess-outcome}: even with deterministic transitions and constant state--action concentrability, the outcome label is informative on only an exponentially small set of trajectories, which forces an exponential sample complexity.

\begin{theorem}[Exponential lower bound under all-success aggregation]\label{thm:main-lb-allsuccess}
  For any horizon $H\ge1$, there exists a family of deterministic
  finite-horizon MDPs with the all-success outcome in
  \Cref{ex:allsuccess-outcome} and state--action concentrability
  $C_{sa}(\pi^\star)\le 2$, such that any offline learner using $n$ i.i.d.\
  outcome-labeled trajectories satisfies
  $\inf_{\widehat\pi}\sup_{\btheta\in\{0,1\}^H}
  \EE[J_\btheta(\pi^\star_\btheta)-J_\btheta(\widehat\pi)]
  \ge \frac14(1-n\cdot 2^{1-H})$.
  In particular, achieving expected suboptimality below $1/8$ requires
  $\Omega(2^H)$ trajectories.
\end{theorem}
The proof is given in \Cref{subsec:lb-allsuccess}.
It shows that, without suitable structure on $\sigma$, outcome-based learning can be statistically intractable.
We therefore focus on structured aggregations.
The first requirement is that the loss of per-step reward information in the scalar outcome is controlled by the reward-process coefficient $\kappa_\mu(\sigma)$ in \Cref{def:kappa}.
The second requirement is that the generalized objective admits a Bellman-learnable aggregation, with the Bellman inverse coefficient $\chi_\mu(\sigma)$ in \Cref{def:chi} measuring how much per-step information is preserved by the Bellman operator.

\begin{definition}[Reward Process Coefficient]\label{def:kappa}
  For an aggregation $\sigma$ and a  distribution $\mu$, define
  \begin{align*}
    \kappa_\mu(\sigma)\coloneqq
    \sup_{r\ne r'}
    \frac{\sum_{h=1}^H\EE_\mu\bigl[(r_h-r'_h)^2(s_h,a_h)\bigr]}
    {\EE_{\tau\sim\mu}\bigl[(R(\tau;r)-R(\tau;r'))^2\bigr]}.
  \end{align*}
\end{definition}

The coefficient $\kappa_\mu(\sigma)$ measures how hard it is to recover per-step rewards from outcome signals: a large $\kappa_\mu(\sigma)$ means that distinct reward profiles look nearly identical after being aggregated by $\sigma$.

Next, we represent $R(\tau;r)$ through stage-wise maps~$\sigma_h$.
For $\tau=(s_1,a_1,\ldots,s_H,a_H)$, let $\tau_h=(s_h,a_h,\ldots,s_H,a_H)$ denote the trajectory truncated at stage $h$, and $R_h(\tau_h;r)$ denote the return from stage $h$ onward, so that $R(\tau;r)=R_1(\tau_1;r)$. We define $R_h(\tau_h;r)$ recursively through stage-wise aggregation functions $\sigma_h:\RR\times\RR\to\RR$ and a terminal constant $g\in\RR$, where $\sigma_h(u,v)$ aggregates the current per-step reward $u$ with the continuation return $v$ from stage $h$ onward:
\begin{align}
  R_{H+1} \coloneqq g, \quad                                    
  R_h(\tau_h;r) & \coloneqq \sigma_h\bigl(r_h(s_h,a_h),   R_{h+1}(\tau_{h+1};r)\bigr),\qquad h=H,\ldots,1.\label{eq:main-traj-return}
\end{align}

\paragraph{Generalized Bellman operator}
Given a policy $\pi$, candidate reward $r=(r_h)_{h=1}^H\in\cR$, and value function $f=(f_h)_{h=1}^H\in\cF$, define the generalized Bellman operator $\cT_r^\pi$ by
\begin{align}\label{eq:def-gen-bellman}
  (\cT_r^\pi f)_h(s,a)\coloneqq\sigma_h\bigl(r_h(s,a),\,\EE_{s'\sim P_h(\cdot\mid s,a)}[f_{h+1}(s',\pi)]\bigr),\qquad (s,a)\in\cS_h\times\cA_h,
\end{align}
with boundary convention $f_{H+1}(\cdot,\pi)\equiv g$. When $r=r^\star$, we write $\cT^\pi\coloneqq\cT_{r^\star}^\pi$.

\begin{definition}[Bellman Inverse Coefficient]\label{def:chi}
  For aggregation $\sigma$ and trajectory distribution $\mu$, define
  \begin{align*}
    \chi_\mu(\sigma)\coloneqq\sup_{\pi,f,r\ne r'}\frac{\sum_{h=1}^H\EE_\mu\bigl[(r_h-r'_h)^2(s_h,a_h)\bigr]}{\sum_{h=1}^H\EE_\mu\bigl[((\cT_r^\pi f)_h-(\cT_{r'}^\pi f)_h)^2(s_h,a_h)\bigr]}.
  \end{align*}
\end{definition}
The coefficient $\chi_\mu(\sigma)$ measures how much  Bellman targets can compress reward differences: smaller $\chi$ means that closeness of Bellman targets more tightly controls closeness of the underlying rewards.

\begin{remark}[Why two coefficients are necessary]\label{remark:two-ratios}
The Reward Process Coefficient $\kappa_\mu(\sigma)$ and the Bellman Inverse
Coefficient $\chi_\mu(\sigma)$ capture two distinct inverse problems that any
outcome-based generalized RL algorithm must address. 
The coefficient
$\kappa_\mu(\sigma)$ governs information loss on the \emph{data} side: it
measures how well the scalar trajectory outcome $R(\tau;r^\star)$ identifies the
latent per-step reward process $r^\star$ through the aggregation rule $\sigma$.
In particular, the statistical complexity of learning per-step rewards from trajectory outcomes is precisely captured by $\kappa_\mu(\sigma)$, as detailed in \Cref{app:sigma-regression}.
The coefficient $\chi_\mu(\sigma)$ governs information loss on the
\emph{dynamic-programming} side: it measures how well generalized
Bellman targets preserve and propagate reward differences during policy
optimization.

Both coefficients appear in the rate of \Cref{thm:main-gen-outcome} because
Generalized OPAC must solve both tasks: it must recover the latent reward process
from outcome labels, controlled by $\kappa_\mu(\sigma)$, and optimize the
generalized objective through Bellman-style dynamic programming, controlled by
$\chi_\mu(\sigma)$. These two roles are independent and cannot be
merged into a single quantity. In simpler settings, one of the two coefficients
may become unnecessary. With process supervision, the outcome
inverse problem disappears, so $\kappa_\mu(\sigma)$ is not needed. For the
standard cumulative-reward objective, $\chi_\mu(\sigma)\equiv 1$, since per-step
reward differences pass through the  Bellman operator without additional
compression.
\end{remark}

We further restrict $\sigma_h(u,v)$ to satisfy the following regularity \Cref{assumption:sigma-regularity}, which is not meant to characterize all trajectory-level
objectives; rather, it identifies a Bellman-learnable subclass for which outcome-based
generalized offline RL is statistically tractable. Cumulative returns use $\sigma_h(u,v)=u+v$ with $g=0$; all-success (\Cref{ex:allsuccess-outcome}) uses $\sigma_h(u,v)=uv$ with $g=1$.
\Cref{sec:examples} presents additional objectives and relates statistical difficulty to $\kappa_\mu(\sigma)$ and $\chi_\mu(\sigma)$.
\begin{assumption}\label{assumption:sigma-regularity}
  \begin{enumerate*}[label=(\arabic*)]
    \item Affine in~$v$: $\sigma_h(u,v)=a_h(u)v+b_h(u)$.\footnote{Affine dependence on~$v$ is necessary for $\EE_v[\sigma_h(u,v)]=\sigma_h(u,\EE v)$ to hold for every law of~$v$.}
    \item No expansion of the continuation: $a_h(u)\in[0,1]$.
    \item $a_h,b_h$ are $L$-Lipschitz in~$u$.
    \item For every $f\in\cF$, $\pi\in\Pi$, $h$, and $(s,a)$, with $v_h(s,a)=\EE_{s'\sim P_h(\cdot\mid s,a)}[V_{h+1}^{\pi,f}(s')]$, there exists $\tilde r_h(s,a)$ such that $\sigma_h(\tilde r_h(s,a),v_h(s,a))=f_h(s,a)$.
  \end{enumerate*}
\end{assumption}

\paragraph{Empirical criteria and generalized OPAC}
We now instantiate OPAC for the generalized objective. The actor--critic
structure is unchanged; only the Bellman target, reward loss, and
trajectory weights are adapted. Define the generalized
one-step target
$y_h^{r,f,\pi}\coloneqq
\sigma_h(r_h(s_h,a_h),f_{h+1}(s_{h+1},\pi))$ for $h=1,\ldots,H$, with the
convention $f_{H+1}=g$. For the dataset $\cD$, define
\begin{subequations}\label{eq:gen-all-losses-emp}
  \begin{align}
    \cL_\cD^{W^r}(\pi,r,f)
    &\coloneqq
    \sum_{h=1}^H
    \EE_\cD\!\left[
      W_h^r(\tau)\bigl(f_h(s_h,\pi)-f_h(s_h,a_h)\bigr)
    \right],
    \label{eq:gen-loss-policy-emp}
    \\
    \cL_\cD^{\mathrm{BE}}(\pi,r,f)
    &\coloneqq
    \sum_{h=1}^{H}
    \left\{
    \EE_\cD\!\left[
      \bigl(f_h(s_h,a_h)-y_h^{r,f,\pi}\bigr)^2
    \right]
    -
    \min_{f'\in\cF_h}
    \EE_\cD\!\left[
      \bigl(f'(s_h,a_h)-y_h^{r,f,\pi}\bigr)^2
    \right]
    \right\},
    \label{eq:gen-loss-be-emp}
    \\
    \cL_\cD^{\mathrm{RM}}(r)
    &\coloneqq
    \EE_\cD\!\left[
      \bigl(R(\tau;r)-Y\bigr)^2
    \right].
    \label{eq:gen-loss-rm-emp}
  \end{align}
\end{subequations}
Here $W_1^r\equiv 1$ and
$W_h^r(\tau)\coloneqq\prod_{j=1}^{h-1}a_j(r_j(s_j,a_j))$ for $h\in[H]$.

Generalized OPAC uses the same policy initialization, layer-wise
exponential-weights update, and mixture output $\bar\pi\coloneqq\mathrm{Unif}(\pi_1,\ldots,\pi_K)$ as OPAC. The only change is the
pessimistic evaluation step: for $k=1,\ldots,K$, compute
\begin{align}\label{eq:gen-alg-pessimism}
  (f_k,r_k)
  \in
  \argmin_{(f,r)\in\cF\times\cR}
  \cL_\cD^{W^r}(\pi_k,r,f)
  +\beta \cL_\cD^{\mathrm{BE}}(\pi_k,r,f)
  +\beta \cL_\cD^{\mathrm{RM}}(r),
\end{align}
then apply the same policy-improvement step as in
\Cref{eq:alg-policy-optimization}. 
\begin{theorem}[Sample complexity of generalized OPAC]\label{thm:main-gen-outcome}
  Suppose that \Cref{assumption:realizable-R,assumption:realizable-f,assumption:completeness,assumption:sigma-regularity}  hold.  Running generalized OPAC for $K$ iterations with suitable
  parameters $\eta$ and $\beta$ returns a policy $\bar\pi$ such that, with
  probability at least $1-\delta$,
  \begin{equation}\label{eq:main-text-rate}
    \begin{aligned}
      J(\pi^\star)-J(\bar\pi)
      \le\widetilde O\!\Bigl(
       & V_{\max}^2L\sqrt{\frac{\kappa_\mu(\sigma)H^2C_{sa}(\pi^\star)}{n}}
      +V_{\max}^2L\sqrt{\frac{\chi_\mu(\sigma)H^4}{n}}
      +V_{\max}H\sqrt{\frac{\log|\cA|}{K}}                                        \\
       & +\Phi\sqrt{\varepsilon_{\cF,\cF}}+(\Phi+H^2)\sqrt{\varepsilon_\cF}\Bigr)
    \end{aligned}
  \end{equation}
  where $\Phi=\widetilde O\bigl(LV_{\max}\sqrt{\kappa_\mu(\sigma)HC_{sa}(\pi^\star)}+V_{\max}L\sqrt{\chi_\mu(\sigma)H^3}+\sqrt{HC_{sa}(\pi^\star)}\bigr)$.
\end{theorem}
The proof of \Cref{thm:main-gen-outcome} is given in \Cref{sec:proof-gen-upper-bound}. 
The dependence on the aggregation map $\sigma$ separates into two complementary
bottlenecks, captured by $\kappa_\mu(\sigma)$ and $\chi_\mu(\sigma)$
(\Cref{def:kappa,def:chi}): outcome aggregation may obscure the latent per-step
rewards, while generalized Bellman propagation may further compress reward
differences when forming one-step targets.

\section{Conclusion}\label{section:conclusion}

We developed a sample-complexity theory for offline RL with trajectory-level
supervision. For cumulative-reward objectives, OPAC achieves a sharp
$\widetilde O(H^2\sqrt{C_{sa}(\pi^\star)/n})$ rate, showing the statistical
cost of compressing $H$ process-level rewards into one trajectory-level label.
The same template extends to trajectory preferences. For nonlinear trajectory
aggregations, we show an exponential barrier in general, and
polynomial sample complexity for structured aggregations controlled by
$\kappa_\mu(\sigma)$ and $\chi_\mu(\sigma)$. Together, these results delineate
when trajectory-level supervision enables efficient offline policy optimization
and when missing process-level rewards are fundamentally restrictive.

\paragraph{Limitations.}
Our generalized framework assumes a known aggregation rule with Bellman-style
structure, and does not fully characterize which trajectory-level objectives are
both statistically learnable and practically meaningful. Identifying broader
structural conditions for efficient outcome-based learning remains an important
direction.

\bibliographystyle{plainnat}
\bibliography{ref}

\clearpage

\appendix
\onecolumn
\crefalias{section}{appendix}
\crefalias{subsection}{appendix}

\begin{center}
{\LARGE Appendix}
\end{center}

\section{Related Work}\label{section:related-work}

\paragraph{Offline RL Theory}

Our work builds on a rich body of theoretical results in offline RL; see \citet{jiang2025offline} for a recent survey. A central principle is pessimism in the face of uncertainty, which counteracts distribution shift by penalizing value estimates in regions poorly covered by data \citep{yin2021towards,jin2021pessimism,rashidinejad2021bridging,xie2021bellman,cheng2022adversarially,blanchet2023double}. Closely related to our algorithmic approach are pessimistic and adversarially trained actor-critic methods for offline RL \citep{xie2021bellman,zanette2021provable,cheng2022adversarially,ji2024self}. On the data coverage side, recent works have progressively weakened the required assumptions from all-policy concentrability to single-policy concentrability and partial coverage, with the minimax optimal sample complexity of model-based offline RL now settled \citep{rashidinejad2021bridging,xie2021bellman,xie2021policy,zhan2022offline,shi2022pessimistic,li2024settling}. On the function approximation side, a key question is what structural conditions on the function class enable efficient learning, spanning realizability-only guarantees, unifying complexity measures such as the Bellman eluder dimension, fundamental impossibility results, instance-dependent bounds, and characterizations of inherent Bellman error \citep{xie2021batch,jin2021bellman,foster2021offline,nguyen2023instance,jin2023provably,golowich2024role}.

\paragraph{Preference Learning}
A complementary line of theoretical work studies RL from preference feedback rather than scalar rewards. Empirically, reinforcement learning from human feedback (RLHF) \citep{christiano2017deep} has become a central recipe for aligning large generative models, and direct preference optimization \citep{rafailov2023direct} reformulates the RLHF objective as a supervised loss on preference data. On the theoretical side, \citet{zhu2023principled} provide principled guarantees for RLHF under both pairwise and $K$-wise comparison models, \citet{xiong2024iterative} bridge theory and practice for iterative preference learning under KL regularization, and \citet{zhan2023provable} establish provable guarantees for offline preference-based RL. Our preference setting in \Cref{sec:preference} adopts the standard Bradley--Terry--Luce comparison model \citep{bradley1952rank,christiano2017deep} and operates entirely in the offline regime. Compared with these prior analyses, we maintain dependence on a single-policy state--action concentrability $C_{sa}(\pi^\star)$ rather than the trajectory-level concentrability that often appears in preference-based analyses, and our guarantee depends on the BTL model only through interpretable constants $\alpha_{\mathrm C}$ and $c_{\mathrm C}$ that capture preference identifiability and concentration on bounded return ranges.

\paragraph{Outcome-based RL}
Outcome-based RL---where the learner observes only trajectory-level signals rather than per-step rewards---has been studied under several names. A long theoretical line on \emph{aggregate / trajectory / bandit feedback} \citep{neu2013efficient,efroni2021reinforcement,chatterji2021theory,chen2022human,cassel2024near,lancewicki2025near} considers learners that receive only an episode-level reward, primarily in online tabular or linear MDP settings. A second, more recent line studies outcome supervision under general function approximation, motivated by reasoning datasets and human-feedback alignment in large language models \citep{stiennon2020learning,cobbe2021training,ouyang2022training,bai2022training,jaech2024openai,guo2025deepseek}: \citet{jia2025we} show that offline outcome data can be transformed into per-step pseudo-rewards via trajectory-level regression at a cost controlled by the single-policy state--action concentrability rather than the much larger trajectory-level concentrability; \citet{yuan2025trajectory} propose trajectory Bellman residual minimization for LLM reasoning and improved the change-of-trajectory-measure inequality that we adopt which is initially introduced by \citet{jia2025we}; and \citet{chen2025outcome} establish an online optimism-based algorithm with $\widetilde O(C_{\mathrm{cov}}H^3/\varepsilon^2)$ sample complexity together with an exponential outcome-versus-process separation. Outcome-level feedback is also implicit in goal-reaching with sparse rewards \citep{andrychowicz2017hindsight}, in clinical and healthcare RL where labels are revealed only at trajectory boundaries \citep{komorowski2018artificial,gottesman2019guidelines}, and in trajectory-level human preferences (covered separately below). Compared to these works, we focus on the \emph{offline} regime under partial (single-policy) coverage and additionally study \emph{generalized} (nonlinear) trajectory-level objectives, where outcome-based learning may be statistically intractable and requires new complexity measures.

\paragraph{Compared with PRM Methods}
A rapidly growing empirical literature trains reward models that score either intermediate steps (process reward models, PRMs) or completed trajectories (outcome reward models, ORMs), particularly for verifying language-model reasoning \citep{lightman2023let,uesato2022solving}. \citet{uesato2022solving} compare process- and outcome-based feedback at scale on math word problems, while \citet{lightman2023let} report that step-by-step process supervision can outperform outcome supervision on challenging math benchmarks. \citet{jia2025we} subsequently provide a theoretical perspective showing that, under suitable coverage and realizability conditions, outcome supervision can match process supervision up to polynomial-in-horizon factors, suggesting that the empirical gap may stem from algorithmic rather than purely statistical sources. Our results sharpen this picture in the offline setting through matched upper and lower bounds: we precisely quantify the statistical price of outcome supervision as a single extra horizon factor relative to the process-supervised lower bound of \citet{xie2021policy}, arising purely from compressing $H$ per-step signals into a single label. We further extend this analysis to BTL preferences and to nonlinear trajectory-level objectives, where the picture changes qualitatively: structured aggregations remain tractable through the complexity measures we introduce, whereas certain nonlinear outcomes (e.g., all-success) admit exponential lower bounds even under constant state--action coverage.

\paragraph{Generalized RL Objectives}
Different optimization objectives correspond to different ways of aggregating stage-wise signals into trajectory-level criteria. Theoretical lines in this direction include risk-sensitive RL \citep{howard1972risk,bastani2022regret,wang2023near,wang2024reductions,han2025risk}, planning with submodular or other general (e.g., permutation-symmetric) trajectory objectives \citep{wang2020submodular,wang2020planning}, max-reward RL that replaces the cumulative return by the maximum per-step reward \citep{quah2006maximum,gottipati2020maximum,veviurko2402max}, dynamic-programming foundations for non-cumulative objectives that introduce parallel notions of generalized reward aggregation, value functions, and Bellman operators \citep{cui2023reinforcement}, optimization of state--action occupancy utilities beyond cumulative rewards under hidden convexity or general utilities where Bellman dynamic programming breaks down \citep{zhang2020variational,barakat2023reinforcement}, control synthesis from linear temporal logic specifications \citep{bozkurt2020control,bagatella2024directed}, and PAC learnability for computably continuous RL objectives \citep{yang2023computably}. Most of these works analyze convergence of dynamic-programming or online policy-gradient procedures rather than offline finite-sample complexity.

Generalized aggregations also arise naturally in language-model post-training and reasoning, mirroring the examples in \Cref{sec:examples}. Process- and outcome-supervised reward modeling for math reasoning \citep{cobbe2021training,uesato2022solving,lightman2023let,wang2024math,setlur2024rewarding} effectively imposes \emph{all-step-correct} (product) aggregations over per-step verifiers, matching the all-success example for which we establish an exponential $\Omega(2^H)$ lower bound. RLHF with KL regularization \citep{christiano2017deep,ouyang2022training,bai2022training,rafailov2023direct,xiong2024iterative} replaces the standard additive return with a non-cumulative criterion closely related to the exponential-utility / entropic-risk instance covered in \Cref{sec:examples}. Compared with these lines, we develop the first finite-sample-complexity theory for \emph{offline} RL under generalized trajectory-level objectives, and characterize tractable regimes through the reward-process coefficient $\kappa_\mu(\sigma)$ and the Bellman inverse coefficient $\chi_\mu(\sigma)$.

\section{Technical Lemmas for Outcome-based Offline RL}
This appendix uses the following population loss notation throughout. For a policy $\pi\in\Pi$, a candidate reward $r\in\cR$, and a critic $f\in\cF$, define
\begin{subequations}\label{eq:all_losses}
  \begin{align}
    \cL_\mu(\pi,f)
     & \coloneqq
    \sum_{h=1}^{H}\EE_\mu   \left[f_h(s_h,\pi)-f_h(s_h,a_h)\right],
    \label{eq:loss_policy}  \\
    \cL_\mu^{\mathrm{BE}}(\pi,r,f)
     & \coloneqq
    \sum_{h=1}^{H}\EE_\mu   \left[\left((f-\cT_r^\pi f)(s_h,a_h)\right)^2\right],
    \label{eq:loss_bellman} \\
    \cL_\mu^{\mathrm{RM}}(r)
     & \coloneqq
    \EE_\mu   \left[\left(\sum_{h=1}^H (r_h-r_h^\star)(s_h,a_h)\right)^2\right]
    =
    \EE_\mu   \left[(R(\tau;r)-R^\star(\tau))^2\right].
    \label{eq:loss_reward}
  \end{align}
\end{subequations}

\subsection{Change-of-Measure and Optimization Tools}
This subsection collects the two auxiliary tools used to control the change-of-measure terms and the policy-optimization term in the proof of the upper bound.

The following statement follows from Lemma~1 or Corollary~6 of \citet{yuan2025trajectory}.
\begin{lemma}[Change of trajectory measure]\label{lemma:change-of-trajectory}

  For any finite-horizon MDP, behaviour policy $\mu$, target policy $\pi$, and bounded measurable function $g:\cS\times\cA\to\RR$ with
  $\EE_\mu[(\sum_{h=1}^H g(s_h,a_h))^2]>0$, we have
  \begin{align*}
    \frac{\left(\mathbb{E}_{\pi}\left[\sum_{h=1}^{H} g(s_{h}, a_{h})\right]\right)^{2}}{\mathbb{E}_{\mu}\left[\left(\sum_{h=1}^{H} g(s_{h}, a_{h})\right)^{2}\right]}
    \leq H C_{sa}(\pi).
  \end{align*}

\end{lemma}
The next lemma is the standard exponential-weights no-regret bound \citep{freund1997decision,cesa2006prediction}, specialised to our layer-wise policy update; we include the proof for completeness.
\begin{lemma}[No-regret bound]\label{lemma:hedge-regret}
  Let $\cA$ be a finite action space with $|\cA|<\infty$, and let
  \begin{align*}
    \Pi
    =
    \bigl\{
    \pi=(\pi_1,\dots,\pi_H):\pi_h:\cS_h\to\Delta(\cA),\ h\in[H]
    \bigr\}
  \end{align*}
  be the class of non-stationary policies. Let $\{f_k\}_{k=1}^{K}\subset\cF$ be an arbitrary (possibly data-adaptive) sequence of critics such that, for every $k\in[K]$ and $h\in[H]$,
  \begin{align*}
    f_{k,h}:\cS_h\times\cA\to[0,V_{\max}].
  \end{align*}
  Set $\eta=V_{\max}\sqrt{K/(8\log|\cA|)}$ in the update below.
  Initialize
  \begin{align*}
    \pi_{1,h}(\cdot\mid s)=\mathrm{Unif}(\cA),
    \qquad\forall    h\in[H],\ s\in\cS_h,
  \end{align*}
  and for each $k=1,\dots,K-1$ and $h\in[H]$, update
  \begin{align}
    \pi_{k+1,h}(a\mid s)
    =
    \frac{\pi_{k,h}(a\mid s)\exp\bigl(f_{k,h}(s,a)/\eta\bigr)}{\sum_{a'\in\cA}\pi_{k,h}(a'\mid s)\exp\bigl(f_{k,h}(s,a')/\eta\bigr)}.
    \label{eq:hedge-update}
  \end{align}
  For any comparator $\pi\in\Pi$, let $d_{h,S}^\pi(s)\coloneqq \sum_{a\in\cA_h}d_h^\pi(s,a)$ denote the step-$h$ state marginal, and define
  \begin{align*}
    \mathrm{Reg}_k(\pi)
    \coloneqq
    \sum_{h=1}^{H}
    \EE_{s\sim d_{h,S}^\pi}
    \Bigl[
      \EE_{a\sim\pi_h(\cdot\mid s)}[f_{k,h}(s,a)]
      -
      \EE_{a\sim\pi_{k,h}(\cdot\mid s)}[f_{k,h}(s,a)]
      \Bigr].
  \end{align*}
  Then, for every comparator $\pi\in\Pi$,
  \begin{align}
    \frac{1}{K}\sum_{k=1}^{K}\mathrm{Reg}_k(\pi)
    \le
    H   V_{\max}\sqrt{\frac{\log|\cA|}{2K}}.
    \label{eq:hedge-optimized}
  \end{align}
\end{lemma}

\begin{proof}
  We prove a pointwise regret bound for each fixed pair $(s,h)$, and then average over $s\sim d_{h,S}^\pi$ and sum over $h$.

  Fix any $h\in[H]$ and $s\in\cS_h$. Define
  \begin{align*}
    q_k(\cdot)\coloneqq \pi_{k,h}(\cdot\mid s),
    \qquad
    g_k(\cdot)\coloneqq f_{k,h}(s,\cdot)\in[0,V_{\max}]^{|\cA|}.
  \end{align*}
  Then \cref{eq:hedge-update} becomes the exponential-weights update
  \begin{align*}
    q_{k+1}(a)
    =
    \frac{q_k(a)\exp(g_k(a)/\eta)}{\sum_{a'\in\cA}q_k(a')\exp(g_k(a')/\eta)}.
  \end{align*}

  Introduce weights
  \begin{align*}
    w_1(a)=q_1(a)=\frac{1}{|\cA|},
    \qquad
    w_{k+1}(a)=w_k(a)\exp(g_k(a)/\eta),
  \end{align*}
  and let $W_k\coloneqq \sum_{a\in\cA}w_k(a)$. Since $W_1=1$ and $q_k(a)=w_k(a)/W_k$,
  \begin{align*}
    \frac{W_{k+1}}{W_k}
    =
    \sum_{a\in\cA}q_k(a)\exp(g_k(a)/\eta)
    =
    \EE_{a\sim q_k}   \bigl[\exp(g_k(a)/\eta)\bigr].
  \end{align*}

  \paragraph{Upper bound on $\log W_{K+1}$} By Hoeffding's lemma, since $g_k(a)/\eta\in[0,V_{\max}/\eta]$ for all $a\in\cA$,
  \begin{align*}
    \log\EE_{a\sim q_k}   \bigl[\exp(g_k(a)/\eta)\bigr]
    \le
    \frac{1}{\eta}\EE_{a\sim q_k}[g_k(a)]
    +
    \frac{V_{\max}^2}{8\eta^2}.
  \end{align*}
  Summing over $k=1,\dots,K$ and using $W_1=1$,
  \begin{align}
    \log W_{K+1}
    \le
    \frac{1}{\eta}\sum_{k=1}^{K}\EE_{a\sim q_k}[g_k(a)]
    +
    \frac{V_{\max}^{2}K}{8\eta^2}.
    \label{eq:upper-W}
  \end{align}

  \paragraph{Lower bound on $\log W_{K+1}$}
  \begin{align*}
    W_{K+1}
    =
    \sum_{a\in\cA}q_1(a)\exp \left(\frac{1}{\eta}\sum_{k=1}^{K}g_k(a)\right).
  \end{align*}
  Since $q_1=\mathrm{Unif}(\cA)$ is strictly positive, we can write, for any $p\in\Delta(\cA)$,
  \begin{align*}
    W_{K+1}
    =
    \sum_{a\in\cA} p(a)\cdot\frac{q_1(a)}{p(a)}\exp \left(\frac{1}{\eta}\sum_{k=1}^{K}g_k(a)\right)
    =
    \EE_{a\sim p}   \left[\frac{q_1(a)}{p(a)}\exp \left(\frac{1}{\eta}\sum_{k=1}^{K}g_k(a)\right)\right].
  \end{align*}
  Applying Jensen's inequality to the concave function $\log(\cdot)$ yields
  \begin{align*}
    \log W_{K+1}
    \ge
    \EE_{a\sim p}   \left[\log\frac{q_1(a)}{p(a)}+\frac{1}{\eta}\sum_{k=1}^{K}g_k(a)\right]
    =
    \frac{1}{\eta}\sum_{k=1}^{K}\EE_{a\sim p}[g_k(a)]
    -
    \KL(p\|q_1).
  \end{align*}
  Combining with \cref{eq:upper-W} gives
  \begin{align*}
    \sum_{k=1}^{K}
    \bigl(\EE_{a\sim p}[g_k(a)]-\EE_{a\sim q_k}[g_k(a)]\bigr)
    \le
    \eta\KL(p\|q_1)
    +
    \frac{V_{\max}^{2}K}{8\eta}.
  \end{align*}

  Since $q_1=\mathrm{Unif}(\cA)$,
  \begin{align*}
    \KL(p\|q_1)
    =
    \sum_{a\in\cA}p(a)\log \bigl(p(a)|\cA|\bigr)
    =
    \log|\cA|+\sum_{a\in\cA}p(a)\log p(a)
    \le
    \log|\cA|,
  \end{align*}
  where the last inequality uses $\sum_{a}p(a)\log p(a)\le 0$ for any probability distribution $p$. Hence, for every $p\in\Delta(\cA)$,
  \begin{align}
    \sum_{k=1}^{K}
    \bigl(\EE_{a\sim p}[g_k(a)]-\EE_{a\sim q_k}[g_k(a)]\bigr)
    \le
    \eta\log|\cA|
    +
    \frac{V_{\max}^{2}K}{8\eta}.
    \label{eq:pointwise-regret}
  \end{align}

  Setting $p(\cdot)=\pi_h(\cdot\mid s)$, recalling $g_k(a)=f_{k,h}(s,a)$ and $q_k(\cdot)=\pi_{k,h}(\cdot\mid s)$, \cref{eq:pointwise-regret} becomes a deterministic (pointwise-in-$s$) inequality. Taking $\EE_{s\sim d_{h,S}^\pi}[\cdot]$ preserves it, and summing over $h=1,\dots,H$ yields the intermediate bound
  \begin{align}
    \sum_{k=1}^{K}\mathrm{Reg}_k(\pi)
    \le
    H   \left(
    \eta\log|\cA|+\frac{V_{\max}^{2}K}{8\eta}
    \right).
    \label{eq:hedge-raw}
  \end{align}

  Finally, when $|\cA|\ge 2$, AM--GM minimizes the right-hand side of \cref{eq:hedge-raw} over $\eta>0$ at
  \begin{align*}
    \eta^\star
    =
    V_{\max}\sqrt{\frac{K}{8\log|\cA|}},
    \qquad
    \eta^\star\log|\cA|+\frac{V_{\max}^{2}K}{8\eta^\star}
    =
    V_{\max}\sqrt{\frac{K\log|\cA|}{2}},
  \end{align*}
  and dividing by $K$ gives \cref{eq:hedge-optimized}. The case $|\cA|=1$ is trivial.
\end{proof}

\subsection{Concentration Lemmas}\label{subsection:concentration}
This subsection collects the high-probability estimates used to transfer the population losses in the analysis to their empirical counterparts.
\begin{lemma}[Uniform concentration of the policy loss]\label{lemma:concentration-L}
  Assume every $f=(f_1,\dots,f_H)\in\cF$ satisfies $f_h(s,a)\in[0,V_{\max}]$ for all $h\in[H]$ and $(s,a)\in\cS_h\times\cA$, and that $\cD=\{\tau_i\}_{i=1}^n$ consists of $n$ i.i.d.\ trajectories drawn from $\mu$. Fix $\delta\in(0,1)$ and define
  \begin{align}
    \varepsilon_{\mathrm{Perf}}   \coloneqq   2   H^2   V_{\max}^2   \frac{\log(2|\Pi|   |\cF|/\delta)}{n}.\label{eq:def-epsilon-perf}
  \end{align}
  Then with probability at least $1-\delta$, simultaneously for every $\pi\in\Pi$ and every $f\in\cF$,
  \begin{align}
    \bigl|\cL_\mu(\pi,f)-\cL_\cD(\pi,f)\bigr|   \le   \sqrt{\varepsilon_{\mathrm{Perf}}}.\label{eq:perf-single}
  \end{align}
\end{lemma}

\begin{proof}
  Fix $(\pi,f)\in\Pi\times\cF$. For each trajectory $\tau_i=(s_{i,1},a_{i,1},\dots,s_{i,H},a_{i,H})$, set
  \begin{align*}
    Z_i(\pi,f)   \coloneqq   \sum_{h=1}^H\bigl[f_h(s_{i,h},\pi(s_{i,h}))-f_h(s_{i,h},a_{i,h})\bigr].
  \end{align*}
  Since the $\tau_i$ are i.i.d.\ under $\mu$, the variables $\{Z_i(\pi,f)\}_{i=1}^n$ are i.i.d., and by definition of $\cL_\mu(\pi,f)$ and $\cL_\cD(\pi,f)$ in \Cref{section:algorithms},
  \begin{align*}
    \EE_{\tau_i\sim\mu}\bigl[Z_i(\pi,f)\bigr]   =   \cL_\mu(\pi,f),\qquad \frac{1}{n}\sum_{i=1}^n Z_i(\pi,f)   =   \cL_\cD(\pi,f).
  \end{align*}
  Moreover, since $f_h$ takes values in $[0,V_{\max}]$, each summand satisfies
  \begin{align*}
    f_h(s_{i,h},\pi(s_{i,h}))-f_h(s_{i,h},a_{i,h})   \in   [-V_{\max},   V_{\max}],
  \end{align*}
  and therefore $|Z_i(\pi,f)|\le H   V_{\max}$ almost surely. Hoeffding's inequality applied to the i.i.d.\ bounded random variables $\{Z_i(\pi,f)\}_{i=1}^n$ gives, for any $\delta'\in(0,1)$, with probability at least $1-\delta'$,
  \begin{align*}
    \bigl|\cL_\mu(\pi,f)-\cL_\cD(\pi,f)\bigr|   \le   H   V_{\max}\sqrt{\frac{2\log(2/\delta')}{n}}.
  \end{align*}
  Taking a union bound over all pairs $(\pi,f)\in\Pi\times\cF$ with $\delta'=\delta/(|\Pi|   |\cF|)$, and noting that $H   V_{\max}\sqrt{2\log(2|\Pi||\cF|/\delta)/n}=\sqrt{\varepsilon_{\mathrm{Perf}}}$ by definition \cref{eq:def-epsilon-perf}, yields \cref{eq:perf-single} simultaneously for every $(\pi,f)\in\Pi\times\cF$.
\end{proof}

\begin{lemma}[Reward-model concentration via Bernstein]
  \label{lem:reward-model-concentration-bernstein}\label{lem:noisy-rm-signal-concentration}
  Let $\cD=\{(\tau_i,Y_i)\}_{i=1}^n$ be $n$ i.i.d.\ samples collected under $\mu$, where
  \begin{align*}
    Y_i=\sum_{h=1}^H r^\star_h(s_{i,h},a_{i,h})+\xi_i,
    \qquad \EE[\xi_i\mid\tau_i]=0,
    \qquad |\xi_i|\le H \text{ a.s.},
  \end{align*}
  Assume $r^\star\in\cR$ and every $r\in\cR$ satisfies $0\le r_h(s,a)\le 1$. Define the population and empirical reward-model losses
  \begin{align*}
    \cL_\mu^{\mathrm{RM}}(r)
     & :=
    \EE_{\mu}   \left[\Big(\textstyle\sum_{h=1}^H (r-r^\star)(s_h,a_h)\Big)^2\right], \\
    \cL_\cD^{\mathrm{RM}}(r)
     & :=
    \frac1n\sum_{i=1}^n\Big(Y_i-\textstyle\sum_{h=1}^H r_h(s_{i,h},a_{i,h})\Big)^2.
  \end{align*}
  Then for any $\delta\in(0,1)$, with probability at least $1-\delta$, simultaneously for every $r\in\cR$,
  \begin{align}
    \cL_\mu^{\mathrm{RM}}(r)
    \le
    2\bigl[\cL_\cD^{\mathrm{RM}}(r)-\cL_\cD^{\mathrm{RM}}(r^\star)\bigr]
    +   \varepsilon_{\mathrm{RM}},
    \label{eq:rm-bernstein-solved}
  \end{align}
  where
  \begin{align*}
    \varepsilon_{\mathrm{RM}}
    =
    \frac{128}{3}\cdot\frac{H^2   \log(|\cR|/\delta)}{n},
  \end{align*}
\end{lemma}

\begin{proof}
  Fix any $r\in\cR$ and write $\Delta_i:=\sum_{h=1}^H(r-r^\star)(s_{i,h},a_{i,h})$. Under the assumption $0\le r_h,r^\star_h\le 1$ we have $|\Delta_i|\le H$. A direct expansion gives
  \begin{align*}
    \bigl(Y_i-\textstyle\sum_h r\bigr)^2 - \bigl(Y_i-\sum_h r^\star\bigr)^2
    =(\Delta_i-\xi_i)^2-\xi_i^2=\Delta_i^2-2\Delta_i\xi_i,
  \end{align*}
  so, letting $U_i(r):=\Delta_i^2-2\Delta_i\xi_i$,
  \begin{align}
    \cL_\cD^{\mathrm{RM}}(r)-\cL_\cD^{\mathrm{RM}}(r^\star)=\frac1n\sum_{i=1}^n U_i(r).
    \label{eq:UiDef}
  \end{align}
  \emph{Mean.} Since $\EE[\xi_i\mid\tau_i]=0$,
  \begin{align}
    \EE[U_i(r)]=\EE[\Delta_i^2]-2\EE[\Delta_i\EE[\xi_i\mid\tau_i]]=\EE[\Delta_i^2]=\cL_\mu^{\mathrm{RM}}(r).
    \label{eq:UiMean}
  \end{align}

  \emph{Variance (self-bounding).} Using $\EE[\xi_i\mid\tau_i]=0$, we have $\mathrm{Cov}(\Delta_i^2,\Delta_i\xi_i)=\EE[\Delta_i^3\EE[\xi_i\mid\tau_i]]=0$, so
  \begin{align*}
    \Var(U_i(r))=\Var(\Delta_i^2)+4\Var(\Delta_i\xi_i).
  \end{align*}
  Each term is self-bounded by $\cL_\mu^{\mathrm{RM}}(r)$:
  \begin{align*}
    \Var(\Delta_i^2)\le \EE[\Delta_i^4]\le H^2   \EE[\Delta_i^2]=H^2   \cL_\mu^{\mathrm{RM}}(r),
  \end{align*}
  and, using the tower property,
  \begin{align*}
    \Var(\Delta_i\xi_i)
    =\EE[\Delta_i^2\xi_i^2]
    =\EE[\Delta_i^2   \EE[\xi_i^2\mid\tau_i]]
    \le H^2   \cL_\mu^{\mathrm{RM}}(r).
  \end{align*}
  Combining,
  \begin{align}
    \Var(U_i(r))\le 5H^2   \cL_\mu^{\mathrm{RM}}(r).
    \label{eq:UiVar}
  \end{align}

  \emph{Range.} Since $|\xi_i|\le H$, $|U_i(r)|\le H^2+2H^2=3H^2$. We use the looser bound $|U_i(r)|\le M:=8H^2$, so that $|U_i(r)-\EE U_i(r)|\le 2M$.

  \emph{Bernstein and self-bounding.} Apply one-sided Bernstein to the i.i.d.\ variables $\{U_i(r)\}_{i=1}^n$ with confidence $\eta\in(0,1)$, using \cref{eq:UiMean}--\cref{eq:UiVar} and the centered range bound $2M$: with probability at least $1-\eta$,
  \begin{align*}
    \cL_\mu^{\mathrm{RM}}(r)-\bigl[\cL_\cD^{\mathrm{RM}}(r)-\cL_\cD^{\mathrm{RM}}(r^\star)\bigr]
    \le \sqrt{\frac{10H^2   \cL_\mu^{\mathrm{RM}}(r)   \log(1/\eta)}{n}}+\frac{2M\log(1/\eta)}{3n}.
  \end{align*}
  Taking $\eta=\delta/|\cR|$ and a union bound over $r\in\cR$, set
  \[
    x:=\cL_\mu^{\mathrm{RM}}(r),\qquad
    T:=\cL_\cD^{\mathrm{RM}}(r)-\cL_\cD^{\mathrm{RM}}(r^\star),\qquad
    \beta':=4H^2\frac{\log(|\cR|/\delta)}{n}.
  \]
  Since $10H^2\log(|\cR|/\delta)/n\le 8\beta'$ and $2M\log(|\cR|/\delta)/(3n)=4\beta'/3$, the Bernstein inequality implies
  \begin{align}
    x\le T+\sqrt{8\beta' x}+\frac{4}{3}\beta'.
    \label{eq:self-bound-x-rm}
  \end{align}
  Finally, $\sqrt{8\beta' x}\le x/2+4\beta'$ by AM-GM. Substituting this into \cref{eq:self-bound-x-rm} and moving $x/2$ to the left gives
  \[
    x\le 2T+\frac{32}{3}\beta'
    =
    2T+\frac{128}{3}\cdot\frac{H^2\log(|\cR|/\delta)}{n},
  \]
  which is \cref{eq:rm-bernstein-solved}.
\end{proof}

For a critic $f$ and policy $\pi$, write $V_{f,h}^{\pi}(s):=f_h(s,\pi_h)$ and set $V_{f,H+1}^{\pi}\equiv0$. In the outcome-based setting, the Bellman error in \cref{eq:loss_bellman_emp} is computed by plugging in the candidate reward $r$; no observed per-step reward is used.

\begin{lemma}[Bellman-error concentration]
  \label{lem:bellman-error-concentration-uniform}\label{lemma:bellman-error-concentration}
  Assume $\Pi,\cR,\cF$ are finite, every $r\in\cR$ satisfies $0\le r_h(s,a)\le 1$, every $f\in\cF$ satisfies $0\le f_h(s,a)\le H$, and uniform approximate Bellman completeness holds with error $\varepsilon_{\cF,\cF}$ (\Cref{assumption:completeness}). Then with probability at least $1-\delta$, the following holds simultaneously for all $(\pi,r,f)\in \Pi\times\cR\times\cF$:
  \begin{align*}
    \cL_\mu^{\mathrm{BE}}(\pi,r,f)
    \le
    2   \cL_{\cD}^{\mathrm{BE}}(\pi,r,f)
    +   \varepsilon_{\mathrm{BE}}
    +   4   \varepsilon_{\cF,\cF},
    \qquad
    \varepsilon_{\mathrm{BE}}
    \coloneqq
    \frac{256}{3}\cdot\frac{H^3   \log \bigl(2|\Pi||\cR||\cF|/\delta\bigr)}{n}.
  \end{align*}
\end{lemma}

\begin{proof}
  Fix any $(\pi,r,f)\in \Pi\times\cR\times\cF$. For any bounded $g$, define the population and empirical unminimized Bellman-regression residuals
  \begin{align*}
    L_\mu(g)
     & :=
    \sum_{h=1}^H
    \EE_\mu   \left[
                \bigl(g_h(s_h,a_h)-r_h(s_h,a_h)-f_{h+1}(s_{h+1},\pi)\bigr)^2
                \right], \\
    L_\cD(g)
     & :=
    \frac1n\sum_{i=1}^n\sum_{h=1}^H
    \bigl(g_h(s_{i,h},a_{i,h})-r_h(s_{i,h},a_{i,h})-f_{h+1}(s_{i,h+1},\pi)\bigr)^2.
  \end{align*}
  The key observation is that $r_h(s_h,a_h)$ is deterministic given $(s_h,a_h)$, so conditioning on $(s_h,a_h)$,
  \begin{align*}
    \EE   \left[
            r_h(s_h,a_h)+V_{f,h+1}^{\pi}(s_{h+1})
            \middle|    s_h,a_h
            \right]
    =
    r_h(s_h,a_h)+\EE[V_{f,h+1}^\pi(s_{h+1})\mid s_h,a_h]
    =
    (\cT_r^\pi f)_h(s_h,a_h),
  \end{align*}
  and moreover $\Var(r_h(s_h,a_h)+V_{f,h+1}^\pi(s_{h+1})\mid s_h,a_h)=\Var(V_{f,h+1}^\pi(s_{h+1})\mid s_h,a_h)$, since adding a $(s_h,a_h)$-measurable constant leaves the conditional variance unchanged. Hence the bias--variance decomposition gives
  \begin{align*}
     & \EE   \left[
               \bigl(
               g_h(s_h,a_h)-r_h(s_h,a_h)-V_{f,h+1}^{\pi}(s_{h+1})
               \bigr)^2
               \middle|    s_h,a_h
               \right] \\
     & \qquad=
    \bigl(
    g_h(s_h,a_h)-(\cT_r^\pi f)_h(s_h,a_h)
    \bigr)^2
    +
    \Var \left(
    V_{f,h+1}^{\pi}(s_{h+1})
    \middle|    s_h,a_h
    \right).
  \end{align*}
  Taking expectation over $(s_h,a_h)\sim d_h^\mu$ and summing over $h=1,\dots,H$, we obtain
  \begin{align}
    L_\mu(g)
    =
    \|g-\cT_r^\pi f\|_{2,\mu}^2
    +
    \sum_{h=1}^H
    \EE_{d_h^\mu}   \left[
                      \Var \left(
                      V_{f,h+1}^{\pi}(s_{h+1})
                      \middle|    s_h,a_h
                      \right)
                      \right].
    \label{eq:population-bias-variance-uniform}
  \end{align}
  The second term does not depend on $g$. Hence, setting $g=f$ and $g=\cT_r^\pi f$ in \cref{eq:population-bias-variance-uniform}, we obtain
  \begin{align}
    L_\mu(f)-L_\mu(\cT_r^\pi f)
    =
    \|f-\cT_r^\pi f\|_{2,\mu}^2
    =
    \cL_\mu^{\mathrm{BE}}(\pi,r,f),
    \label{eq:population-gap-identity-uniform}
  \end{align}

  For each trajectory $i$, define
  \begin{align*}
    Z_i(\pi,r,f)
     & :=
    \sum_{h=1}^H
    \Bigl[
      \ell_{i,h}(f;\pi,r,f)-\ell_{i,h}(\cT_r^\pi f;\pi,r,f)
      \Bigr],
  \end{align*}
  where
  \begin{align*}
    \ell_{i,h}(g;\pi,r,f)
    :=
    \bigl(
    g_h(s_{i,h},a_{i,h})-r_h(s_{i,h},a_{i,h})-V_{f,h+1}^{\pi}(s_{i,h+1})
    \bigr)^2.
  \end{align*}
  Since the trajectories are i.i.d., the variables $\{Z_i(\pi,r,f)\}_{i=1}^n$ are i.i.d. Moreover, by \cref{eq:population-gap-identity-uniform},
  \begin{align}
    \EE[Z_i(\pi,r,f)]
    =
    L_\mu(f)-L_\mu(\cT_r^\pi f)
    =
    \cL_\mu^{\mathrm{BE}}(\pi,r,f).
    \label{eq:mean-of-Z-uniform}
  \end{align}

  We now establish a self-bounding variance inequality. For each $(i,h)$, define
  \begin{align*}
    A_{i,h}
     & :=
    f_h(s_{i,h},a_{i,h})
    -
    r_h(s_{i,h},a_{i,h})
    -
    V_{f,h+1}^{\pi}(s_{i,h+1}), \\
    B_{i,h}
     & :=
    (\cT_r^\pi f)_h(s_{i,h},a_{i,h})
    -
    r_h(s_{i,h},a_{i,h})
    -
    V_{f,h+1}^{\pi}(s_{i,h+1}).
  \end{align*}
  Then
  \begin{align*}
    \ell_{i,h}(f;\pi,r,f)-\ell_{i,h}(\cT_r^\pi f;\pi,r,f)
    =
    A_{i,h}^2-B_{i,h}^2.
  \end{align*}
  Since
  \begin{align*}
    0\le f_h(s,a)\le H,\qquad 0\le(\cT_r^\pi f)_h(s,a)\le H+1\le 2H,
    \qquad
    0\le r_h(s,a)\le 1,
    \qquad
    0\le V_{f,h+1}^{\pi}(s')\le H,
  \end{align*}
  we have
  \begin{align*}
    |A_{i,h}|,\ |B_{i,h}|
    \le
    2H.
  \end{align*}
  Therefore
  \begin{align*}
    |A_{i,h}^2-B_{i,h}^2|
     & =
    |A_{i,h}-B_{i,h}|\cdot|A_{i,h}+B_{i,h}| \\
     & \le
    4H   |A_{i,h}-B_{i,h}|.
  \end{align*}
  Since
  \begin{align*}
    A_{i,h}-B_{i,h}
    =
    f_h(s_{i,h},a_{i,h})-(\cT_r^\pi f)_h(s_{i,h},a_{i,h}),
  \end{align*}
  it follows that
  \begin{align}
    \bigl(A_{i,h}^2-B_{i,h}^2\bigr)^2
    \le
    16H^2
    \bigl(
    f_h(s_{i,h},a_{i,h})-(\cT_r^\pi f)_h(s_{i,h},a_{i,h})
    \bigr)^2.
    \label{eq:per-step-square-bound-uniform}
  \end{align}
  Using Cauchy--Schwarz $(\sum_{h=1}^H a_h)^2\le H\sum_{h=1}^H a_h^2$,
  \begin{align*}
    Z_i(\pi,r,f)^2
     & =
    \left(
    \sum_{h=1}^H
    \bigl[
      \ell_{i,h}(f;\pi,r,f)-\ell_{i,h}(\cT_r^\pi f;\pi,r,f)
      \bigr]
    \right)^2                                              \\
     & \le
    H\sum_{h=1}^H
    \bigl[
      \ell_{i,h}(f;\pi,r,f)-\ell_{i,h}(\cT_r^\pi f;\pi,r,f)
      \bigr]^2 \\
     & \le
    16H^3
    \sum_{h=1}^H
    \bigl(
    f_h(s_{i,h},a_{i,h})-(\cT_r^\pi f)_h(s_{i,h},a_{i,h})
    \bigr)^2,
  \end{align*}
  where the last step used \cref{eq:per-step-square-bound-uniform}. Taking expectation,
  \begin{align}
    \Var(Z_i(\pi,r,f))
     & \le
    \EE[Z_i(\pi,r,f)^2]\notag                              \\
     & \le
    16H^3
    \sum_{h=1}^H
    \EE_{d_h^\mu}   \left[
                      \bigl(
                      f_h(s_h,a_h)-(\cT_r^\pi f)_h(s_h,a_h)
                      \bigr)^2
                      \right]\notag \\
     & =
    16H^3   \cL_\mu^{\mathrm{BE}}(\pi,r,f).
    \label{eq:variance-bound-Z-uniform}
  \end{align}
  Also, since $|A_{i,h}|,|B_{i,h}|\le 2H$, each per-step gap satisfies
  \begin{align}
    \bigl|
    \ell_{i,h}(f;\pi,r,f)-\ell_{i,h}(\cT_r^\pi f;\pi,r,f)
    \bigr|
    \le
    8H^2,
  \end{align}
  and summing over $h=1,\dots,H$ gives
  \begin{align}
    |Z_i(\pi,r,f)|
    \le
    8H^3.
    \label{eq:uniform-bound-Z-uniform}
  \end{align}

  Applying Bernstein's inequality to the i.i.d.\ variables $\{Z_i(\pi,r,f)\}_{i=1}^n$, we obtain that for any fixed triple $(\pi,r,f)$ and any $\eta\in(0,1)$, with probability at least $1-\eta$,
  \begin{align}
    \cL_\mu^{\mathrm{BE}}(\pi,r,f)
    \le
    \frac1n\sum_{i=1}^n Z_i(\pi,r,f)
    +
    \sqrt{
      \frac{
        32H^3
        \cL_\mu^{\mathrm{BE}}(\pi,r,f)
        \log(1/\eta)
      }{n}
    }
    +
    \frac{16H^3\log(1/\eta)}{3n}.
    \label{eq:bernstein-BE-uniform}
  \end{align}

  Setting $\eta=\delta/(2|\Pi||\cR||\cF|)$ and taking a union bound over all triples $(\pi,r,f)\in\Pi\times\cR\times\cF$, with probability at least $1-\delta/2$, simultaneously for all such triples,
  \begin{align}
    \cL_\mu^{\mathrm{BE}}(\pi,r,f)
    \le
    \frac1n\sum_{i=1}^n Z_i(\pi,r,f)
    +
    \sqrt{
      \frac{
        32H^3
        \cL_\mu^{\mathrm{BE}}(\pi,r,f)
        \log(2|\Pi||\cR||\cF|/\delta)
      }{n}
    }
    +
    \frac{16H^3\log(2|\Pi||\cR||\cF|/\delta)}{3n}.
    \label{eq:bernstein-BE-union}
  \end{align}

  We now convert $\frac1n\sum_i Z_i(\pi,r,f)=L_{\cD}(f)-L_{\cD}(\cT_r^\pi f)$ into an empirical quantity computable from $\cF$. Let
  \begin{align*}
    g^\star(\pi,r,f)   \in   \argmin_{g\in\cF}\|g-\cT_r^\pi f\|_{2,\mu}^2,
  \end{align*}
  so that by completeness, $\|g^\star-\cT_r^\pi f\|_{2,\mu}^2\le\varepsilon_{\cF,\cF}$. Since $g^\star\in\cF$,
  \begin{align}
    \frac1n\sum_{i=1}^n Z_i(\pi,r,f)
     & =
    \bigl[L_{\cD}(f)-L_{\cD}(g^\star)\bigr]+\bigl[L_{\cD}(g^\star)-L_{\cD}(\cT_r^\pi f)\bigr]\notag \\
     & \le
    \underbrace{\bigl[L_{\cD}(f)-\min_{g\in\cF}L_{\cD}(g)\bigr]}_{=\cL_{\cD}^{\mathrm{BE}}(\pi,r,f)}
    +\underbrace{\bigl[L_{\cD}(g^\star)-L_{\cD}(\cT_r^\pi f)\bigr]}_{=:T(\pi,r,f)},
    \label{eq:Zi-to-empiricalBE}
  \end{align}
  where we suppress the arguments $(\pi,r,f)$ in $L_{\cD}(\cdot)$ for brevity. The bias term $T(\pi,r,f)$ is handled by the same Bernstein self-bounding argument applied to the pair $(g^\star,\cT_r^\pi f)$ in place of $(f,\cT_r^\pi f)$: replaying \cref{eq:per-step-square-bound-uniform}--\cref{eq:uniform-bound-Z-uniform} with $g^\star$ instead of $f$ gives
  \begin{align*}
    \EE[T(\pi,r,f)]=\|g^\star-\cT_r^\pi f\|_{2,\mu}^2\le\varepsilon_{\cF,\cF},
    \qquad
    \Var \bigl(T(\pi,r,f)\text{ per sample}\bigr)\le 16H^3   \|g^\star-\cT_r^\pi f\|_{2,\mu}^2\le 16H^3   \varepsilon_{\cF,\cF}.
  \end{align*}
  Writing $\beta:=4H^3\log(2|\Pi||\cR||\cF|/\delta)/n$, Bernstein's inequality (union-bounded over $(\pi,r,f)\in\Pi\times\cR\times\cF$, which also fixes $g^\star$) yields, with probability at least $1-\delta/2$,
  \begin{align}
    T(\pi,r,f)
    \le   \varepsilon_{\cF,\cF}
    +\sqrt{8\beta   \varepsilon_{\cF,\cF}}
    +\frac{4}{3}\beta
    \le   2\varepsilon_{\cF,\cF}+\frac{16}{3}\beta,
    \label{eq:T-concentration}
  \end{align}
  where the last step uses AM--GM $\sqrt{8\beta   \varepsilon_{\cF,\cF}}\le\varepsilon_{\cF,\cF}+2\beta$ and then loosens the constant.

  Combining \cref{eq:bernstein-BE-union}, \cref{eq:Zi-to-empiricalBE} and \cref{eq:T-concentration}, and writing
  \begin{align*}
    x
     & :=
    \cL_\mu^{\mathrm{BE}}(\pi,r,f),   \qquad
    M
    :=
    \cL_{\cD}^{\mathrm{BE}}(\pi,r,f),
  \end{align*}
  we obtain
  \begin{align*}
    x
    \le
    M+2\varepsilon_{\cF,\cF}+\frac{16}{3}\beta+\sqrt{8\beta x}+\frac43\beta.
  \end{align*}
  Using AM--GM, $\sqrt{8\beta x}\le x/2+4\beta$, so
  \begin{align*}
    x
    \le
    M+2\varepsilon_{\cF,\cF}+\frac{x}{2}+4\beta+\frac{16}{3}\beta+\frac43\beta
    =   M+2\varepsilon_{\cF,\cF}+\frac{x}{2}+\frac{32}{3}\beta,
  \end{align*}
  and rearranging yields $x\le 2M+4\varepsilon_{\cF,\cF}+\frac{64}{3}\beta$. Equivalently,
  \begin{align*}
    \cL_\mu^{\mathrm{BE}}(\pi,r,f)
    \le
    2   \cL_{\cD}^{\mathrm{BE}}(\pi,r,f)
    +\varepsilon_{\mathrm{BE}}
    +4   \varepsilon_{\cF,\cF},
  \end{align*}
  as claimed.
\end{proof}

\subsection{Approximation and Comparator Lemmas}\label{subsection:approximation}
This subsection collects the approximation statements associated with the comparator pair $(f_\pi,r^\star)$ used in the pessimism step.
We use the Bellman operator $\cT_{r^\star}^\pi$ from \cref{eq:def-bellman-op}. The constants $V_{\max}$ and $\varepsilon_\cF$ refer, respectively, to the uniform critic bound and the critic-realizability error in \Cref{assumption:realizable-f}.
\begin{lemma}[Adapted from Theorem 8 in \citet{cheng2022adversarially}]\label{lemma:bounded-em-Bellman-error}
  For any $\pi\in\Pi$, let $f_\pi$ be defined as follows,
  \begin{align*}
    f_\pi\coloneqq \argmin_{f\in\cF}\sup_{h,\nu}\Vert f_h-(\cT_{r^\star}^\pi f)_h\Vert_{2,\nu}^2.
  \end{align*}
  Under \Cref{assumption:realizable-f,assumption:completeness}, the following inequality holds with probability at least $1-\delta$ for all $\pi\in\Pi$:
  \begin{align*}
    \cL_{\cD}^{\mathrm{BE}}(\pi,r^\star,f_\pi)
    \leq
    \varepsilon_{\mathrm{apx}},
    \qquad
    \varepsilon_{\mathrm{apx}}
    \coloneqq
    O \bigg(\frac{H   V^2_{\mathrm{max}}\log (|\cF|   |\Pi|/\delta)}{n}+\varepsilon_\cF\bigg).
  \end{align*}

\end{lemma}
The factor $H$ multiplying $V_{\max}^2$ in \Cref{lemma:bounded-em-Bellman-error} comes from the sum-over-$h$ convention in \cref{eq:all_losses}: since rewards and critics are bounded (with $V_{\max}=H$ in this section), each per-step squared residual is $O(V_{\max}^2)$, so one trajectory contributes $O(HV_{\max}^2)$. This is the finite-horizon analogue of the per-step residual bound in \citet[Theorem 8]{cheng2022adversarially}; for infinite classes, the same argument uses covering numbers in place of $|\cF|$ and $|\Pi|$.

\begin{lemma}[Finite-horizon performance difference lemma~\citep{kakade2002approximately}]\label{lemma:performance-difference}
  For any fixed per-step reward function $r$ and any policies $\pi,\pi'\in\Pi$, let $J_r(\pi)\coloneqq\EE_{\tau\sim\pi}[R(\tau;r)]$. Then
  \begin{align*}
    J_r(\pi)-J_r(\pi')
    =
    \sum_{h=1}^{H}\EE_{s_h,a_h\sim d_h^\pi}
    \bigl[Q_h^{\pi',r}(s_h,a_h)-V_h^{\pi',r}(s_h)\bigr].
  \end{align*}
  In particular, when $r=r^\star$, this reduces to the same identity under the shorthand notation introduced in \Cref{section:preliminary}.
\end{lemma}
\begin{lemma}[Finite-horizon analogue of Eq.~(16) in \citet{cheng2022adversarially}]\label{lemma:bound-diff-Q-and-f_pi}
  Under \Cref{assumption:realizable-f}, for every $\pi_k\in\Pi$,
  \begin{align*}
    \bigl\vert \cL_{\mu}(\pi_k,Q^{\pi_k})-\cL_\mu(\pi_k,f_{\pi_k})\bigr\vert    \leq    2H\sqrt{\varepsilon_\cF}   =   O(H\sqrt{\varepsilon_\cF}).
  \end{align*}

\end{lemma}

\begin{proof}
  Write $f\coloneqq f_{\pi_k}$, $\pi\coloneqq\pi_k$ and $\cT^\pi\coloneqq\cT_{r^\star}^\pi$ throughout the proof. Define the \emph{fake per-step reward induced by $(f,\pi)$},
  \begin{align*}
    \widetilde r_h(s_h,a_h)   \coloneqq    f_h(s_h,a_h)-\EE_{s_{h+1}\sim P(\cdot\mid s_h,a_h)}\bigl[f_{h+1}(s_{h+1},\pi(s_{h+1}))\bigr],\qquad h\in[H],
  \end{align*}
  so that $R(\tau;\widetilde r)=\sum_{h=1}^H\widetilde r_h(s_h,a_h)$, with $f_{H+1}\equiv 0$. By construction $f$ satisfies the Bellman equation under reward $\widetilde r$, so $f_h=Q_h^{\pi,\widetilde r}$ and, letting $J_{\widetilde r}(\cdot)$ denote the return in the MDP with reward $\widetilde r$,
  \begin{align}\label{eq:fake-reward-identity}
    J_{\widetilde r}(\pi)=f_1(s_1,\pi),\qquad J_{\widetilde r}(\mu)=\EE_{\tau\sim\mu}[R(\tau;\widetilde r)].
  \end{align}

  \emph{Step 1: two applications of the performance difference lemma.} By \Cref{lemma:performance-difference} applied in the true MDP,
  \begin{align}\label{eq:pdl-true}
    \cL_\mu(\pi,Q^{\pi})   =   \sum_{h=1}^H\EE_{d^\mu_h}   \bigl[V_h^{\pi}(s_h)-Q_h^{\pi}(s_h,a_h)\bigr]   =   J(\pi)-J(\mu).
  \end{align}
  By \Cref{lemma:performance-difference} applied in the fake MDP (with reward $\widetilde r$, same dynamics) and using \cref{eq:fake-reward-identity},
  \begin{align}\label{eq:pdl-fake}
    \cL_\mu(\pi,f)   =   \sum_{h=1}^H\EE_{d^\mu_h}   \bigl[f_h(s_h,\pi)-f_h(s_h,a_h)\bigr]   =   J_{\widetilde r}(\pi)-J_{\widetilde r}(\mu)   =   f_1(s_1,\pi)-\EE_{\mu}[R(\tau;\widetilde r)].
  \end{align}
  Subtracting \cref{eq:pdl-fake} from \cref{eq:pdl-true},
  \begin{align}\label{eq:master-diff}
    \cL_\mu(\pi,Q^{\pi})-\cL_\mu(\pi,f)   =   \underbrace{\bigl[J(\pi)-f_1(s_1,\pi)\bigr]}_{=:A}   +   \underbrace{\bigl[\EE_{\mu}[R(\tau;\widetilde r)]-J(\mu)\bigr]}_{=:B}.
  \end{align}

  \emph{Step 2: express $A$ and $B$ as Bellman residuals of $f$.} By the definition of $\widetilde r$ and $\cT^\pi f_h(s_h,a_h)=r^\star(s_h,a_h)+\EE_{s_{h+1}}[f_{h+1}(s_{h+1},\pi)]$,
  \begin{align*}
    r^\star(s_h,a_h)-\widetilde r_h(s_h,a_h)   =   \cT^\pi f_h(s_h,a_h)-f_h(s_h,a_h).
  \end{align*}
  Taking expectations under $d^\pi_h$ and $d^\mu_h$ respectively and summing over $h$,
  \begin{align}
    A   =   \EE_{\tau\sim\pi}[R^\star(\tau)-R(\tau;\widetilde r)]   =   \sum_{h=1}^H\EE_{d^\pi_h}   \bigl[(\cT^\pi f_h-f_h)(s_h,a_h)\bigr],\label{eq:A-bellman} \\
    B   =   \EE_{\tau\sim\mu}[R(\tau;\widetilde r)-R^\star(\tau)]   =   \sum_{h=1}^H\EE_{d^\mu_h}   \bigl[(f_h-\cT^\pi f_h)(s_h,a_h)\bigr].\label{eq:B-bellman}
  \end{align}

  \emph{Step 3: bound via Assumption~\ref{assumption:realizable-f}.} Combining \cref{eq:master-diff}--\cref{eq:B-bellman} and applying Jensen's inequality ($\|\cdot\|_{1,\nu}\le\|\cdot\|_{2,\nu}$) term-by-term,
  \begin{align*}
    \bigl|\cL_\mu(\pi,Q^{\pi})-\cL_\mu(\pi,f)\bigr|
    \le   \sum_{h=1}^H\bigl\|f_h-\cT^\pi f_h\bigr\|_{2,d^\pi_h}
    +\sum_{h=1}^H\bigl\|f_h-\cT^\pi f_h\bigr\|_{2,d^\mu_h}.
  \end{align*}
  Since $f=f_\pi$ attains the $\argmin$ in the definition of $f_\pi$, \Cref{assumption:realizable-f} gives
  \[
    \sup_{h,\nu}
    \|f_h-(\cT^\pi f)_h\|_{2,\nu}^2
    \le\varepsilon_\cF.
  \]
  Taking $\nu=d_h^\pi$ and $\nu=d_h^\mu$ gives $\|f_h-\cT^\pi f_h\|_{2,\nu}^2\le\varepsilon_\cF$ for each $h\in[H]$ and each $\nu\in\{d^\pi_h,d^\mu_h\}$. Summing $2H$ such terms gives
  \begin{align*}
    \bigl|\cL_\mu(\pi,Q^{\pi})-\cL_\mu(\pi,f)\bigr|   \le   2H\sqrt{\varepsilon_\cF},
  \end{align*}
  which is the claim. \qedhere
\end{proof}

We recall the policy-loss notation used in the decomposition. For a distribution $\mu$,
\begin{align*}
  \cL_\mu(\pi,f) & \coloneqq \sum_{h=1}^{H}\EE_{s_h,a_h\sim d^\mu_h}\big[f_h(s_h,\pi(s_h))-f_h(s_h,a_h)\big].
\end{align*}
\begin{lemma}[Regret Decomposition]\label{lemma:regret-decomposition}
  Let $\pi$ be an arbitrary competitor policy, $\widehat\pi\in\Pi$ be some learned policy, $\widehat r\in\cR$ be some learned reward function, and $f\in\cF$ be an arbitrary function over $\cS\times\cA$. Then we have the following decomposition:
  \begin{align*}
    J(\pi)-J(\widehat\pi) & =\sum_{h=1}^H\EE_{s_h,a_h\sim d^\pi_h}\big[\cT_{\rhat}^{\widehat\pi}f_h(s_h,a_h)-f_h(s_h,a_h)\big] +\sum_{h=1}^H\EE_{s_h,a_h\sim d^\mu_h}\big[f_h(s_h,a_h)-\cT_{\rhat}^{\widehat\pi}f_h(s_h,a_h)\big]\notag \\
                          & +\sum_{h=1}^H\EE_{s_h,a_h\sim d^\pi_h}\big[f_h(s_h,a_h)-f_h(s_h,\widehat\pi(s_h))\big] - \cL_\mu(\widehat\pi,Q^{\widehat\pi})+\cL_\mu(\widehat\pi,f)                                                        \\
                          & + \sum_{h=1}^H\EE_{s_h,a_h\sim d^\pi_h}\big[r^\star(s_h,a_h)-\rhat(s_h,a_h)\big]+\sum_{h=1}^H\EE_{s_h,a_h\sim d^\mu_h}\big[\rhat(s_h,a_h)-r^\star(s_h,a_h)\big].
  \end{align*}
\end{lemma}
\begin{proof}
  Define $r_h^{f,\widehat \pi}(s_h,a_h)\coloneqq f_h(s_h,a_h)-\EE_{s_{h+1}}[f_{h+1}(s_{h+1},\widehat\pi)]$. Then $r^{f,\widehat\pi}=(r^{f,\widehat\pi}_1,\ldots,r^{f,\widehat\pi}_H)$ is a fake reward function given $f$ and $\widehat\pi$, with trajectory return $R(\tau;r^{f,\widehat\pi})$. We use the subscript $(\cdot)_{r^{f,\widehat \pi}}$ to denote functions or operators under another MDP $\cM_{r^{f,\widehat \pi}}$, which have the same dynamics with $\cM$. The only difference is $\cM_{r^{f,\widehat \pi}}$ has a reward function $r^{f,\widehat \pi}$.
  As the policy $\pi(\cdot\mid s)$ and the transition kernel $P(\cdot\mid s,a)$ has no relationship with rewards, we get the same distribution $d^\pi$, no matter under MDP $\cM$ or $\cM_{r^{f,\widehat\pi}}$.

  Since $f$ satisfies the Bellman equation with respect to $\cM_{r^{f,\widehat \pi}}$, i.e. $f_h(s_h,a_h)=r_h^{f,\widehat \pi}(s_h,a_h)+\EE_{s_{h+1}}[f_{h+1}(s_{h+1},\widehat\pi)]$, we know $f_h=Q_h^{\widehat\pi,r^{f,\widehat \pi}}$, for $h\in[H]$.

  We perform a performance decomposition:
  \begin{align}
    J(\pi)-J(\widehat\pi)=J(\pi)-J(\mu)-(J(\widehat\pi)-J(\mu)).\label{eq:first-decomposition}
  \end{align}
  By performance difference lemma \ref{lemma:performance-difference},
  \begin{align*}
    J(\widehat\pi)-J(\mu) & = - \big(J(\mu)-J(\widehat\pi)\big)                                                                                \\
                          & = -\bigg(\sum_{h=1}^{H} \EE_{s_h,a_h\sim d^\mu_h}\big[Q_h^{\widehat\pi}(s_h,a_h)-V_h^{\widehat\pi}(s_h)\big]\bigg) \\
                          & =\sum_{h=1}^{H} \EE_{s_h,a_h\sim d^\mu_h}\big[V_h^{\widehat\pi}(s_h)-Q_h^{\widehat\pi}(s_h,a_h)\big]               \\
                          & = \cL_\mu(\widehat\pi,Q^{\widehat\pi}).
  \end{align*}
  Define $\Delta(\widehat\pi)\coloneqq \cL_\mu(\widehat\pi,Q^{\widehat\pi})-\cL_\mu(\widehat\pi,f)$, then we can
  rewrite the second term of \cref{eq:first-decomposition} as
  \begin{align*}
    J(\widehat\pi)-J(\mu) & =\cL_\mu(\widehat\pi,Q^{\widehat\pi})                                                                                                                         \\
                          & = \Delta(\widehat\pi) +\cL_\mu(\widehat\pi,f)                                                                                                                 \\
                          & =\Delta(\widehat\pi)+\sum_{h=1}^{H}\EE_{s_h,a_h\sim d^\mu_h}\big[f_h(s_h,\widehat\pi(s_h))-f_h(s_h,a_h)\big]                                                  \\
                          & =\Delta(\widehat\pi)+\sum_{h=1}^{H}\EE_{s_h,a_h\sim d^\mu_h}\big[V_h^{\widehat\pi,r^{f,\widehat \pi}}(s_h)-Q_h^{\widehat\pi,r^{f,\widehat \pi}}(s_h,a_h)\big] \\
                          & =\Delta(\widehat\pi)+ J_{r^{f,\widehat\pi}}(\widehat\pi)-J_{r^{f,\widehat\pi}}(\mu)\tag{Performance difference lemma \ref{lemma:performance-difference}}      \\
                          & =\Delta(\widehat\pi)+V_1^{\widehat\pi,r^{f,\widehat\pi}}(s_1)-\EE_{\tau\sim\mu}[R(\tau;r^{f,\widehat\pi})]                                                    \\
                          & =\Delta(\widehat\pi)+f_1(s_1,\widehat\pi)-\EE_{\tau\sim\mu}[R(\tau;r^{f,\widehat\pi})].
  \end{align*}
  Substituting this identity into \cref{eq:first-decomposition} gives
  \begin{align}
    J(\pi)-J(\widehat\pi)
    =\bigl(J(\pi)-f_1(s_1,\widehat\pi)\bigr)
    +\bigl(\EE_{\tau\sim\mu}[R(\tau;r^{f,\widehat\pi})]-J(\mu)\bigr)
    -\Delta(\widehat\pi).
    \label{eq:regret-decomposition-0}
  \end{align}
  By the definition of the expected return $J(\mu)$,
  \begin{align}
    \EE_{\tau\sim\mu}[R(\tau;r^{f,\widehat\pi})]-J(\mu) & =\EE_{\tau\sim\mu}\big[R(\tau;r^{f,\widehat\pi})-R^\star(\tau)\big]\notag                                                                                  \\
                                                        & =\sum_{h=1}^H\EE_{s_h,a_h\sim d^\mu_h}\big[r^{f,\widehat\pi}(s_h,a_h)-r^\star(s_h,a_h)\big]    \notag                                                      \\
                                                        & =\sum_{h=1}^H \EE_{s_h,a_h\sim d^\mu_h}[f_h(s_h,a_h)-\EE_{s_{h+1}}[f_{h+1}(s_{h+1},\widehat\pi)]-r^\star(s_h,a_h)] \tag{Definition of $r^{f,\widehat\pi}$} \\
                                                        & =\sum_{h=1}^H\EE_{s_h,a_h\sim d^\mu_h}\big[f_h(s_h,a_h)-\cT^{\widehat\pi}(s_h,a_h)\big].\label{eq:diff1}
  \end{align}
  Since $J(\pi)=\sum_{h=1}^H \EE_{d^\pi_h} [r^\star(s_h,a_h)]=\EE_{\tau\sim\pi}[R^\star(\tau)]$,
  \begin{align}
    J(\pi)-f_1(s_1,\widehat\pi) & =(J(\pi)-\EE_{\tau\sim\pi}[R(\tau;r^{f,\widehat\pi})])+(\EE_{
                                                                                            \tau\sim\pi}[R(\tau;r^{f,\widehat\pi})]-f_1(s_1,\widehat\pi))\notag                                                                                                       \\
                                & =\sum_{h=1}^H\EE_{s_h,a_h\sim d^\pi_h}[r^\star(s_h,a_h)-r^{f,\widehat\pi}(s_h,a_h)]+J_{r^{f,\widehat\pi}}(\pi)-J_{r^{f,\widehat\pi}}(\widehat\pi)\notag                                                                             \\
                                & =\sum_{h=1}^H\EE_{s_h,a_h\sim d^\pi_h}\big[\cT^{\widehat\pi}f(s_h,a_h)-f(s_h,a_h)\big] +\sum_{h=1}^H\EE_{s_h,a_h\sim d^\pi_h}\big[Q_h^{\widehat\pi,r^{f,\widehat\pi}}(s_h,a_h)-V_h^{\widehat\pi,r^{f,\widehat\pi}}(s_h)\big] \notag \\
                                & = \sum_{h=1}^H\EE_{s_h,a_h\sim d^\pi_h}\big[\cT^{\widehat\pi}f(s_h,a_h)-f(s_h,a_h)\big] +\sum_{h=1}^H\EE_{s_h,a_h\sim d^\pi_h}\big[f_h(s_h,a_h)-f_h(s_h,\widehat\pi(s_h))\big].
    \label{eq:diff2}
  \end{align}

  Putting \cref{eq:diff1},\cref{eq:diff2} into \cref{eq:regret-decomposition-0}, we get
  \begin{align}
    J(\pi)-J(\widehat\pi) & = \sum_{h=1}^H\EE_{s_h,a_h\sim d^\pi_h}\big[\cT^{\widehat\pi}f_h(s_h,a_h)-f_h(s_h,a_h)\big] +\sum_{h=1}^H\EE_{s_h,a_h\sim d^\pi_h}\big[f_h(s_h,a_h)-f_h(s_h,\widehat\pi(s_h))\big]\notag \\
                          & +\sum_{h=1}^H\EE_{s_h,a_h\sim d^\mu_h}\big[f_h(s_h,a_h)-\cT^{\widehat\pi}f_h(s_h,a_h)\big] - \cL_\mu(\widehat\pi,Q^{\widehat\pi})+\cL_\mu(\widehat\pi,f).\label{eq:decomposition-final}
  \end{align}
  For a given learned reward function $\rhat$, we can write
  \begin{align*}
    \cT^{\pihat}f_h(s_h,a_h) & =r^\star(s_h,a_h)+\EE_{s_{h+1}\sim P(\cdot\mid s_h,a_h)} \big[f_{h+1}(s_{h+1},\pihat(s_{h+1}))\big]                                  \\
                             & = \rhat(s_h,a_h)+\EE_{s_{h+1}\sim P(\cdot\mid s_h,a_h)} \big[f_{h+1}(s_{h+1},\pihat(s_{h+1}))\big] + r^\star(s_h,a_h)-\rhat(s_h,a_h) \\
                             & = \cT_{\rhat}^\pihat f_h(s_h,a_h)+r^\star(s_h,a_h)-\rhat(s_h,a_h).
  \end{align*}
  Plugging this into \cref{eq:decomposition-final}, we get
  \begin{align*}
    J(\pi)-J(\widehat\pi) & =\sum_{h=1}^H\EE_{s_h,a_h\sim d^\pi_h}\big[\cT_{\rhat}^{\widehat\pi}f_h(s_h,a_h)-f_h(s_h,a_h)\big] +\sum_{h=1}^H\EE_{s_h,a_h\sim d^\mu_h}\big[f_h(s_h,a_h)-\cT_{\rhat}^{\widehat\pi}f_h(s_h,a_h)\big]\notag \\
                          & +\sum_{h=1}^H\EE_{s_h,a_h\sim d^\pi_h}\big[f_h(s_h,a_h)-f_h(s_h,\widehat\pi(s_h))\big] - \cL_\mu(\widehat\pi,Q^{\widehat\pi})+\cL_\mu(\widehat\pi,f)                                                        \\
                          & + \sum_{h=1}^H\EE_{s_h,a_h\sim d^\pi_h}\big[r^\star(s_h,a_h)-\rhat(s_h,a_h)\big]+\sum_{h=1}^H\EE_{s_h,a_h\sim d^\mu_h}\big[\rhat(s_h,a_h)-r^\star(s_h,a_h)\big].
  \end{align*}
\end{proof}

\section{Proof of the Outcome-Based Offline RL Upper Bound}\label{section:proof-additive-upper-bound}
In this section we prove \Cref{theorem:outcome-based-offline-rl}, using the population losses defined in \cref{eq:all_losses}.
\subsection{Regret Decomposition}
Our objective is to find an upper bound for
\begin{align*}
  J(\pi)-J(\bar\pi) = \frac{1}{K}\sum_{k=1}^K\big( J(\pi)-J(\pi_k)\big).
\end{align*}
By \Cref{lemma:regret-decomposition}, for any $\pi_k\in\Pi$, $f_k\in\cF$, and $r_k\in\cR$ learned in the algorithm, we have
\begin{align}
  J(\pi)-J(\pi_k) & =\underbrace{\sum_{h=1}^H\EE_{s_h,a_h\sim d^\pi_h}\big[\cT_{r_k}^{\pi_k}f_{k,h}(s_h,a_h)-f_{k,h}(s_h,a_h)\big] +\sum_{h=1}^H\EE_{s_h,a_h\sim d^\mu_h}\big[f_{k,h}(s_h,a_h)-\cT_{r_k}^{\pi_k}f_{k,h}(s_h,a_h)\big]}_{(\mathrm{I})_k}\notag \\
                  & +\underbrace{\sum_{h=1}^H\EE_{s_h,a_h\sim d^\pi_h}\big[f_{k,h}(s_h,a_h)-f_{k,h}(s_h,\pi_k(s_h))\big]}_{(\mathrm{II})_k}+\underbrace{\cL_\mu(\pi_k,f_k) - \cL_\mu(\pi_k,Q^{\pi_k})}_{(\mathrm{III})_k}\notag                               \\
                  & + \underbrace{\sum_{h=1}^H\EE_{s_h,a_h\sim d^\pi_h}\big[r^\star(s_h,a_h)-r_k(s_h,a_h)\big]+\sum_{h=1}^H\EE_{s_h,a_h\sim d^\mu_h}\big[r_k(s_h,a_h)-r^\star(s_h,a_h)\big]}_{(\mathrm{IV})_k}.\label{eq:regret-decomposition}
\end{align}

\subsection{Bounding the Policy-Optimization Term}
Term $(\mathrm{II})_k=\sum_{h=1}^H\EE_{s_h,a_h\sim d^\pi_h}\big[f_{k,h}(s_h,a_h)-f_{k,h}(s_h,\pi_k(s_h))\big]$ is controlled by the no-regret property of the layer-wise softmax update \cref{eq:alg-policy-optimization}. Writing $d_{h,S}^\pi(s)\coloneqq\sum_a d_h^\pi(s,a)$ for the state marginal, we have
\begin{align*}
  (\mathrm{II})_k
  =\sum_{h=1}^H\EE_{s\sim d_{h,S}^\pi}   \Big[\EE_{a\sim\pi_h(\cdot\mid s)}[f_{k,h}(s,a)]-\EE_{a\sim\pi_{k,h}(\cdot\mid s)}[f_{k,h}(s,a)]\Big]
  =\mathrm{Reg}_k(\pi),
\end{align*}
because $d_h^\pi(s,a)=d_{h,S}^\pi(s)\pi_h(a\mid s)$. Since every $f_k\in\cF$ satisfies $f_{k,h}(s,a)\in[0,V_{\max}]$, \Cref{lemma:hedge-regret} applies with the comparator $\pi\in\Pi$. Choosing $\eta=\eta^\star= V_{\max}\sqrt{K/(8\log|\cA|)}$, \cref{eq:hedge-optimized} yields
\begin{align}
  \frac{1}{K}\sum_{k=1}^K(\mathrm{II})_k
  \le
  H   V_{\max}\sqrt{\frac{\log|\cA|}{2K}}
  =   \widetilde O   \left(\frac{H   V_{\max}}{\sqrt K}\right).
  \label{eq:term-II-bound}
\end{align}

\subsection{Bounding the Bellman and Reward-Mismatch Terms}
Recall the state--action concentrability $C_{sa}(\pi)$ from \cref{eq:def-Csa}. For any $g:\cS_h\times\cA\to\mathbb R$, the pointwise bound $d_h^\pi(s,a)\le C_{sa}(\pi)d_h^\mu(s,a)$ gives
\begin{align*}
  \EE_{d^\pi_h}[g(s_h,a_h)^2]   \le   C_{sa}(\pi)   \EE_{d^\mu_h}[g(s_h,a_h)^2].
\end{align*}
Applying this inequality, Cauchy--Schwarz, and AM--GM with free parameter $\beta/2>0$ (so that, after the factor-$2$ Bellman-error concentration \cref{eq:conc-BE} below, the pessimism coefficient matches the algorithmic one, $\beta$, in \cref{eq:alg-pessimism}), we have
\begin{align*}
  \EE_{s_h,a_h\sim d^\pi_h}[(\cT_{r_k}^{\pi_k}f_k-f_k)(s_h,a_h)]
   & \leq \frac{C_{sa}(\pi)}{\beta} + \frac{\beta}{4} \EE_{s_h,a_h\sim d^\mu_h}[((\cT_{r_k}^{\pi_k}f_k-f_k)(s_h,a_h))^2], \\
  \EE_{s_h,a_h\sim d^\mu_h}[(f_k-\cT_{r_k}^{\pi_k}f_k)(s_h,a_h)]
   & \leq \frac{1}{\beta}+ \frac{\beta}{4} \EE_{s_h,a_h\sim d^\mu_h}[((\cT_{r_k}^{\pi_k}f_k-f_k)(s_h,a_h))^2].
\end{align*}
Summing over $h\in[H]$,
\begin{align}
  (\mathrm{I})_k\leq \frac{H(C_{sa}(\pi)+1)}{\beta}+\frac{\beta}{2}\sum_{h=1}^H \EE_{s_h,a_h\sim d^\mu_h}\big[((\cT_{r_k}^{\pi_k}f_k-f_k)(s_h,a_h))^2\big].\label{eq:upperbound-I}
\end{align}
Similarly, by Change-of-Trajectory-Measure  \Cref{lemma:change-of-trajectory}
\begin{align*}
  \sup_f\frac{\left(\mathbb{E}_{\tau\sim\pi}\left[ \sum_h f(s_h, a_h)\right]\right)^{2}}{\mathbb{E}_{\tau\sim\mu}\left[\left( \sum_h f(s_h, a_h)\right)^{2}\right]} \leq HC_{sa}(\pi),
\end{align*}
selecting $g_k=r^\star-r_k$,
we have
\begin{align*}
  \sum_{h=1}^H g_k(s_h,a_h) = \sum_{h=1}^H\big( r^\star(s_h,a_h)-r_k(s_h,a_h)\big) = R^\star(\tau)-R(\tau;r_k).
\end{align*}
Therefore, applying Cauchy--Schwarz again with parameter $\beta/2$,
\begin{align*}
  \EE_{\tau\sim\pi} \bigg[\sum_{h=1}^H \big(r^\star(s_h,a_h)-r_k(s_h,a_h)\big)\bigg] & \leq \frac{HC_{sa}(\pi)}{\beta} + \frac{\beta}{4} \EE_{\tau\sim\mu} [(R^\star(\tau)-R(\tau;r_k))^2], \\
  \EE_{\tau\sim\mu}\bigg[\sum_{h=1}^H \big(r_k(s_h,a_h)-r^\star(s_h,a_h)\big)\bigg]  & \leq \frac{1}{\beta} + \frac{\beta}{4} \EE_{\tau\sim\mu} [(R^\star(\tau)-R(\tau;r_k))^2].
\end{align*}
\begin{align}
  (\mathrm{IV})_k\leq \frac{HC_{sa}(\pi)+1}{\beta}+\frac{\beta}{2} \EE_{\tau\sim\mu} [(R^\star(\tau)-R(\tau;r_k))^2].\label{eq:upperbound-IV}
\end{align}
Then by combining \cref{eq:upperbound-I} and \cref{eq:upperbound-IV} we get
\begin{align}
  (\mathrm{I})_k+(\mathrm{III})_k+(\mathrm{IV})_k
   & \leq \frac{H(C_{sa}(\pi)+1)}{\beta}+\frac{\beta}{2}\sum_{h=1}^H \EE_{s_h,a_h\sim d^\mu_h}\big[((\cT_{r_k}^{\pi_k}f_k-f_k)(s_h,a_h))^2\big]\notag                                                                    \\
   & + \frac{HC_{sa}(\pi)+1}{\beta}+\frac{\beta}{2} \EE_{\tau\sim\mu} [(R^\star(\tau)-R(\tau;r_k))^2]\notag                                                                                                              \\
   & + \cL_\mu(\pi_k,f_k) - \cL_\mu(\pi_k,Q^{\pi_k})\notag                                                                                                                                                               \\
   & \leq \frac{2H(C_{sa}(\pi)+1)}{\beta}+\frac{\beta}{2}\cL^{\mathrm{BE}}_\mu(\pi_k,r_k,f_k) + \frac{\beta}{2} \cL_\mu^{\mathrm{RM}}(r_k)+ \cL_\mu(\pi_k,f_k) - \cL_\mu(\pi_k,Q^{\pi_k}),\label{eq:upperbound-I-III-IV}
\end{align}
where in the last inequality we used $HC_{sa}(\pi)+1\le H(C_{sa}(\pi)+1)$ and the definitions \cref{eq:all_losses}.

\subsection{Pessimism and Empirical Transfer}
Now fix the high-probability event on which the concentration and approximation lemmas from \Cref{subsection:concentration,subsection:approximation} hold simultaneously for every $\pi\in\Pi$, $r\in\cR$, and $f\in\cF$. Because the change-of-measure step \cref{eq:upperbound-I-III-IV} produces the coefficient $\beta/2$ on the population BE/RM losses, and the Bellman-error (\Cref{lem:bellman-error-concentration-uniform}) and reward-model (\Cref{lem:reward-model-concentration-bernstein}) concentration inequalities introduce a factor $2$ in front of the empirical losses, the two factors cancel so that the pessimistic evaluation step \cref{eq:alg-pessimism} is invoked with the algorithmic coefficient $\beta$. On this event we have, for every $k\in[K]$,
\begin{subequations}\label{eq:conc-bundle}
  \begin{align}
    \cL_\mu^{\mathrm{BE}}(\pi_k,r_k,f_k)                                 & \leq 2\cL_\cD^{\mathrm{BE}}(\pi_k,r_k,f_k)+\varepsilon_{\mathrm{BE}}+4   \varepsilon_{\cF,\cF}, \label{eq:conc-BE}        \\
    \cL_\mu^{\mathrm{RM}}(r_k)                                           & \leq 2\bigl[\cL_\cD^{\mathrm{RM}}(r_k)-\cL_\cD^{\mathrm{RM}}(r^\star)\bigr]+\varepsilon_{\mathrm{RM}}, \label{eq:conc-RM} \\
    \sup_{\pi\in\Pi,   f\in\cF}\bigl|\cL_\mu(\pi,f)-\cL_\cD(\pi,f)\bigr| & \leq \sqrt{\varepsilon_{\mathrm{Perf}}},\label{eq:conc-Perf}                                                              \\
    |\cL_\mu(\pi_k,Q^{\pi_k})-\cL_\mu(\pi_k,f_{\pi_k})|                  & \leq O(H\sqrt{\varepsilon_\cF}) \quad\text{(\Cref{lemma:bound-diff-Q-and-f_pi})}. \label{eq:conc-apx}
  \end{align}
\end{subequations}
We now bound the quantity
\begin{align*}
  \Delta_k   \coloneqq   \frac{\beta}{2}\cL^{\mathrm{BE}}_\mu(\pi_k,r_k,f_k)+\frac{\beta}{2}\cL_\mu^{\mathrm{RM}}(r_k)+\cL_\mu(\pi_k,f_k)-\cL_\mu(\pi_k,Q^{\pi_k})
\end{align*}
that appears on the right-hand side of \cref{eq:upperbound-I-III-IV}, in four explicit steps.

\smallskip
\noindent\textbf{Step (A): transfer BE/RM losses to the empirical data.}
Using \cref{eq:conc-BE}--\cref{eq:conc-RM},
\begin{align}
  \frac{\beta}{2}\cL^{\mathrm{BE}}_\mu(\pi_k,r_k,f_k)+\frac{\beta}{2}\cL_\mu^{\mathrm{RM}}(r_k)
  \leq   \beta\cL^{\mathrm{BE}}_\cD(\pi_k,r_k,f_k)+\beta\bigl[\cL_\cD^{\mathrm{RM}}(r_k)-\cL_\cD^{\mathrm{RM}}(r^\star)\bigr]+\frac{\beta}{2}\bigl(\varepsilon_{\mathrm{BE}}+\varepsilon_{\mathrm{RM}}+4   \varepsilon_{\cF,\cF}\bigr).\label{eq:stepA}
\end{align}
The $r^\star$-centred form is essential under noisy observations, but it cancels exactly when the same quantity is subtracted from the pessimism RHS in Step (D) below.

\smallskip
\noindent\textbf{Step (B): transfer the performance-loss gap $\cL_\mu(\pi_k,f_k)-\cL_\mu(\pi_k,Q^{\pi_k})$ to the empirical data.}
Insert $f_{\pi_k}$ as a pivot and use \cref{eq:conc-Perf} together with \cref{eq:conc-apx}:
\begin{align}
  \cL_\mu(\pi_k,f_k)-\cL_\mu(\pi_k,Q^{\pi_k})
   & =\underbrace{\bigl[\cL_\mu(\pi_k,f_k)-\cL_\cD(\pi_k,f_k)\bigr]}_{\le\sqrt{\varepsilon_{\mathrm{Perf}}}\text{ by }\cref{eq:conc-Perf}}                         \notag  \\
   & \quad+\underbrace{\bigl[\cL_\cD(\pi_k,f_{\pi_k})-\cL_\mu(\pi_k,f_{\pi_k})\bigr]}_{\le\sqrt{\varepsilon_{\mathrm{Perf}}}\text{ by }\cref{eq:conc-Perf}}        \notag  \\
   & \quad+\underbrace{\bigl[\cL_\mu(\pi_k,f_{\pi_k})-\cL_\mu(\pi_k,Q^{\pi_k})\bigr]}_{\le O(H\sqrt{\varepsilon_\cF})\text{ by }\cref{eq:conc-apx}}                 \notag \\
   & \quad+\cL_\cD(\pi_k,f_k)-\cL_\cD(\pi_k,f_{\pi_k})                                                                                                             \notag   \\
   & \le   \cL_\cD(\pi_k,f_k)-\cL_\cD(\pi_k,f_{\pi_k})+2\sqrt{\varepsilon_{\mathrm{Perf}}}+O(H\sqrt{\varepsilon_\cF}),\label{eq:stepB}
\end{align}
where we applied the uniform concentration \cref{eq:conc-Perf} twice (once to $(\pi_k,f_k)\in\Pi\times\cF$ and once to $(\pi_k,f_{\pi_k})\in\Pi\times\cF$; the latter uses $f_{\pi_k}\in\cF$ from \Cref{lemma:bounded-em-Bellman-error}).

\smallskip
\noindent\textbf{Step (C): invoke pessimistic optimality at coefficient $\beta$.}
By the definition of the pessimistic estimator \cref{eq:alg-pessimism}, $(f_k,r_k)\in\cF\times\cR$ minimises
\begin{align}
  (f,r)\mapsto \cL_\cD(\pi_k,f)+\beta\cL^{\mathrm{BE}}_\cD(\pi_k,r,f)+\beta\cL^{\mathrm{RM}}_\cD(r).\tag{\ref{eq:alg-pessimism}}
\end{align}
Plugging in the comparator $(f_{\pi_k},r^\star)\in\cF\times\cR$ (using $f_{\pi_k}\in\cF$ from \Cref{assumption:realizable-f} and $r^\star\in\cR$ from \Cref{assumption:realizable-R}) yields
\begin{align}
   & \cL_\cD(\pi_k,f_k)+\beta\cL^{\mathrm{BE}}_\cD(\pi_k,r_k,f_k)+\beta\cL^{\mathrm{RM}}_\cD(r_k)                      \notag                                                                         \\
   & \qquad\le   \cL_\cD(\pi_k,f_{\pi_k})+\beta\cL^{\mathrm{BE}}_\cD(\pi_k,r^\star,f_{\pi_k})+\beta\cL^{\mathrm{RM}}_\cD(r^\star).                                                   \label{eq:stepC}
\end{align}

\smallskip
\noindent\textbf{Step (D): bound the comparator Bellman-error loss, and cancel $\cL_\cD^{\mathrm{RM}}(r^\star)$.}
Subtracting $\beta\cL_\cD^{\mathrm{RM}}(r^\star)$ from both sides of \cref{eq:stepC} gives the $r^\star$-centred pessimism inequality
\begin{align}
   & \cL_\cD(\pi_k,f_k)+\beta\cL^{\mathrm{BE}}_\cD(\pi_k,r_k,f_k)+\beta\bigl[\cL^{\mathrm{RM}}_\cD(r_k)-\cL^{\mathrm{RM}}_\cD(r^\star)\bigr]                      \notag \\
   & \qquad\le   \cL_\cD(\pi_k,f_{\pi_k})+\beta\cL^{\mathrm{BE}}_\cD(\pi_k,r^\star,f_{\pi_k}).\label{eq:stepC-centred}
\end{align}
By \Cref{lemma:bounded-em-Bellman-error}, the remaining comparator Bellman-error term on the RHS is controlled by
\begin{align}
  \beta\cL^{\mathrm{BE}}_\cD(\pi_k,r^\star,f_{\pi_k})   \le   \beta\varepsilon_{\mathrm{apx}}.\label{eq:stepD}
\end{align}

\smallskip
\noindent\textbf{Combining (A)--(D).}
Adding \cref{eq:stepA} and \cref{eq:stepB} gives
\begin{align*}
  \Delta_k
   & \le   \bigl[\beta\cL^{\mathrm{BE}}_\cD(\pi_k,r_k,f_k)+\beta\bigl(\cL^{\mathrm{RM}}_\cD(r_k)-\cL^{\mathrm{RM}}_\cD(r^\star)\bigr)+\cL_\cD(\pi_k,f_k)-\cL_\cD(\pi_k,f_{\pi_k})\bigr] \\
   & \qquad+\frac{\beta}{2}\bigl(\varepsilon_{\mathrm{BE}}+\varepsilon_{\mathrm{RM}}+4   \varepsilon_{\cF,\cF}\bigr)+2\sqrt{\varepsilon_{\mathrm{Perf}}}+O(H\sqrt{\varepsilon_\cF}).
\end{align*}
By the $r^\star$-centred pessimism inequality \cref{eq:stepC-centred}, the bracketed empirical quantity is bounded by $\beta\cL^{\mathrm{BE}}_\cD(\pi_k,r^\star,f_{\pi_k})$, which is in turn controlled by \cref{eq:stepD}. Note that the potentially constant-order term $\cL^{\mathrm{RM}}_\cD(r^\star)$ has cancelled on both sides and does not appear in the final bound. Therefore
\begin{align}
  \Delta_k   \le   \frac{\beta}{2}\bigl(\varepsilon_{\mathrm{BE}}+\varepsilon_{\mathrm{RM}}+4   \varepsilon_{\cF,\cF}+2\varepsilon_{\mathrm{apx}}\bigr)+2\sqrt{\varepsilon_{\mathrm{Perf}}}+O(H\sqrt{\varepsilon_\cF}).\label{eq:delta-k-final}
\end{align}
\subsection{Final Rate and Parameter Choice}
Define
\begin{align*}
  \varepsilon
  \coloneqq
  \varepsilon_{\mathrm{BE}}+\varepsilon_{\mathrm{RM}}+4   \varepsilon_{\cF,\cF}+2\varepsilon_{\mathrm{apx}},
  \qquad
  \xi
  \coloneqq
  2\sqrt{\varepsilon_{\mathrm{Perf}}}+O(H\sqrt{\varepsilon_\cF}).
\end{align*}
Since $\Delta_k$ is precisely the last line of \cref{eq:upperbound-I-III-IV}, substituting \cref{eq:delta-k-final} into \cref{eq:upperbound-I-III-IV} gives
\begin{align}
  (\mathrm{I})_k+(\mathrm{III})_k+(\mathrm{IV})_k
  \leq
  \frac{2H(C_{sa}(\pi)+1)}{\beta} + \Delta_k
  \leq
  \frac{2H(C_{sa}(\pi)+1)}{\beta} + \frac{\beta}{2}\varepsilon+\xi.
  \label{eq:final-bound-I-III-IV}
\end{align}
Since $n$ denotes the number of trajectories, the concentration terms above all scale with $n^{-1}$. In this cumulative-return setting, the boundedness convention before \cref{eq:loss_policy_emp} sets $V_{\max}=H$, so under the sum-over-$h$ convention of \cref{eq:all_losses}, we may upper bound every concentration term at the trajectory scale $H$:
\begin{align*}
  \varepsilon_{\mathrm{BE}}   & = \widetilde O   \left(\frac{H^3}{n}\right)\quad\text{(\Cref{lem:bellman-error-concentration-uniform})},\qquad
  \varepsilon_{\mathrm{RM}} = \widetilde O   \left(\frac{H^2}{n}\right)\quad\text{(\Cref{lem:reward-model-concentration-bernstein})},           \\
  \varepsilon_{\mathrm{apx}}  & = \widetilde O   \left(\frac{H^3}{n}+\varepsilon_\cF\right)\quad\text{(\Cref{lemma:bounded-em-Bellman-error})}, \\
  \varepsilon_{\mathrm{Perf}} & = \widetilde O   \left(\frac{H^4}{n}\right)\quad\text{(\Cref{lemma:concentration-L})}.
\end{align*}
The factor $H$ in $\varepsilon_{\mathrm{BE}}$ and $\varepsilon_{\mathrm{apx}}$ relative to $\varepsilon_{\mathrm{RM}}$ comes from the sum over $h\in[H]$ in $\cL^{\mathrm{BE}}$. Hence
\begin{align*}
  \varepsilon
  = \widetilde O   \left(\frac{H^3}{n}+\varepsilon_\cF+\varepsilon_{\cF,\cF}\right),
  \qquad
  \xi
  = \widetilde O   \left(\frac{H^2}{\sqrt n}+H\sqrt{\varepsilon_\cF}\right).
\end{align*}
Optimising the quadratic $\frac{2H(C_{sa}(\pi)+1)}{\beta}+\frac{\beta}{2}\varepsilon$ over $\beta>0$ yields the minimiser
\begin{align*}
  \beta^\star=2\sqrt{\frac{H(C_{sa}(\pi)+1)}{\varepsilon}},
\end{align*}
at which
\begin{align*}
  (\mathrm{I})_k+(\mathrm{III})_k+(\mathrm{IV})_k
  \leq
  2\sqrt{H(C_{sa}(\pi)+1)   \varepsilon}+\xi.
\end{align*}
Substituting the bound on $\varepsilon$ and using $\sqrt{a+b+c}\le\sqrt a+\sqrt b+\sqrt c$,
\begin{align*}
  (\mathrm{I})_k+(\mathrm{III})_k+(\mathrm{IV})_k
   & \le \widetilde O   \left(H\sqrt{\frac{H^2(C_{sa}(\pi)+1)}{n}}+\sqrt{H(C_{sa}(\pi)+1)\varepsilon_\cF}+\sqrt{H(C_{sa}(\pi)+1)\varepsilon_{\cF,\cF}}\right)+\xi.
\end{align*}
Finally, applying the no-regret bound \cref{eq:term-II-bound} for the $(\mathrm{II})_k$ terms and substituting $\xi=\widetilde O(H^2/\sqrt n+H\sqrt{\varepsilon_\cF})$,
\begin{align*}
  J(\pi)-J(\bar\pi) & = \frac{1}{K}\sum_{k=1}^K \big(J(\pi)-J(\pi_k)\big)                                                                                                                                          \\
                    & \leq \frac{1}{K}\sum_{k=1}^K(\mathrm{II})_k+\frac{1}{K}\sum_{k=1}^K\bigl((\mathrm{I})_k+(\mathrm{III})_k+(\mathrm{IV})_k\bigr)                                                               \\
                    & \leq \frac{1}{K}\sum_{k=1}^K (\mathrm{II})_k + 2\sqrt{H(C_{sa}(\pi)+1)\varepsilon}+\xi                                                                                                       \\
                    & \leq   H^2\sqrt{\frac{\log|\cA|}{2K}}   +   2\sqrt{H(C_{sa}(\pi)+1)\varepsilon}   +   \xi                                                                                                    \\
                    & = \widetilde O   \left(H^2\sqrt{\frac{C_{sa}(\pi)}{n}}+\frac{H^2}{\sqrt K}+\sqrt{HC_{sa}(\pi)   \varepsilon_\cF}+\sqrt{HC_{sa}(\pi)   \varepsilon_{\cF,\cF}}+H\sqrt{\varepsilon_\cF}\right),
\end{align*}
where the performance-loss contribution $H^2/\sqrt n$ has been absorbed into the leading statistical rate $H^2\sqrt{C_{sa}(\pi)/n}$.

In particular, taking $K\ge n$ and assuming exact realizability ($\varepsilon_\cF=0$) and exact Bellman completeness ($\varepsilon_{\cF,\cF}=0$),
\begin{align*}
  J(\pi)-J(\bar\pi)
  = \widetilde O   \left(H^2\sqrt{\frac{C_{sa}(\pi)}{n}}\right).
\end{align*}

\section{Proof of the Preference-Based Upper Bound}
\label{appendix:pref-upper}

We prove \Cref{thm:pref-outcome-based-main}. In this cumulative-return setting, $V_{\max}=H$ and every candidate trajectory return lies in $[0,H]$. The argument follows the proof of \Cref{theorem:outcome-based-offline-rl}. The policy and Bellman empirical losses use the $2n$ unlabelled trajectories contained in the preference pairs; since these trajectories are i.i.d.\ from $\mu$, the corresponding concentration bounds change only by universal constants. Relative to the scalar-outcome proof, the main changes are that the squared reward-model loss is replaced by the logistic preference loss, and the reward-mismatch term is controlled through a centered change-of-measure step, since preferences identify rewards only up to uniform shifts of cumulative returns.

We use the following constants for the fixed BTL model on the return range $[0,H]$:
\begin{align}
  \alpha_{\mathrm C}
   & \coloneqq
  \frac{\gamma^2}{2}
  \inf_{|w|\le H}
  \frac{\exp(\gamma w)}{(1+\exp(\gamma w))^2}, \notag \\
  a_{\mathrm C}
   & \coloneqq
  \frac{1}{1+e^{\gamma H}},
  \qquad
  B_{\mathrm C}
  \coloneqq
  \log\frac{1-a_{\mathrm C}}{a_{\mathrm C}}, \notag   \\
  b_{\mathrm C}
   & \coloneqq
  \sup_{p,q\in[a_{\mathrm C},1-a_{\mathrm C}],   p\ne q}
  \frac{
    p\log^2(p/q)+(1-p)\log^2((1-p)/(1-q))
  }{
    \mathrm{KL}(\mathrm{Bern}(p)   \|   \mathrm{Bern}(q))
  },
  \qquad
  c_{\mathrm C}
  \coloneqq
  2b_{\mathrm C}+\frac{8}{3}B_{\mathrm C}.
  \label{eq:pref-model-constants}
\end{align}
These constants depend only on $\gamma$ and $H$.

\subsection{Preference Identifiability and Concentration}

For a pair $(\tau^+,\tau^-)$, define the return-difference notation
\[
  \Delta_r(\tau^+,\tau^-)\coloneqq R(\tau^+;r)-R(\tau^-;r),
  \qquad
  \Delta^\star(\tau^+,\tau^-)\coloneqq \Delta_{r^\star}(\tau^+,\tau^-).
\]
Let
\[
  \ell_r(\tau^+,\tau^-,y)
  \coloneqq
  -y\log \mathrm{C}_r(\tau^+,\tau^-)
  -(1-y)\log(1-\mathrm{C}_r(\tau^+,\tau^-))
\]
be the BTL log loss induced by $r$. The population preference loss is
\[
  \cL_\mu^{\mathrm{Pref}}(r)
  \coloneqq
  \EE_{\tau^+,\tau^-\sim\mu,   y\mid\tau^\pm}
  \bigl[\ell_r(\tau^+,\tau^-,y)\bigr].
\]

\begin{lemma}[BTL preference identifiability]\label{lem:pref-curvature-main}
  Under the BTL model \cref{eq:btl-model}, with $\alpha_{\mathrm C}$ defined in \cref{eq:pref-model-constants}, for every $r\in\cR$,
  \[
    \cL_\mu^{\mathrm{Pref}}(r)-\cL_\mu^{\mathrm{Pref}}(r^\star)
    \ge
    \alpha_{\mathrm C}
    \EE_{\mu\otimes\mu}   \left[(\Delta_r-\Delta^\star)^2\right]
    =
    2\alpha_{\mathrm C}
    \mathrm{Var}_{\mu}   \left(R(\cdot;r)-R^\star(\cdot)\right).
  \]
  Equivalently,
  \[
    \mathrm{Var}_{\mu}   \left(R(\cdot;r)-R^\star(\cdot)\right)
    \le
    \frac{1}{2\alpha_{\mathrm C}}
    \bigl(\cL_\mu^{\mathrm{Pref}}(r)-\cL_\mu^{\mathrm{Pref}}(r^\star)\bigr).
  \]
\end{lemma}

\begin{proof}
  Fix a trajectory pair $(\tau^+,\tau^-)$. Write
  \[
    d_r=\Delta_r(\tau^+,\tau^-),
    \qquad
    d_\star=\Delta^\star(\tau^+,\tau^-).
  \]
  Since $R(\tau;r)\in[0,H]$, both $d_r$ and $d_\star$ lie in $[-H,H]$. Define the one-dimensional logistic link
  \[
    s_\gamma(w)=\frac{\exp(\gamma w)}{1+\exp(\gamma w)}.
  \]
  For this fixed pair, the true label distribution is
  \[
    y\sim\mathrm{Bern}(s_\gamma(d_\star)).
  \]
  Let $p_\star=s_\gamma(d_\star)$ and define the conditional log-risk
  \[
    \phi(w)
    =
    -p_\star\log s_\gamma(w)
    -(1-p_\star)\log(1-s_\gamma(w)).
  \]
  Then
  \[
    \EE_y[\ell_r(\tau^+,\tau^-,y)]=\phi(d_r),
    \qquad
    \EE_y[\ell_{r^\star}(\tau^+,\tau^-,y)]=\phi(d_\star).
  \]
  A direct differentiation gives
  \[
    \phi'(w)=\gamma\bigl(s_\gamma(w)-p_\star\bigr),
    \qquad
    \phi''(w)=\gamma^2 s_\gamma(w)(1-s_\gamma(w)).
  \]
  Since $p_\star=s_\gamma(d_\star)$, we have $\phi'(d_\star)=0$. Also,
  \[
    m_{\mathrm C}
    \coloneqq
    \inf_{|w|\le H}\gamma^2s_\gamma(w)(1-s_\gamma(w))
    >0,
  \]
  because $s_\gamma(w)(1-s_\gamma(w))$ is continuous and strictly positive on the compact interval $[-H,H]$. Taylor's theorem with remainder gives, for some point $\bar w$ between $d_r$ and $d_\star$,
  \[
    \phi(d_r)-\phi(d_\star)
    =
    \phi'(d_\star)(d_r-d_\star)
    +
    \frac12\phi''(\bar w)(d_r-d_\star)^2
    \ge
    \frac{m_{\mathrm C}}{2}(d_r-d_\star)^2.
  \]
  By the definition of $\alpha_{\mathrm C}$ in \cref{eq:pref-model-constants}, $m_{\mathrm C}/2=\alpha_{\mathrm C}$. Since the trajectory pair was arbitrary, taking expectation over $\tau^+,\tau^-\sim\mu$ gives
  \[
    \cL_\mu^{\mathrm{Pref}}(r)-\cL_\mu^{\mathrm{Pref}}(r^\star)
    \ge
    \alpha_{\mathrm C}
    \EE_{\mu\otimes\mu}   \left[(\Delta_r-\Delta^\star)^2\right].
  \]
  It remains to rewrite the pairwise error as a variance. Let
  \[
    g(\tau)=R(\tau;r)-R^\star(\tau).
  \]
  Then
  \[
    \Delta_r(\tau^+,\tau^-)-\Delta^\star(\tau^+,\tau^-)
    =
    g(\tau^+)-g(\tau^-).
  \]
  Since $\tau^+$ and $\tau^-$ are independent draws from $\mu$,
  \[
    \begin{aligned}
      \EE_{\mu\otimes\mu}[(g(\tau^+)-g(\tau^-))^2]
       & =
      \EE_\mu[g(\tau)^2]+\EE_\mu[g(\tau)^2]
      -2\EE_\mu[g(\tau)]^2 \\
       & =
      2   \mathrm{Var}_\mu(g),
    \end{aligned}
  \]as desired.
\end{proof}

\begin{lemma}[Preference log loss concentration]\label{lem:pref-concentration-main}
  Let $a_{\mathrm C}$, $B_{\mathrm C}$, $b_{\mathrm C}$, and $c_{\mathrm C}$ be defined in \cref{eq:pref-model-constants}.
  With probability at least $1-\delta$, simultaneously for all $r\in\cR$,
  \[
    \cL_\mu^{\mathrm{Pref}}(r)-\cL_\mu^{\mathrm{Pref}}(r^\star)
    \le
    2\Bigl[\cL_{\cD_{\mathrm{Pref}}}^{\mathrm{Pref}}(r)-\cL_{\cD_{\mathrm{Pref}}}^{\mathrm{Pref}}(r^\star)\Bigr]
    +
    c_{\mathrm C}\frac{\log(|\cR|/\delta)}{n}.
  \]
\end{lemma}

\begin{proof}
  The structure is the same as the $r^\star$-centred reward-model concentration in \Cref{lem:reward-model-concentration-bernstein}: we control the population excess loss by the empirical loss difference relative to $r^\star$. The difference is that the random variable is now a log-likelihood ratio, whose mean is a KL divergence and whose variance is self-bounded by that same KL divergence.

  Fix $r\in\cR$. Let $Z=(\tau^+,\tau^-,y)$ and let $P_\star$ denote the true law of $Z$: the pair $(\tau^+,\tau^-)$ is drawn from $\mu\otimes\mu$, and then
  \[
    y\mid \tau^+,\tau^- \sim \mathrm{Bern}\bigl(\mathrm{C}_{r^\star}(\tau^+,\tau^-)\bigr).
  \]
  Let $P_r$ be the same law except that the conditional Bernoulli success probability is $\mathrm{C}_r(\tau^+,\tau^-)$. Since the pair marginal is the same under $P_\star$ and $P_r$, only the conditional Bernoulli law changes.

  Write $p_\star=\mathrm{C}_{r^\star}(\tau^+,\tau^-)$ and $p_r=\mathrm{C}_r(\tau^+,\tau^-)$. Since $P_\star$ and $P_r$ have the same pair distribution, their likelihood ratio is just the Bernoulli label likelihood ratio. Thus
  \[
    \begin{aligned}
      \ell_r(\tau^+,\tau^-,y)-\ell_{r^\star}(\tau^+,\tau^-,y)
       & =
      \left[-y\log p_r-(1-y)\log(1-p_r)\right]         \\
       & \quad -
      \left[-y\log p_\star-(1-y)\log(1-p_\star)\right] \\
       & =
      y\log\frac{p_\star}{p_r}
      +(1-y)\log\frac{1-p_\star}{1-p_r}                \\
       & =
      \log\frac{p_\star^y(1-p_\star)^{1-y}}{p_r^y(1-p_r)^{1-y}}.
    \end{aligned}
  \]
  Thus the single-sample loss difference is exactly the log-likelihood ratio:
  \[
    W_r(Z)
    \coloneqq
    \ell_r(\tau^+,\tau^-,y)-\ell_{r^\star}(\tau^+,\tau^-,y)
    =
    \log\frac{P_\star(y\mid\tau^+,\tau^-)}{P_r(y\mid\tau^+,\tau^-)}.
  \]
  For this fixed trajectory pair, taking expectation only over the label $y\mid\tau^+,\tau^-$ gives
  \[
    \begin{aligned}
      \EE_{y\mid\tau^+,\tau^-}[W_r(Z)]
       & =
      p_\star\log\frac{p_\star}{p_r}
      +(1-p_\star)\log\frac{1-p_\star}{1-p_r} \\
       & =
      \mathrm{KL}   \left(
      \mathrm{Bern}(p_\star)   \|   \mathrm{Bern}(p_r)
      \right).
    \end{aligned}
  \]
  Now taking the outer expectation over the trajectory pair $(\tau^+,\tau^-)\sim\mu\otimes\mu$,
  \[
    \begin{aligned}
      \cL_\mu^{\mathrm{Pref}}(r)-\cL_\mu^{\mathrm{Pref}}(r^\star)
       & =
      \EE_{(\tau^+,\tau^-)\sim\mu\otimes\mu}
      \left[
        \EE_{y\mid\tau^+,\tau^-}[W_r(Z)]
        \right] \\
       & =
      \EE_{Z\sim P_\star}[W_r(Z)]
      =
      \mathrm{KL}(P_\star   \|   P_r).
    \end{aligned}
  \]
  Also, for the observed dataset $Z_1,\ldots,Z_n$,
  \[
    \frac1n\sum_{i=1}^n W_r(Z_i)
    =
    \cL_{\cD_{\mathrm{Pref}}}^{\mathrm{Pref}}(r)
    -
    \cL_{\cD_{\mathrm{Pref}}}^{\mathrm{Pref}}(r^\star).
  \]
  Write
  \[
    L_r\coloneqq \mathrm{KL}(P_\star   \|   P_r),
    \qquad
    \widehat L_r
    \coloneqq
    \frac1n\sum_{i=1}^n W_r(Z_i).
  \]

  We now prove concentration of $\widehat L_r$ around $L_r$. The BTL probabilities are uniformly bounded away from $0$ and $1$ on the return range $[0,H]$. Indeed, for all $r\in\cR$ and all trajectory pairs,
  \[
    a_{\mathrm C}
    \le
    \mathrm{C}_r(\tau^+,\tau^-)
    \le
    1-a_{\mathrm C}.
  \]
  Hence the log-likelihood ratio is bounded:
  \[
    |W_r(Z)|
    \le
    B_{\mathrm C}.
  \]
  We also need the standard self-bounding property of bounded Bernoulli log likelihoods. For $p,q\in[a_{\mathrm C},1-a_{\mathrm C}]$, define
  \[
    k(p,q)=\mathrm{KL}(\mathrm{Bern}(p)   \|   \mathrm{Bern}(q))
  \]
  and
  \[
    m(p,q)
    =
    p\log^2\frac{p}{q}
    +(1-p)\log^2\frac{1-p}{1-q}.
  \]
  Since the square $[a_{\mathrm C},1-a_{\mathrm C}]^2$ is compact and $m(p,q)/k(p,q)$ has a finite continuous extension at $p=q$, the constant $b_{\mathrm C}$ in the lemma statement is finite and satisfies $m(p,q)\le b_{\mathrm C}k(p,q)$ for all $p,q$ in this interval.

  Applying this pointwise with $p=\mathrm{C}_{r^\star}(\tau^+,\tau^-)$ and $q=\mathrm{C}_r(\tau^+,\tau^-)$, the conditional second moment over the label satisfies
  \[
    \EE_{y\mid\tau^+,\tau^-}[W_r(Z)^2]
    \le
    b_{\mathrm C}
    \EE_{y\mid\tau^+,\tau^-}[W_r(Z)].
  \]
  Taking the outer expectation over $(\tau^+,\tau^-)\sim\mu\otimes\mu$ gives
  \[
    \EE_{Z\sim P_\star}[W_r(Z)^2]
    \le
    b_{\mathrm C}   \EE_{Z\sim P_\star}[W_r(Z)]
    =
    b_{\mathrm C}L_r.
  \]
  Therefore $\mathrm{Var}_{Z\sim P_\star}(W_r(Z))\le b_{\mathrm C}L_r$.

  Now apply one-sided Bernstein to the i.i.d.\ variables $W_r(Z_1),\ldots,W_r(Z_n)$. Since $|W_r(Z)|\le B_{\mathrm C}$ and $L_r=\EE_{Z\sim P_\star}[W_r(Z)]\le \EE_{Z\sim P_\star}[|W_r(Z)|]\le B_{\mathrm C}$, the lower-tail centered variables satisfy $L_r-W_r(Z)\le 2B_{\mathrm C}$. Thus, for any fixed $r$ and any $\eta\in(0,1)$, with probability at least $1-\eta$,
  \[
    L_r-\widehat L_r
    \le
    \sqrt{\frac{2b_{\mathrm C}L_r\log(1/\eta)}{n}}
    +
    \frac{4B_{\mathrm C}\log(1/\eta)}{3n}.
  \]
  Take $\eta=\delta/|\cR|$ and union bound over $r\in\cR$. On the resulting event, set
  \[
    a_n=\frac{\log(|\cR|/\delta)}{n}.
  \]
  Then every $r\in\cR$ satisfies
  \[
    L_r
    \le
    \widehat L_r
    +
    \sqrt{2b_{\mathrm C}L_r a_n}
    +
    \frac{4B_{\mathrm C}}{3}a_n.
  \]
  Using $\sqrt{2b_{\mathrm C}L_r a_n}\le L_r/2+b_{\mathrm C}a_n$ and rearranging,
  \[
    L_r
    \le
    2\widehat L_r
    +
    c_{\mathrm C}a_n.
  \]
  Substituting back the definitions of $L_r$ and $\widehat L_r$ proves the lemma. Here $B_{\mathrm C}$, $b_{\mathrm C}$, and $c_{\mathrm C}$ are constants determined by the BTL comparison model on the bounded return range. More explicitly, they depend on $a_{\mathrm C}=1/(1+e^{\gamma H})$, and hence on $\gamma H$, but not on $n$, $\delta$, or the size of the reward class except through the logarithmic factor above.
\end{proof}

\begin{lemma}[Centred change of trajectory measure]\label{lemma:change-of-trajectory-centred-main}
  For any policy $\pi$ and any per-step function $g:\cS\times\cA\to\RR$, write $\bar g(\tau)=\sum_{h=1}^H g(s_h,a_h)$. Then
  \[
    \left(\EE_\pi[\bar g]-\EE_\mu[\bar g]\right)^2
    \le
    H C_{sa}(\pi)   \mathrm{Var}_\mu(\bar g).
  \]
\end{lemma}

\begin{proof}
  Let $\bar g_\mu=\EE_\mu[\bar g]$ and define $\tilde g(s_h,a_h)=g(s_h,a_h)-\bar g_\mu/H$. Then $\sum_h\tilde g(s_h,a_h)=\bar g(\tau)-\bar g_\mu$. Applying \Cref{lemma:change-of-trajectory} to $\tilde g$ gives
  \[
    \left(\EE_\pi[\bar g-\bar g_\mu]\right)^2
    \le
    H C_{sa}(\pi)
    \EE_\mu[(\bar g-\bar g_\mu)^2],
  \]
  which is the desired inequality.
\end{proof}

\subsection{Proof of the Preference-based Upper Bound}

Fix a competitor $\pi\in\Pi$. The regret decomposition \cref{eq:regret-decomposition} is unchanged because the objective is still the cumulative return $J(\pi)=\EE_\pi[R^\star(\tau)]$. Terms $(\mathrm{I})_k$, $(\mathrm{II})_k$, and $(\mathrm{III})_k$ are controlled exactly as in the scalar-outcome proof. The only population-level change is the reward-mismatch term
\[
  (\mathrm{IV})_k
  =
  \EE_\pi[R^\star(\tau)-R(\tau;r_k)]
  +
  \EE_\mu[R(\tau;r_k)-R^\star(\tau)].
\]
Let $g_{k,h}(s,a)\coloneqq r_{k,h}(s,a)-r_h^\star(s,a)$ and $\bar g_k(\tau)\coloneqq\sum_{h=1}^H g_{k,h}(s_h,a_h)=R(\tau;r_k)-R^\star(\tau)$. Then
\[
  (\mathrm{IV})_k
  =
  -\bigl(\EE_\pi[\bar g_k]-\EE_\mu[\bar g_k]\bigr).
\]
By \Cref{lemma:change-of-trajectory-centred-main}, \Cref{lem:pref-curvature-main}, and AM--GM with parameter $\beta>0$,
\begin{align}
  |(\mathrm{IV})_k|
   & \le
  \sqrt{H C_{sa}(\pi)   \mathrm{Var}_\mu(\bar g_k)} \notag                            \\
   & \le
  \sqrt{\frac{H C_{sa}(\pi)}{2\alpha_{\mathrm C}}
    \bigl(\cL_\mu^{\mathrm{Pref}}(r_k)-\cL_\mu^{\mathrm{Pref}}(r^\star)\bigr)} \notag \\
   & \le
  \frac{H C_{sa}(\pi)}{4\alpha_{\mathrm C}\beta}
  +
  \frac{\beta}{2}
  \bigl(\cL_\mu^{\mathrm{Pref}}(r_k)-\cL_\mu^{\mathrm{Pref}}(r^\star)\bigr).
  \label{eq:pref-IV-main}
\end{align}
Combining this with the bound on term $(\mathrm{I})$ from \cref{eq:upperbound-I} and adding term $(\mathrm{III})$ gives
\begin{align}
  (\mathrm{I})_k+(\mathrm{III})_k+(\mathrm{IV})_k
   & \le
  \frac{H(C_{sa}(\pi)+1)}{\beta}
  +
  \frac{H C_{sa}(\pi)}{4\alpha_{\mathrm C}\beta}
  +
  \frac{\beta}{2}\cL_\mu^{\mathrm{BE}}(\pi_k,r_k,f_k) \notag \\
   & \qquad
  +
  \frac{\beta}{2}
  \bigl(\cL_\mu^{\mathrm{Pref}}(r_k)-\cL_\mu^{\mathrm{Pref}}(r^\star)\bigr)
  +
  \cL_\mu(\pi_k,f_k)-\cL_\mu(\pi_k,Q^{\pi_k}).
  \label{eq:pref-pop-main}
\end{align}

We now transfer the population losses to empirical losses. On the same high-probability event used in the proof of \Cref{theorem:outcome-based-offline-rl}, \cref{eq:conc-BE}, \cref{eq:conc-Perf}, and \cref{eq:conc-apx} hold. In addition, \Cref{lem:pref-concentration-main} gives the centered preference concentration bound with
\[
  \varepsilon_{\mathrm{Pref}}
  \coloneqq
  c_{\mathrm C}\frac{\log(|\cR|/\delta)}{n}.
\]
Therefore,
\begin{align}
   & \frac{\beta}{2}\cL_\mu^{\mathrm{BE}}(\pi_k,r_k,f_k)
  +
  \frac{\beta}{2}
  \bigl(\cL_\mu^{\mathrm{Pref}}(r_k)-\cL_\mu^{\mathrm{Pref}}(r^\star)\bigr)
  \notag                                                 \\
   & \qquad\le
  \beta\cL_\cD^{\mathrm{BE}}(\pi_k,r_k,f_k)
  +
  \beta
  \bigl[
    \cL_{\cD_{\mathrm{Pref}}}^{\mathrm{Pref}}(r_k)
    -
    \cL_{\cD_{\mathrm{Pref}}}^{\mathrm{Pref}}(r^\star)
    \bigr]
  +
  \frac{\beta}{2}
  \bigl(\varepsilon_{\mathrm{BE}}+\varepsilon_{\mathrm{Pref}}+4\varepsilon_{\cF,\cF}\bigr).
  \label{eq:pref-transfer-main}
\end{align}
The performance-loss transfer is exactly \cref{eq:stepB}:
\[
  \cL_\mu(\pi_k,f_k)-\cL_\mu(\pi_k,Q^{\pi_k})
  \le
  \cL_\cD(\pi_k,f_k)-\cL_\cD(\pi_k,f_{\pi_k})
  +
  2\sqrt{\varepsilon_{\mathrm{Perf}}}
  +
  O(H\sqrt{\varepsilon_\cF}).
\]
By the optimality of $(f_k,r_k)$ in \cref{eq:pref-alg-pessimism}, and subtracting the common term $\beta\cL_{\cD_{\mathrm{Pref}}}^{\mathrm{Pref}}(r^\star)$ from both sides,
\begin{align}
   & \cL_\cD(\pi_k,f_k)
  +
  \beta\cL_\cD^{\mathrm{BE}}(\pi_k,r_k,f_k)
  +
  \beta
  \bigl[
    \cL_{\cD_{\mathrm{Pref}}}^{\mathrm{Pref}}(r_k)
    -
    \cL_{\cD_{\mathrm{Pref}}}^{\mathrm{Pref}}(r^\star)
    \bigr]
  \notag                \\
   & \qquad\le
  \cL_\cD(\pi_k,f_{\pi_k})
  +
  \beta\cL_\cD^{\mathrm{BE}}(\pi_k,r^\star,f_{\pi_k}).
  \label{eq:pref-pess-main}
\end{align}
Combining \cref{eq:pref-pop-main}, \cref{eq:pref-transfer-main}, the performance-loss transfer above, and \cref{eq:pref-pess-main}, the empirical terms are bounded as
\begin{align*}
   & \cL_\cD(\pi_k,f_k)
  +\beta\cL_\cD^{\mathrm{BE}}(\pi_k,r_k,f_k)
  +\beta\bigl[
          \cL_{\cD_{\mathrm{Pref}}}^{\mathrm{Pref}}(r_k)
          -
          \cL_{\cD_{\mathrm{Pref}}}^{\mathrm{Pref}}(r^\star)
          \bigr]
  -\cL_\cD(\pi_k,f_{\pi_k}) \\
   & \qquad\le
  \beta\cL_\cD^{\mathrm{BE}}(\pi_k,r^\star,f_{\pi_k}).
\end{align*}
By \Cref{lemma:bounded-em-Bellman-error},
\[
  \cL_\cD^{\mathrm{BE}}(\pi_k,r^\star,f_{\pi_k})
  \le
  \varepsilon_{\mathrm{apx}},
\]
so the left-hand side of the previous display is at most $\beta\varepsilon_{\mathrm{apx}}$, which is written below as $(\beta/2)\cdot 2\varepsilon_{\mathrm{apx}}$ to match the other slack terms. Therefore,
\begin{align}
  (\mathrm{I})_k+(\mathrm{III})_k+(\mathrm{IV})_k
   & \le
  \frac{H(C_{sa}(\pi)+1)}{\beta}
  +
  \frac{H C_{sa}(\pi)}{4\alpha_{\mathrm C}\beta}
  \notag   \\
   & \quad
  +
  \frac{\beta}{2}
  \bigl(\varepsilon_{\mathrm{BE}}+\varepsilon_{\mathrm{Pref}}+4\varepsilon_{\cF,\cF}+2\varepsilon_{\mathrm{apx}}\bigr)
  +
  2\sqrt{\varepsilon_{\mathrm{Perf}}}
  +
  O(H\sqrt{\varepsilon_\cF}).
  \label{eq:pref-final-main}
\end{align}

It remains to optimize $\beta$. By \Cref{lem:pref-concentration-main}, $\varepsilon_{\mathrm{Pref}}=\widetilde O(c_{\mathrm C}/n)$. As in the scalar-outcome proof,
\[
  \varepsilon_{\mathrm{BE}}=\widetilde O(H^3/n),
  \qquad
  \varepsilon_{\mathrm{apx}}=\widetilde O(H^3/n+\varepsilon_\cF),
  \qquad
  \varepsilon_{\mathrm{Perf}}=\widetilde O(H^4/n),
\]
and the first two terms in \cref{eq:pref-final-main} are minimized by choosing
\[
  \beta
  =
  \sqrt{
    \frac{
      2\bigl(H(C_{sa}(\pi)+1)+H C_{sa}(\pi)/(4\alpha_{\mathrm C})\bigr)
    }{
      \varepsilon_{\mathrm{BE}}+\varepsilon_{\mathrm{Pref}}+4\varepsilon_{\cF,\cF}+2\varepsilon_{\mathrm{apx}}
    }
  }.
\]
With this choice, \cref{eq:pref-final-main} gives the explicit intermediate bound
\[
  \begin{aligned}
    (\mathrm{I})_k+(\mathrm{III})_k+(\mathrm{IV})_k
    =
    \widetilde O   \Bigg(
     & \sqrt{
         \left[
           H(C_{sa}(\pi)+1)+\frac{H C_{sa}(\pi)}{4\alpha_{\mathrm C}}
           \right]
         \left[
           \frac{H^3+c_{\mathrm C}}{n}
           +\varepsilon_\cF+\varepsilon_{\cF,\cF}
           \right]
       } \\
     & \quad
    +\frac{H^2}{\sqrt n}
    +H\sqrt{\varepsilon_\cF}
    \Bigg).
  \end{aligned}
\]
Equivalently, using $C_{sa}(\pi)\ge 1$ and expanding the square-root product,
\[
  \begin{aligned}
    (\mathrm{I})_k+(\mathrm{III})_k+(\mathrm{IV})_k
    =
    \widetilde O   \Bigg(
     & \sqrt{1+\frac{1}{\alpha_{\mathrm C}}}
    \Bigg[
      H^2\sqrt{\frac{C_{sa}(\pi)}{n}}
      +
      \sqrt{\frac{c_{\mathrm C}H C_{sa}(\pi)}{n}} \\
      &\qquad\qquad
      +
      \sqrt{HC_{sa}(\pi)\varepsilon_\cF}
      +
      \sqrt{HC_{sa}(\pi)\varepsilon_{\cF,\cF}}
      \Bigg]
    +
    H\sqrt{\varepsilon_\cF}
    \Bigg).
  \end{aligned}
\]
Adding the no-regret bound \cref{eq:term-II-bound} for term $(\mathrm{II})$ and using $V_{\max}=H$ in this cumulative-return setting, averaging over $k=1,\ldots,K$ gives
\[
  \begin{aligned}
    J(\pi)-J(\bar\pi)
    =
    \widetilde O   \Bigg(
     & \sqrt{1+\frac{1}{\alpha_{\mathrm C}}}
    \Bigg[
      H^2\sqrt{\frac{C_{sa}(\pi)}{n}}
      +
      \sqrt{\frac{c_{\mathrm C}H C_{sa}(\pi)}{n}} \\
      &\qquad\qquad
      +
      \sqrt{HC_{sa}(\pi)\varepsilon_\cF}
      +
      \sqrt{HC_{sa}(\pi)\varepsilon_{\cF,\cF}}
      \Bigg]
    +
    \frac{H^2}{\sqrt K}
    +
    H\sqrt{\varepsilon_\cF}
    \Bigg).
  \end{aligned}
\]
Here $\alpha_{\mathrm C}$ and $c_{\mathrm C}$ are the BTL preference constants defined in \cref{eq:pref-model-constants}. They depend on $\gamma$ and $H$ through the bounded comparison range, but not on $n$, $\delta$, or $|\cR|$ except through the explicit factor $\log(|\cR|/\delta)$ already included in the concentration term.
Taking $\pi=\pi^\star$ in the explicit display above proves \cref{eq:pref-main-rate}. 

\section{Lower Bound for Outcome-based Learning}

\subsection[Proof of theorem: lower bound for sum-reward outcomes]{Proof of \Cref{thm:outcome_based_lower_bound}: Lower Bound for Sum-Reward Outcomes}\label{subsec:lb-sum-reward-outcomes}
We construct a hard family of finite-horizon MDPs with deterministic transitions. Let the horizon be $H$, the state space be
\[
  \cS=\{s_1,\dots,s_H\},
\]
and the action space be $\cA=\{0,1\}$. The transition kernel is deterministic:
\[
  P_h(s_{h+1}\mid s_h,a)=1,\qquad \forall h\in[H-1],\ a\in\cA.
\]
Thus the state is simply the time index, and the only decision is which action to take at each step.

The environments differ only in their reward functions. For a parameter
\[
  \btheta=(\theta_1,\dots,\theta_H)\in\{0,1\}^H
\]
and a (to-be-chosen) gap parameter $\Delta\in[0,1]$, define the reward at step $h$ by
\[
  r_{h,\btheta}(s_h,a)
  =
  \frac{1}{2}+\frac{\Delta}{2}   \mathbf 1\{a=\theta_h\}
  -\frac{\Delta}{2}   \mathbf 1\{a\neq \theta_h\},
  \qquad a\in\{0,1\}.
\]
Equivalently, at each step exactly one action has reward $(1+\Delta)/2$, and the other has reward $(1-\Delta)/2$; both lie in $[0,1]$, so the trajectory return satisfies $\sum_h r_{h,\btheta}(s_h,a_h)\in[0,H]$. The optimal policy $\pi_\btheta^\star$ selects action $\theta_h$ at every step $h$, and its value is
\[
  J_\btheta(\pi_\btheta^\star)=\frac{H}{2}+\frac{H\Delta}{2}.
\]

We assume the behavior policy $\mu$ chooses the two actions uniformly at every step, so
\[
  d_h^\mu(s_h,a)=\frac12.
\]
For any target policy $\pi$, we always have $d_h^\pi(s_h,a)\le 1$, therefore
\[
  C_{sa}
  =
  \max_\pi\max_{h,s,a}\frac{d_h^\pi(s,a)}{d_h^\mu(s,a)}
  \le 2.
\]
Thus $C_{sa}=2$ in this construction.

In the outcome-based setting, the learner does not observe stepwise rewards. Instead, for a trajectory
\[
  \tau=(s_1,a_1,\dots,a_H,s_H),
\]
it only observes a bounded trajectory-level label
\[
  Y(\tau)\mid\tau\sim H\cdot\mathrm{Bernoulli}   \left(\frac{R_\btheta(\tau)}{H}\right),
  \qquad
  R_\btheta(\tau)=\sum_{h=1}^H r_{h,\btheta}(s_h,a_h).
\]
Thus $Y(\tau)\in\{0,H\}$ and $\EE[Y(\tau)\mid\tau]=R_\btheta(\tau)$, matching the bounded unbiased trajectory-outcome model used in the upper bound.

Now identify a trajectory with its action vector
\[
  \xb=(a_1,\dots,a_H)\in\{0,1\}^H.
\]
Under the uniform behavior policy, $\xb$ is distributed uniformly over $\{0,1\}^H$. Moreover,
\[
  \sum_{h=1}^H r_{h,\btheta}(s_h,a_h)
  =
  \frac{H}{2}+\frac{\Delta}{2}\sum_{h=1}^H (1-2\mathbf 1\{a_h\neq \theta_h\})
  =
  \frac{H}{2}+\frac{H\Delta}{2}-\Delta\|\xb-\btheta\|_1.
\]
Consequently the conditional success probability of the Bernoulli label is
\[
  p_\btheta(\xb)
  \coloneqq
  \PP_\btheta(Y=H\mid \xb)
  =
  \frac{R_\btheta(\xb)}{H}
  =
  \frac12+\frac{\Delta}{2}-\frac{\Delta}{H}\|\xb-\btheta\|_1.
\]

\paragraph{Reduction to a statistical estimation problem}
We therefore consider the estimation problem where $\btheta\in\Theta\subseteq\{0,1\}^H$ is unknown and we observe $n$ i.i.d.\ samples $(\xb_1,Y_1),\dots,(\xb_n,Y_n)$ from $P_\btheta$, where
\[
  \xb_i\sim \mathrm{Unif}(\{0,1\}^H),
  \qquad
  Y_i\mid\xb_i\sim H\cdot\mathrm{Bernoulli}\bigl(p_\btheta(\xb_i)\bigr).
\]
An estimator takes
\[
  S=(\xb_1,Y_1,\dots,\xb_n,Y_n)
\]
as input and outputs $\widehat\btheta(S)$. We measure estimation error by
\[
  \ell(\btheta,\widehat\btheta)=\Delta\|\btheta-\widehat\btheta\|_1.
\]
Define
\[
  \rho(\btheta,\btheta')=\Delta\|\btheta-\btheta'\|_1,
  \qquad
  \Phi(x)=x.
\]

By Fano's inequality, if $\Theta=\{\btheta_1,\dots,\btheta_M\}$ is a $\delta$-packing under $\rho$, then
\[
  \mathfrak R^*
  :=
  \inf_{\widehat\btheta}\sup_{\btheta\in\Theta}\EE[\ell(\btheta,\widehat\btheta)]
  \ge
  \Phi(\delta/2)
  \left(
  1-\frac{\frac{1}{M^2}\sum_{j,k=1}^M \KL(P_j^{\otimes n}\|P_k^{\otimes n})+\log 2}{\log M}
  \right).
\]

We now bound the KL divergence. Let $P_i:=P_{\btheta_i}$ and $P_j:=P_{\btheta_j}$ denote the single-sample laws. Since the marginal law of $\xb$ does not depend on $\btheta$,
\[
  \KL(P_i\|P_j)
  =
  \frac{1}{2^H}\sum_{\xb}
  \KL\big(P_i(Y\mid \xb)   \|   P_j(Y\mid \xb)\big).
\]
Conditioned on $\xb$, both models are Bernoulli labels scaled by $H$, and the scaling does not change the KL divergence. Let $p_i(\xb)\coloneqq p_{\btheta_i}(\xb)$ and $p_j(\xb)\coloneqq p_{\btheta_j}(\xb)$. If $\Delta\le 1/2$, then $p_i(\xb),p_j(\xb)\in[1/4,3/4]$ for every $\xb$, and the standard Bernoulli KL bound gives
\[
  \KL\big(P_i(Y\mid \xb)   \|   P_j(Y\mid \xb)\big)
  =
  \KL\big(\mathrm{Bern}(p_i(\xb))   \|   \mathrm{Bern}(p_j(\xb))\big)
  \le
  8\bigl(p_i(\xb)-p_j(\xb)\bigr)^2.
\]
Moreover,
\[
  p_i(\xb)-p_j(\xb)
  =
  -\frac{\Delta}{H}
  \Big(\|\xb-\btheta_i\|_1-\|\xb-\btheta_j\|_1\Big).
\]
Thus
\[
  \KL(P_i\|P_j)
  \le
  \frac{8\Delta^2}{H^2}\cdot \frac{1}{2^H}\sum_{\xb}
  \Big(\|\xb-\btheta_i\|_1-\|\xb-\btheta_j\|_1\Big)^2.
\]

Now let
\[
  \mathcal I:=\{h\in[H]:\theta_{i,h}\neq \theta_{j,h}\},
  \qquad |\mathcal I|=\|\btheta_i-\btheta_j\|_1.
\]
Comparing the two Hamming distances coordinate by coordinate, only coordinates in $\mathcal I$ contribute. For each $h\in\mathcal I$, the difference
$|x_h-\theta_{i,h}|-|x_h-\theta_{j,h}|$ has mean zero and square one under $x_h\sim\mathrm{Unif}\{0,1\}$, and the coordinate contributions are independent. Therefore,
\[
  \frac{1}{2^H}\sum_{\xb}
  \Big(\|\xb-\btheta_i\|_1-\|\xb-\btheta_j\|_1\Big)^2
  =\EE_{\xb}\Big(\|\xb-\btheta_i\|_1-\|\xb-\btheta_j\|_1\Big)^2=
  |\mathcal I|
  =
  \|\btheta_i-\btheta_j\|_1.
\]
It follows that
\[
  \KL(P_i\|P_j)
  \le
  \frac{8\Delta^2}{H^2}   \|\btheta_i-\btheta_j\|_1.
\]
Since the samples are i.i.d.,
\[
  \KL(P_i^{\otimes n}\|P_j^{\otimes n})
  =
  n   \KL(P_i\|P_j)
  \le
  \frac{8n\Delta^2}{H^2}   \|\btheta_i-\btheta_j\|_1.
\]

Next, by the Varshamov--Gilbert lemma, there exists a subset
\[
  \Theta=\{\btheta_1,\dots,\btheta_M\}\subseteq\{0,1\}^H
\]
such that
\[
  \log M\ge \frac{H}{8},
  \qquad
  \|\btheta_i-\btheta_j\|_1\ge \frac{H}{8},
  \quad \forall i\neq j.
\]
Hence $\Theta$ is a $\delta$-packing under $\rho$ with
\[
  \delta
  =
  \min_{i\neq j}\rho(\btheta_i,\btheta_j)
  \ge
  \frac{\Delta H}{8}.
\]
Also, since $\|\btheta_i-\btheta_j\|_1\le H$,
\[
  \max_{i\neq j}\KL(P_i^{\otimes n}\|P_j^{\otimes n})
  \le
  \frac{8n\Delta^2}{H}.
\]

Choose
\[
  \Delta   =   \frac{H}{16\sqrt n},
\]
which is feasible with $\Delta\le 1/2$ under the theorem's standing condition $n\ge 64H^2$. Then
\[
  \max_{i\neq j}\KL(P_i^{\otimes n}\|P_j^{\otimes n})
  \le
  \frac{H}{32}.
\]
Since $\log M\ge H/8$, we have
\[
  \max_{i\neq j}\KL(P_i^{\otimes n}\|P_j^{\otimes n})
  \le \frac14\log M.
\]
Moreover, choose the universal constant $H_0$ in \Cref{thm:outcome_based_lower_bound} large enough that $\exp(H_0/8)\ge 16$. Then $H\ge H_0$ and $\log M\ge H/8$ imply $M\ge 16$, hence $\log M\ge 4\log 2$. Therefore Fano's inequality~\citep{yu1997assouad} gives
\[
  \mathfrak R^*
  \ge
  \Phi(\delta/2)
  \left(
  1-\frac{\frac14\log M+\log 2}{\log M}
  \right)
  \ge
  \frac12\Phi(\delta/2)
  =
  \frac{\delta}{4}
  \ge
  \frac{\Delta H}{32}.
\]
Substituting the choice of $\Delta$ yields
\[
  \mathfrak R^*\ge \frac{H^2}{512\sqrt n}.
\]

\paragraph{Return to the original lower bound for outcome-based RL}
Finally, we translate this estimation lower bound back to policy learning. Given an output policy $\widehat\pi$, define an estimator $\widehat\btheta(\widehat\pi)$ by choosing, at each stage, an action with maximal probability under $\widehat\pi_h(\cdot\mid s_h)$, breaking ties arbitrarily. If $\widehat\theta_h\neq\theta_h$, then $\widehat\pi$ puts probability at most $1/2$ on the optimal action at stage $h$, and hence loses at least $\Delta/2$ in expected reward at that stage. Therefore,
\[
  J_\btheta(\pi_\btheta^\star)-J_\btheta(\widehat\pi)
  \ge
  \frac{\Delta}{2}\|\btheta-\widehat\btheta(\widehat\pi)\|_1.
\]
Hence
\[
  \inf_{\widehat\pi}\sup_{\btheta}
  \EE\big[J_\btheta(\pi_\btheta^\star)-J_\btheta(\widehat\pi)\big]
  \ge
  \frac12 \mathfrak R^*
  \ge
  \frac{H^2}{1024\sqrt n}.
\]
Equivalently, to guarantee suboptimality at most $\varepsilon$, one must have
\[
  \varepsilon \ge \frac{H^2}{1024\sqrt n},
  \qquad\text{that is,}\qquad
  n\ge \frac{H^4}{1024^2   \varepsilon^2}.
\]
Since $C_{sa}=2$ in this construction, this can be written in the standard form
\[
  n=\Omega   \left(\frac{H^4   C_{sa}}{\varepsilon^2}\right).
\]

\subsection[Proof of theorem: exponential lower bound for all-success aggregation]{Proof of \Cref{thm:main-lb-allsuccess}: Exponential Lower Bound for All-Success Aggregation}\label{subsec:lb-allsuccess}

We prove the lower bound for the all-success outcome described in \Cref{ex:allsuccess-outcome}. We construct a hard family of deterministic finite-horizon MDPs with binary rewards. Let the horizon be $H$, the state space be
\[
  \cS=\{s_1,\dots,s_H\},
\]
and the action space be $\cA=\{0,1\}$. The transition kernel is deterministic:
\[
  P_h(s_{h+1}\mid s_h,a)=1,\qquad \forall h\in[H-1],\ a\in\cA.
\]
Thus the state is again only the time index, and the learner chooses one binary action at each step.

The environments differ only in their reward functions. For a parameter
\[
  \btheta=(\theta_1,\dots,\theta_H)\in\{0,1\}^H,
\]
define the binary reward at step $h$ by
\begin{equation}\label{eq:lb-allsuccess-rewards}
  r_{h,\btheta}(s_h,a)=\mathbf 1\{a=\theta_h\}\in\{0,1\},
\end{equation}
and define the observed outcome using the all-success aggregation of \Cref{ex:allsuccess-outcome}:
\[
  Y(\tau)=R_\btheta(\tau)=\prod_{h=1}^H r_{h,\btheta}(s_h,a_h)=\mathbf 1\{   a_h=\theta_h\ \forall h   \}.
\]
The optimal policy $\pi_\btheta^\star$ selects action $\theta_h$ at every step, and therefore
\[
  J_\btheta(\pi_\btheta^\star)=1.
\]

We assume the behavior policy $\mu$ chooses the two actions uniformly at every step. Identifying a trajectory with its action vector
\[
  \xb=(a_1,\dots,a_H)\in\{0,1\}^H,
\]
we have $\xb\sim\mathrm{Unif}(\{0,1\}^H)$ and the observed label is
\[
  y=\mathbf 1\{\xb=\btheta\}.
\]

\paragraph{Coverage}
Since $\mu$ is uniform,
\[
  d_h^\mu(s_h,a)=\frac12.
\]
For the optimal policy, $d_h^{\pi^\star_\btheta}(s_h,\theta_h)=1$, and hence
\[
  C_{sa}(\pi^\star_\btheta)   =   \max_{h,s,a}\frac{d_h^{\pi^\star_\btheta}(s,a)}{d_h^\mu(s,a)}   =   2.
\]
At the trajectory level, however, $\mu(\tau)=2^{-H}$ for every trajectory, while $\pi^\star_\btheta$ puts all its mass on the single trajectory $\xb=\btheta$. Thus
\[
  C_\tau(\pi^\star_\btheta)   \coloneqq   \max_\tau\frac{\PP_{\pi^\star_\btheta}[\tau]}{\PP_\mu[\tau]}   =   2^H.
\]
The construction therefore has constant state--action concentrability but exponential trajectory concentrability.

\paragraph{Reduction to parameter estimation}
For any possibly randomized policy $\widehat\pi$ and any $\btheta\in\{0,1\}^H$, the deterministic-chain structure gives
\begin{equation}\label{eq:lb-allsuccess-valgap}
  J_\btheta(\widehat\pi)=\EE_{\widehat\pi}[R_\btheta(\tau)]=\prod_{h=1}^H \PP_{\widehat\pi}[a_h=\theta_h\mid s_h],
  \qquad
  J_\btheta(\pi^\star_\btheta)=1.
\end{equation}
Given $\widehat\pi$, define an induced estimator $\widehat\btheta(\widehat\pi)$ by choosing the most likely action at every step, breaking ties arbitrarily:
\begin{equation}\label{eq:lb-allsuccess-argmax}
  \widehat\theta_h(\widehat\pi)   \coloneqq   \arg\max_{a\in\{0,1\}}\widehat\pi_h(a\mid s_h)\in\{0,1\},
  \qquad
  \widehat\btheta(\widehat\pi)\coloneqq(\widehat\theta_1,\dots,\widehat\theta_H).
\end{equation}
The following lemma converts value suboptimality into zero-one parameter-estimation loss.

\begin{lemma}[Reduction from value gap to zero-one loss]\label{lem:value-to-param-reduction}
  For every policy $\widehat\pi$ and every $\btheta\in\{0,1\}^H$,
  \[
    J_\btheta(\pi^\star_\btheta)-J_\btheta(\widehat\pi)   \ge   \frac12   \mathbf 1\bigl\{\widehat\btheta(\widehat\pi)\neq\btheta\bigr\}.
  \]
\end{lemma}
\begin{proof}
  If $\widehat\btheta(\widehat\pi)=\btheta$, the right-hand side is zero and there is nothing to prove. Otherwise, there exists $h^\star\in[H]$ such that $\widehat\theta_{h^\star}\neq\theta_{h^\star}$. By \cref{eq:lb-allsuccess-argmax},
  \[
    \widehat\pi_{h^\star}(\theta_{h^\star}\mid s_{h^\star})\le \frac12.
  \]
  Using \cref{eq:lb-allsuccess-valgap} and the bound $\widehat\pi_h(\theta_h\mid s_h)\le 1$ for all $h\neq h^\star$,
  \[
    J_\btheta(\widehat\pi)=   \prod_{h=1}^H\widehat\pi_h(\theta_h\mid s_h)   \le   \widehat\pi_{h^\star}(\theta_{h^\star}\mid s_{h^\star})   \le   \frac12,
  \]
  so $J_\btheta(\pi^\star_\btheta)-J_\btheta(\widehat\pi)\ge 1-1/2=1/2$.
\end{proof}

\paragraph{Indistinguishability via total variation}
Let $P_\btheta$ be the law of one sample $(\xb,y)$, where $\xb\sim\mathrm{Unif}(\{0,1\}^H)$ and $y=\mathbf 1\{\xb=\btheta\}$. Let $P_\btheta^{\otimes n}$ denote the law of $n$ i.i.d.\ samples. We now bound the total variation distance between the observation laws corresponding to two distinct parameters.

\begin{lemma}[Trajectory-coverage TV bound]\label{lem:tv-allsuccess}
  For any distinct $\btheta_0,\btheta_1\in\{0,1\}^H$,
  \[
    \mathrm{TV}(P_{\btheta_0},P_{\btheta_1})=2^{1-H},
    \qquad
    \mathrm{TV}(P_{\btheta_0}^{\otimes n},P_{\btheta_1}^{\otimes n})\le n\cdot 2^{1-H}.
  \]
\end{lemma}
\begin{proof}
  Under both $P_{\btheta_0}$ and $P_{\btheta_1}$, the marginal distribution of $\xb$ is uniform. If $\xb\notin\{\btheta_0,\btheta_1\}$, then $y=0$ under both laws. If $\xb=\btheta_0$, then $y=1$ under $P_{\btheta_0}$ and $y=0$ under $P_{\btheta_1}$; the case $\xb=\btheta_1$ is symmetric. Therefore
  \[
    \mathrm{TV}(P_{\btheta_0},P_{\btheta_1})=\frac12\ \sum_{\xb,y}\bigl|P_{\btheta_0}(\xb,y)-P_{\btheta_1}(\xb,y)\bigr|=\frac12\cdot 2\cdot\bigl(2^{-H}+2^{-H}\bigr)=2^{1-H}.
  \]
  For the $n$-sample bound, couple the two experiments by drawing
  \[
    \xb_1,\dots,\xb_n   \stackrel{\mathrm{i.i.d.}}{\sim}   \mathrm{Unif}(\{0,1\}^H),
    \qquad
    y_i^{(b)}\coloneqq \mathbf 1\{\xb_i=\btheta_b\},\quad b\in\{0,1\}.
  \]
  Under this coupling,
  \[
    X^{(b)}\coloneqq(\xb_1,y_1^{(b)},\dots,\xb_n,y_n^{(b)})
  \]
  has marginal law $P_{\btheta_b}^{\otimes n}$ for each $b\in\{0,1\}$. Define
  \[
    E   \coloneqq   \Bigl\{   \xb_i\notin\{\btheta_0,\btheta_1\}\ \text{for all }i\in[n]   \Bigr\}.
  \]
  On $E$, all labels satisfy $y_i^{(0)}=y_i^{(1)}=0$, so $X^{(0)}=X^{(1)}$. Therefore, for any event $A$ in the observation space,
  \[
    \mathbf 1\{X^{(0)}\in A\}-\mathbf 1\{X^{(1)}\in A\}   =   \bigl(\mathbf 1\{X^{(0)}\in A\}-\mathbf 1\{X^{(1)}\in A\}\bigr)\mathbf 1\{E^c\}   \le   \mathbf 1\{E^c\},
  \]
  where the first equality uses $X^{(0)}=X^{(1)}$ on $E$. Taking expectations gives
  \[
    P_{\btheta_0}^{\otimes n}(A)-P_{\btheta_1}^{\otimes n}(A)   =   \PP[X^{(0)}\in A]-\PP[X^{(1)}\in A]   \le   \PP[E^c].
  \]
  Taking the supremum over $A$,
  \[
    \mathrm{TV}\bigl(P_{\btheta_0}^{\otimes n},P_{\btheta_1}^{\otimes n}\bigr)   =   \sup_A\bigl[P_{\btheta_0}^{\otimes n}(A)-P_{\btheta_1}^{\otimes n}(A)\bigr]   \le   \PP[E^c].
  \]
  By a union bound,
  \[
    \PP[E^c]   \le   \sum_{i=1}^n\PP\bigl[\xb_i\in\{\btheta_0,\btheta_1\}\bigr]   =   n\cdot 2\cdot 2^{-H}   =   n\cdot 2^{1-H},
  \]
  which completes the proof.
\end{proof}

\paragraph{Combining the lemmas}
We now combine the reduction and the TV bound by the two-point Le Cam method. Fix any two distinct parameters $\btheta_0,\btheta_1\in\{0,1\}^H$. By \Cref{lem:value-to-param-reduction}, for every policy $\widehat\pi$ and $i\in\{0,1\}$,
\[
  J_{\btheta_i}(\pi^\star_{\btheta_i})-J_{\btheta_i}(\widehat\pi)   \ge   \frac12   \mathbf 1\bigl\{\widehat\btheta(\widehat\pi)\neq\btheta_i\bigr\}   =   \frac12   \ell_{01}\bigl(\widehat\btheta(\widehat\pi),\btheta_i\bigr),
\]
where
\[
  \ell_{01}(\btheta,\btheta')\coloneqq\mathbf 1\{\btheta\neq\btheta'\}.
\]
Taking expectations under $P_{\btheta_i}^{\otimes n}$ and then taking the supremum over $\btheta\in\{0,1\}^H$,
\begin{equation}\label{eq:lb-reduction-inequality}
  \inf_{\widehat\pi}\sup_{\btheta\in\{0,1\}^H}\EE   \bigl[J_\btheta(\pi^\star_\btheta)-J_\btheta(\widehat\pi)\bigr]
  \ge   \frac12\inf_{\widehat\btheta}\max_{i\in\{0,1\}}\EE_{P_{\btheta_i}^{\otimes n}}   \bigl[\ell_{01}(\widehat\btheta,\btheta_i)\bigr],
\end{equation}
where the infimum on the right ranges over all possibly randomized estimators $\widehat\btheta$ based on the data. Since $\ell_{01}$ is a metric and $\ell_{01}(\btheta_0,\btheta_1)=1$, the two-point Le Cam inequality (see \cite{yu1997assouad}) gives
\begin{equation}\label{eq:lb-lecam-classical}
  \inf_{\widehat\btheta}\max_{i\in\{0,1\}}\EE_{P_{\btheta_i}^{\otimes n}}   \bigl[\ell_{01}(\widehat\btheta,\btheta_i)\bigr]
  \ge   \frac12\bigl(1-\mathrm{TV}(P_{\btheta_0}^{\otimes n},P_{\btheta_1}^{\otimes n})\bigr).
\end{equation}
Combining \cref{eq:lb-reduction-inequality}, \cref{eq:lb-lecam-classical}, and \Cref{lem:tv-allsuccess},
\[
  \inf_{\widehat\pi}\sup_{\btheta\in\{0,1\}^H}\EE   \bigl[J_\btheta(\pi^\star_\btheta)-J_\btheta(\widehat\pi)\bigr]
  \ge   \frac14\bigl(1-\mathrm{TV}(P_{\btheta_0}^{\otimes n},P_{\btheta_1}^{\otimes n})\bigr)
  \ge   \frac14\bigl(1-n\cdot 2^{1-H}\bigr),
\]
as desired. In particular, if $n\le 2^{H-2}$, then the right-hand side is at least $1/8$. Therefore achieving expected suboptimality below any constant $\varepsilon<1/8$ requires
\[
  n   =   \Omega(2^H)   =   \Omega\bigl(C_\tau(\pi^\star)\bigr).
\]
This completes the proof of \Cref{thm:main-lb-allsuccess}.

\section{Generalized Reinforcement Learning Theory}
\label{sec:generalized_rl}

This section develops the RL theory underlying the generalized objective
introduced in \Cref{section:main-result-generalized}. The MDP, policy, and occupancy notation is otherwise the same as in \Cref{section:preliminary}. We restrict attention to Markov policies $\pi=(\pi_h)_{h=1}^H$, where each $\pi_h:\cS_h\to\Delta(\cA_h)$ depends only on the current state. All expectations $\EE_{\tau\sim\pi}$ and occupancies $d_h^\pi$ are taken under the trajectory law induced by such a Markov policy and the fixed transition kernel. As in the main text, we work with the stage-wise affine-in-continuation aggregation
\[
  \sigma_h(u,v)=a_h(u)v+b_h(u),
\]
with terminal constant $g\in\RR$. For any candidate reward $r=(r_h)_{h=1}^H$, define
\begin{align*}
  R_{H+1}(\cdot;r) & \coloneqq g,                                                                                \\
  R_h(\tau_h;r)    & \coloneqq \sigma_h\bigl(r_h(s_h,a_h),   R_{h+1}(\tau_{h+1};r)\bigr),\qquad h=H,H-1,\dots,1,
\end{align*}
and set $R(\tau;r)\coloneqq R_1(\tau_1;r)$ and $J_r(\pi)\coloneqq \EE_{\tau\sim\pi}[R(\tau;r)]$. For the true environment reward $r^\star$, we recover the objective of the main text as $J(\pi)=J_{r^\star}(\pi)$. We establish Bellman evaluation, Bellman optimality, greedy optimality, and the generalized performance-difference lemma with the reward argument kept explicit, as these results are used throughout \Cref{sec:proof-gen-upper-bound}.

\subsection{Generalized Value Functions and Bellman Recursion}

Since $\sigma_h(u,v)$ is affine in its continuation argument $v$, conditional expectations commute with the aggregation in that coordinate. Thus the generalized action-value and state-value functions for a policy $\pi$ and reward $r$,
\begin{align*}
  Q_h^{\pi,r}(s,a) & \coloneqq \EE_\pi[R_h(\tau_h;r) \mid s_h=s,   a_h=a], \\
  V_h^{\pi,r}(s)   & \coloneqq \EE_\pi[R_h(\tau_h;r) \mid s_h=s],
\end{align*}
satisfy the backward Bellman recursion
\begin{align*}
  Q_H^{\pi,r}(s,a) & =  \sigma_H\bigl(r_H(s,a),   g\bigr),                                                                          \\
  Q_h^{\pi,r}(s,a) & =  \sigma_h\bigl(r_h(s,a),   \EE_{s' \sim P_h(\cdot \mid s,a)}[V_{h+1}^{\pi,r}(s')]\bigr), \qquad h \in [H-1], \\
  V_h^{\pi,r}(s)   & = \EE_{a\sim\pi_h(\cdot\mid s)}[ Q_h^{\pi,r}(s,a)],
\end{align*}
with $V^{\pi,r}(s)\coloneqq V_1^{\pi,r}(s)$.

For any collection of stage-wise functions
\[
  f = (f_1,\dots,f_H),
  \qquad
  f_h : \cS \times \cA \to \RR,
\]
define the associated state-value surrogates by
\begin{subequations}
  \begin{align}
    V_{H+1}^{\pi,f} & \coloneqq g, \label{eq:V-pi-f-terminal-det} \\
    V_h^{\pi,f}(s)
                    & \coloneqq
    \EE_{a \sim \pi_h(\cdot \mid s)}[f_h(s,a)],
    \qquad h \in [H], \label{eq:V-pi-f-det}                       \\
    V_{H+1}^{f}     & \coloneqq g,                                \\
    V_h^{f}(s)
                    & \coloneqq
    \max_{a \in \cA} f_h(s,a),
    \qquad h \in [H].
  \end{align}
\end{subequations}
We also use the convention $V_{H+1}^{\pi,r}\coloneqq g$.

\subsection{Generalized Bellman Equations and Optimality}\label{subsec:gen-bellman-opt}

\begin{definition}[Generalized Bellman equation]
  \label{def:bellman-eval-generalized-det}
  Fix a policy $\pi$ and reward $r$. We say that $f = (f_1,\dots,f_H)$ satisfies the Bellman equation for $(\pi,r)$ if, for every $h \in [H]$, $s \in \cS$, and $a \in \cA$,
  \begin{equation}
    f_H(s,a)=\sigma_H\bigl(r_H(s,a),g\bigr),
    \qquad
    f_h(s,a)
    =
    \sigma_h\bigl(r_h(s,a),\EE_{s' \sim P_h(\cdot \mid s,a)}[V_{h+1}^{\pi,f}(s')]\bigr),
    \quad h \in [H-1].
    \label{eq:bellman-eval-generalized-det}
  \end{equation}
\end{definition}

\begin{definition}[Generalized Bellman optimality equation]
  \label{def:bellman-opt-generalized-det}
  Fix a reward $r$. We say that $f = (f_1,\dots,f_H)$ satisfies the Bellman optimality equation for $r$ if, for every $h \in [H]$, $s \in \cS$, and $a \in \cA$,
  \begin{equation}
    f_H(s,a)=\sigma_H\bigl(r_H(s,a),g\bigr),
    \qquad
    f_h(s,a)
    =
    \sigma_h\bigl(r_h(s,a),\EE_{s' \sim P_h(\cdot \mid s,a)}[V_{h+1}^{f}(s')]\bigr),
    \quad h \in [H-1].
    \label{eq:bellman-opt-generalized-det}
  \end{equation}
\end{definition}

Policy evaluation only uses affinity in the continuation value. For optimal control we additionally use the contractivity condition $a_h(u)\in[0,1]$, which makes $v\mapsto\sigma_h(u,v)$ non-decreasing.

\begin{proposition}[Characterization of $Q^{\pi,r}$]
  \label{prop:eval-iff-det}
  Fix a policy $\pi$ and reward $r$. A collection $f = (f_1,\dots,f_H)$ satisfies the Bellman equation~\cref{eq:bellman-eval-generalized-det} for $(\pi,r)$ if and only if
  \[
    f_h(s,a) = Q_h^{\pi,r}(s,a),
    \qquad
    \forall h \in [H],\ s \in \cS,\ a \in \cA.
  \]
  Equivalently, the Bellman equation for $(\pi,r)$ has a unique solution, namely $Q^{\pi,r}$.
\end{proposition}

\begin{proof}
  The ``if'' direction follows immediately from the recursive definitions of $Q^{\pi,r}$ and $V^{\pi,r}$.

  For the ``only if'' direction, suppose $f$ satisfies \cref{eq:bellman-eval-generalized-det}. We prove by backward induction on $h$ that
  \[
    f_h(s,a) = Q_h^{\pi,r}(s,a),
    \qquad
    \forall s \in \cS,\ a \in \cA.
  \]

  At stage $H$, we have
  \[
    f_H(s,a)
    =
    \sigma_H\bigl(r_H(s,a),g\bigr)
    =
    Q_H^{\pi,r}(s,a).
  \]

  Now assume that for some $h < H$,
  \[
    f_{h+1}(s,a) = Q_{h+1}^{\pi,r}(s,a),
    \qquad
    \forall s \in \cS,\ a \in \cA.
  \]
  Then for every $s' \in \cS$,
  \[
    V_{h+1}^{\pi,f}(s')
    =
    \EE_{a' \sim \pi_{h+1}(\cdot \mid s')}[f_{h+1}(s',a')]
    =
    \EE_{a' \sim \pi_{h+1}(\cdot \mid s')}[Q_{h+1}^{\pi,r}(s',a')]
    =
    V_{h+1}^{\pi,r}(s').
  \]
  Substituting this into \cref{eq:bellman-eval-generalized-det} yields
  \[
    f_h(s,a)
    =
    \sigma_h\bigl(r_h(s,a),\EE_{s' \sim P_h(\cdot \mid s,a)}[V_{h+1}^{\pi,r}(s')]\bigr)
    =
    Q_h^{\pi,r}(s,a).
  \]
  This completes the induction.
\end{proof}

We now define the optimal action-value function for reward $r$ by
\begin{align*}
  Q_h^{\star,r}(s,a)
   & \coloneqq
  \sup_{\pi} Q_h^{\pi,r}(s,a),
  \qquad
  h \in [H],\ s \in \cS,\ a \in \cA,
\end{align*}
and the optimal state-value function by
\begin{align*}
  V_h^{\star,r}(s)
   & \coloneqq
  \max_{a \in \cA} Q_h^{\star,r}(s,a),
  \qquad
  h \in [H],
\end{align*}
and set $V_{H+1}^{\star,r} \coloneqq g$.

\begin{theorem}[Bellman optimality equation and greedy optimality]
  \label{thm:opt-iff-greedy-det}
  Fix a reward $r$, and suppose the contractivity condition $a_h(u)\in[0,1]$ holds. Then:
  \begin{enumerate}
    \item The Bellman optimality equation~\cref{eq:bellman-opt-generalized-det} has a unique solution $\bar Q = (\bar Q_1,\dots,\bar Q_H)$.
    \item This solution coincides with the optimal action-value function:
          \[
            \bar Q_h(s,a) = Q_h^{\star,r}(s,a),
            \qquad
            \forall h \in [H],\ s \in \cS,\ a \in \cA.
          \]
          Hence a collection $f = (f_1,\dots,f_H)$ satisfies \cref{eq:bellman-opt-generalized-det} if and only if $f = Q^{\star,r}$.
    \item Any greedy policy with respect to $\bar Q$, namely any policy $\pi^{\mathrm{gr}}$ satisfying
          \[
            \pi_h^{\mathrm{gr}}(a \mid s) = 0,
            \qquad
            \forall a \notin \arg\max_{a' \in \cA} \bar Q_h(s,a'),
          \]
          for every $h \in [H]$ and $s \in \cS$, is optimal for $J_r$. Equivalently, any policy greedy with respect to $Q^{\star,r}$ is optimal for $J_r$.
  \end{enumerate}
\end{theorem}

\begin{proof}
  We divide the proof into three steps.

  \paragraph{Step 1: Existence and uniqueness of a solution to the Bellman optimality equation}
  The equation \cref{eq:bellman-opt-generalized-det} can be solved uniquely by backward recursion.

  At stage $H$, the equation uniquely determines
  \[
    \bar Q_H(s,a)
    =
    \sigma_H\bigl(r_H(s,a),g\bigr).
  \]
  Once $\bar Q_{h+1},\dots,\bar Q_H$ are known, \cref{eq:bellman-opt-generalized-det} uniquely determines $\bar Q_h$ through
  \[
    \bar Q_h(s,a)
    =
    \sigma_h\bigl(r_h(s,a),\EE_{s' \sim P_h(\cdot \mid s,a)}[\bar V_{h+1}(s')]\bigr),
    \qquad
    \bar V_{h+1}(s') \coloneqq \max_{a'} \bar Q_{h+1}(s',a').
  \]
  Thus the solution exists and is unique.

  \paragraph{Step 2: Every policy is dominated by $\bar Q$}
  We claim that for every policy $\pi$,
  \begin{equation}
    Q_h^{\pi,r}(s,a) \le \bar Q_h(s,a),
    \qquad
    \forall h \in [H],\ s \in \cS,\ a \in \cA.
    \label{eq:Qpi_le_Qbar_det}
  \end{equation}
  We prove this by backward induction on $h$.

  At stage $H$, for every policy $\pi$,
  \[
    Q_H^{\pi,r}(s,a)
    =
    \sigma_H\bigl(r_H(s,a),g\bigr)
    =
    \bar Q_H(s,a),
  \]
  so \cref{eq:Qpi_le_Qbar_det} holds.

  Now assume \cref{eq:Qpi_le_Qbar_det} holds at stage $h+1$. Then for every $s' \in \cS$,
  \[
    V_{h+1}^{\pi,r}(s')
    =
    \EE_{a' \sim \pi_{h+1}(\cdot \mid s')}[Q_{h+1}^{\pi,r}(s',a')]
    \le
    \max_{a'} \bar Q_{h+1}(s',a')
    =
    \bar V_{h+1}(s').
  \]
  Taking expectation with respect to $P_h(\cdot \mid s,a)$ in the preceding display and using $a_h(u)\ge0$ gives
  \[
    Q_h^{\pi,r}(s,a)
    =
    \sigma_h\bigl(r_h(s,a),\EE_{s' \sim P_h(\cdot \mid s,a)}[V_{h+1}^{\pi,r}(s')]\bigr)
    \le
    \sigma_h\bigl(r_h(s,a),\EE_{s' \sim P_h(\cdot \mid s,a)}[\bar V_{h+1}(s')]\bigr)
    =
    \bar Q_h(s,a).
  \]
  This proves \cref{eq:Qpi_le_Qbar_det}.

  \paragraph{Step 3: A greedy policy attains $\bar Q$}
  Let $\pi^{\mathrm{gr}}$ be any greedy policy with respect to $\bar Q$, and define
  \[
    \bar V_h(s) \coloneqq \max_{a} \bar Q_h(s,a).
  \]
  We claim that
  \begin{equation}
    Q_h^{\pi^{\mathrm{gr}},r}(s,a) = \bar Q_h(s,a),
    \qquad
    \forall h \in [H],\ s \in \cS,\ a \in \cA,
    \label{eq:greedy_attains_Qbar_det}
  \end{equation}
  and therefore
  \begin{equation}
    V_h^{\pi^{\mathrm{gr}},r}(s) = \bar V_h(s),
    \qquad
    \forall h \in [H],\ s \in \cS.
    \label{eq:greedy_attains_Vbar_det}
  \end{equation}
  Again we use backward induction.

  At stage $H$, the equality is immediate since both sides equal
  \[
    \sigma_H\bigl(r_H(s,a),g\bigr).
  \]

  Assume \cref{eq:greedy_attains_Qbar_det} holds at stage $h+1$. Then for every $s' \in \cS$,
  \[
    V_{h+1}^{\pi^{\mathrm{gr}},r}(s')
    =
    \EE_{a' \sim \pi_{h+1}^{\mathrm{gr}}(\cdot \mid s')}[Q_{h+1}^{\pi^{\mathrm{gr}},r}(s',a')].
  \]
  Since $\pi_{h+1}^{\mathrm{gr}}$ places mass only on maximizers of $\bar Q_{h+1}(s',\cdot)$ and $Q_{h+1}^{\pi^{\mathrm{gr}},r} = \bar Q_{h+1}$ by the induction hypothesis,
  \[
    V_{h+1}^{\pi^{\mathrm{gr}},r}(s')
    =
    \max_{a'} \bar Q_{h+1}(s',a')
    =
    \bar V_{h+1}(s').
  \]
  Using the recursion that defines $Q_h^{\pi^{\mathrm{gr}},r}$, we obtain
  \[
    Q_h^{\pi^{\mathrm{gr}},r}(s,a)
    =
    \sigma_h\bigl(r_h(s,a),\EE_{s' \sim P_h(\cdot \mid s,a)}[\bar V_{h+1}(s')]\bigr)
    =
    \bar Q_h(s,a),
  \]
  which proves \cref{eq:greedy_attains_Qbar_det}. Then
  \[
    V_h^{\pi^{\mathrm{gr}},r}(s)
    =
    \EE_{a \sim \pi_h^{\mathrm{gr}}(\cdot \mid s)}[Q_h^{\pi^{\mathrm{gr}},r}(s,a)]
    =
    \max_a \bar Q_h(s,a)
    =
    \bar V_h(s),
  \]
  proving \cref{eq:greedy_attains_Vbar_det}.

  Combining Steps 2 and 3, for every policy $\pi$,
  \[
    Q_h^{\pi,r}(s,a) \le \bar Q_h(s,a) = Q_h^{\pi^{\mathrm{gr}},r}(s,a),
    \qquad
    \forall h \in [H],\ s \in \cS,\ a \in \cA.
  \]
  Hence
  \[
    \bar Q_h(s,a)
    =
    \sup_\pi Q_h^{\pi,r}(s,a)
    =
    Q_h^{\star,r}(s,a),
  \]
  so $\bar Q = Q^{\star,r}$. This proves item 2. Item 3 follows from
  \[
    V_h^{\pi^{\mathrm{gr}},r}(s)
    =
    \bar V_h(s)
    =
    \max_a Q_h^{\star,r}(s,a)
    =
    V_h^{\star,r}(s),
    \qquad
    \forall h \in [H],\ s \in \cS.
  \]
\end{proof}

\begin{corollary}
  \label{cor:opt-iff-det}
  Fix a reward $r$. Under the contractivity condition $a_h(u)\in[0,1]$, a collection $f = (f_1,\dots,f_H)$ satisfies the Bellman optimality equation \cref{eq:bellman-opt-generalized-det} for $r$ if and only if
  \[
    f_h(s,a) = Q_h^{\star,r}(s,a),
    \qquad
    \forall h \in [H],\ s \in \cS,\ a \in \cA.
  \]
  Moreover, any policy greedy with respect to $f$ is optimal for $J_r$.
\end{corollary}

\begin{proof}
  If $f$ satisfies \cref{eq:bellman-opt-generalized-det}, then \Cref{thm:opt-iff-greedy-det} implies $f = Q^{\star,r}$. Conversely, $Q^{\star,r}$ satisfies \cref{eq:bellman-opt-generalized-det} because it is the unique solution established in \Cref{thm:opt-iff-greedy-det}. The greedy-optimality statement follows from the same theorem.
\end{proof}

\subsection{Generalized Performance Difference Lemma}\label{subsec:gen-pdl}

The following lemma extends the standard performance difference lemma (\Cref{lemma:performance-difference}) to the generalized objective. For a fixed reward $r$, the prefix scaling factors $\prod_{j<h}a_j(r_j(s_j,a_j))$ appear naturally as weights on the per-step advantage terms.

\begin{lemma}[Generalized performance difference lemma]\label{lemma:gen-pdl}
  Under the generalized objective with $\sigma_h(u,v)=a_h(u)   v+b_h(u)$, for any reward $r$ and any two policies $\pi$ and $\pi'$,
  \begin{align*}
    J_r(\pi)-J_r(\pi')
    =
    \sum_{h=1}^H
    \EE_{\tau\sim\pi}   \left[
                          \left(\prod_{j=1}^{h-1}a_j\bigl(r_j(s_j,a_j)\bigr)\right)
                          A_h^{\pi',r}(s_h,a_h)
                          \right],
  \end{align*}
  where $A_h^{\pi',r}(s,a)\coloneqq Q_h^{\pi',r}(s,a)-V_h^{\pi',r}(s)$ and we use the convention $\prod_{j=1}^{0}(\cdot)=1$.
\end{lemma}

\begin{proof}
  We prove by backward induction that for every $h\in[H]$ and $s\in\cS$,
  \begin{equation}\label{eq:gen-pdl-induction}
    V_h^{\pi,r}(s)-V_h^{\pi',r}(s)
    =
    \sum_{h'=h}^{H}
    \EE_{\tau\sim\pi}   \left[
      \left(\prod_{j=h}^{h'-1}a_j(r_j(s_j,a_j))\right)
      A_{h'}^{\pi',r}(s_{h'},a_{h'})
      \middle|    s_h=s
      \right].
  \end{equation}
  The lemma then follows by setting $h=1$ and $s=s_1$.

  \paragraph{Base case ($h=H$)}
  At $h=H$ the only term in the sum is $h'=H$ with an empty product:
  \[
    V_H^{\pi,r}(s)-V_H^{\pi',r}(s)
    =\EE_{a\sim\pi_H}[Q_H^{\pi',r}(s,a)-V_H^{\pi',r}(s)]
    =\EE_{a\sim\pi_H}[A_H^{\pi',r}(s,a)].
  \]
  It remains to verify this equality. Since $V_{H+1}^{\pi,r}=V_{H+1}^{\pi',r}=g$, both $Q_H^{\pi,r}(s,a)$ and $Q_H^{\pi',r}(s,a)$ equal $\sigma_H(r_H(s,a),g)$, so $Q_H^{\pi,r}=Q_H^{\pi',r}$. Therefore
  \[
    V_H^{\pi,r}(s)-V_H^{\pi',r}(s)
    =\EE_{a\sim\pi_H}[Q_H^{\pi,r}(s,a)]-\EE_{a\sim\pi'_H}[Q_H^{\pi',r}(s,a)]
    =\EE_{a\sim\pi_H}[Q_H^{\pi',r}(s,a)]-V_H^{\pi',r}(s)
    =\EE_{a\sim\pi_H}[A_H^{\pi',r}(s,a)].
  \]

  \paragraph{Inductive step}
  Assume \cref{eq:gen-pdl-induction} holds at $h+1$. At stage~$h$, using the Bellman equation
  \[
    Q_h^{\pi,r}(s,a)=a_h(r_h(s,a))   \EE_{s'\sim P_h(\cdot\mid s,a)}[V_{h+1}^{\pi,r}(s')]+b_h(r_h(s,a)),
  \]
  (and analogously for $\pi'$), we get
  \begin{align*}
    Q_h^{\pi,r}(s,a)-Q_h^{\pi',r}(s,a)
    =a_h(r_h(s,a))   \EE_{s'\sim P_h}[V_{h+1}^{\pi,r}(s')-V_{h+1}^{\pi',r}(s')].
  \end{align*}
  Therefore
  \begin{align*}
    V_h^{\pi,r}(s)-V_h^{\pi',r}(s)
     & =\EE_{a\sim\pi_h}[Q_h^{\pi,r}(s,a)]-V_h^{\pi',r}(s)                                                       \\
     & =\EE_{a\sim\pi_h}[Q_h^{\pi',r}(s,a)-V_h^{\pi',r}(s)]+\EE_{a\sim\pi_h}[Q_h^{\pi,r}(s,a)-Q_h^{\pi',r}(s,a)] \\
     & =\EE_{a\sim\pi_h}[A_h^{\pi',r}(s,a)]
    +\EE_{a\sim\pi_h}   \left[a_h(r_h(s,a))   \EE_{s'\sim P_h}[V_{h+1}^{\pi,r}(s')-V_{h+1}^{\pi',r}(s')]\right].
  \end{align*}
  Substituting the induction hypothesis \cref{eq:gen-pdl-induction} at $h+1$ into the second term gives
  \begin{align*}
     & \EE_{a\sim\pi_h}   \left[a_h(r_h(s,a))   \EE_{s'}
                            \sum_{h'=h+1}^{H}
                            \EE_{\tau\sim\pi}   \left[
                              \prod_{j=h+1}^{h'-1}a_j(r_j(s_j,a_j))
                              \cdot A_{h'}^{\pi',r}(s_{h'},a_{h'})
                              \middle|    s_{h+1}=s'
                              \right]\right]              \\
     & =
    \sum_{h'=h+1}^{H}
    \EE_{\tau\sim\pi}   \left[
                          a_h(r_h(s_h,a_h))\prod_{j=h+1}^{h'-1}a_j(r_j(s_j,a_j))
                          \cdot A_{h'}^{\pi',r}(s_{h'},a_{h'})
                          \middle|    s_h=s
                          \right] \\
     & =
    \sum_{h'=h+1}^{H}
    \EE_{\tau\sim\pi}   \left[
                          \prod_{j=h}^{h'-1}a_j(r_j(s_j,a_j))
                          \cdot A_{h'}^{\pi',r}(s_{h'},a_{h'})
                          \middle|    s_h=s
                          \right].
  \end{align*}
  Combining gives
  \[
    V_h^{\pi,r}(s)-V_h^{\pi',r}(s)
    =
    \sum_{h'=h}^{H}
    \EE_{\tau\sim\pi}   \left[
      \prod_{j=h}^{h'-1}a_j(r_j(s_j,a_j))
      \cdot A_{h'}^{\pi',r}(s_{h'},a_{h'})
      \middle|    s_h=s
      \right],
  \]
  completing the induction.
\end{proof}

\section{Proof of the Generalized Outcome-Based Upper Bound}
\label{sec:proof-gen-upper-bound}

This appendix proves \Cref{thm:main-gen-outcome}. The argument follows the same blueprint as the proof of \Cref{theorem:outcome-based-offline-rl} in \Cref{section:proof-additive-upper-bound}: a regret decomposition produces three population losses, each is transferred to the data via concentration, and the resulting AM--GM bound is balanced by tuning $\beta$. The two new ingredients relative to the additive-return analysis are:
\begin{itemize}
  \item the trajectory return is no longer the per-step sum, so the per-step gap $r^\star_h-r_{k,h}$ is replaced by the Bellman--operator gap $\Delta_h^{r_k}\coloneqq(\cT_{r^\star}^{\widehat\pi}f)_h-(\cT_{r_k}^{\widehat\pi}f)_h$, and trajectory return differences no longer decompose into per-step increments;
  \item policy occupancies are weighted by the trajectory factor $W_h^r(\tau)=\prod_{j<h}a_j(r_j(s_j,a_j))$ from \cref{eq:gen-all-losses-emp}, which is in general policy- and reward-dependent.
\end{itemize}
Both are instances of the same principle: the generalized PDL (\Cref{lemma:gen-pdl}) replaces the additive performance-difference identity by a multiplicative one whose weights and per-step gaps couple to the unknown reward $r^\star$.

We adopt the notation of \Cref{section:main-result-generalized,sec:generalized_rl}: the affine aggregation $\sigma_h(u,v)=a_h(u)v+b_h(u)$; the offline dataset $\cD=\{(\tau_i,Y_i)\}_{i=1}^n$ with $\EE[Y_i\mid\tau_i]=R(\tau_i;r^\star)$ and $|Y_i|\le V_{\max}$, so the centred noise $\xi_i\coloneqq Y_i-R(\tau_i;r^\star)$ satisfies $\EE[\xi_i\mid\tau_i]=0$ and $|\xi_i|\le 2V_{\max}$ a.s.; the policy class $\Pi$, reward class $\cR$, and value class $\cF$; the generalized Bellman operator $\cT_r^\pi$ from \cref{eq:def-gen-bellman}, with $\cT^\pi\coloneqq\cT_{r^\star}^\pi$ (by \Cref{prop:eval-iff-det}, $Q^\pi$ is the unique fixed point of $\cT^\pi$); the trajectory return $R(\tau;r)$ from \cref{eq:main-traj-return}; the empirical losses $\cL_\cD^{W^r},\cL_\cD^{\mathrm{BE}},\cL_\cD^{\mathrm{RM}}$ in \cref{eq:gen-all-losses-emp}; the trajectory weights $W_h^r$; the state--action concentrability $C_{sa}(\pi^\star)$ from \cref{eq:def-Csa}; the reward-process complexity $\kappa\coloneqq\kappa_\mu(\sigma)$ from \Cref{def:kappa}; and the Bellman inverse coefficient $\chi\coloneqq\chi_\mu(\sigma)$ from \Cref{def:chi}. Throughout $r^\star=(r_h^\star)_{h=1}^H$ is the true per-step reward and $r=(r_h)_{h=1}^H\in\cR$ a candidate.

\subsection{Setup}
\label{subsec:gen-outcome-offline-rl}

\paragraph{Population losses}
For $(\pi,r,f)\in\Pi\times\cR\times\cF$, the population analogues of \cref{eq:gen-all-losses-emp} are
\begin{equation}\label{eq:gen-pop-losses}
  \begin{aligned}
    \cL_\mu(\pi,f)
     & \coloneqq
    \sum_{h=1}^H\EE_\mu\bigl[f_h(s_h,\pi)-f_h(s_h,a_h)\bigr],
    \\
    \cL_\mu^{W^r}(\pi,r,f)
     & \coloneqq
    \sum_{h=1}^H\EE_\mu\bigl[W_h^r(\tau)\bigl(f_h(s_h,\pi)-f_h(s_h,a_h)\bigr)\bigr],
    \\
    \cL_\mu^{\mathrm{BE}}(\pi,r,f)
     & \coloneqq
    \sum_{h=1}^H\EE_\mu\bigl[\bigl((f_h-\cT_r^\pi f)(s_h,a_h)\bigr)^2\bigr],
    \\
    \cL_\mu^{\mathrm{RM}}(r)
     & \coloneqq
    \EE_{\tau\sim\mu}\bigl[\bigl(R(\tau;r)-R(\tau;r^\star)\bigr)^2\bigr].
  \end{aligned}
\end{equation}
These extend \cref{eq:all_losses} from \Cref{section:proof-additive-upper-bound} in two ways: the policy loss carries the trajectory weight $W_h^r$, and the Bellman-error and reward-model losses use the generalized Bellman operator $\cT_r^\pi$ and trajectory return $R(\tau;\cdot)$ respectively. Because $\sigma_h$ is affine in its continuation, the one-step target $y_h^{r,f,\pi}$ defined in \Cref{section:main-result-generalized} satisfies $\EE[y_h^{r,f,\pi}\mid s_h,a_h]=(\cT_r^\pi f)_h(s_h,a_h)$; subtracting the baseline $\min_{f'\in\cF_h}$ in \cref{eq:gen-loss-be-emp} cancels the conditional variance of $y_h^{r,f,\pi}$, so the empirical loss is an unbiased estimator of $\cL_\mu^{\mathrm{BE}}$ in expectation, the same property used in the proof of \Cref{lemma:bellman-error-concentration}.

\paragraph{Per-step gaps and the $r$-free identity}
Fix $\widehat\pi\in\Pi$, $f\in\cF$, $r\in\cR$, and define
\begin{align*}
  \delta_h(s,a)   & \coloneqq f_h(s,a)-(\cT_r^{\widehat\pi}f)_h(s,a),                                                                                                                                                \\
  \Delta_h^r(s,a) & \coloneqq (\cT_{r^\star}^{\widehat\pi}f)_h(s,a)-(\cT_r^{\widehat\pi}f)_h(s,a)=[a_h(r_h^\star)-a_h(r_h)]\,\EE_{s'\sim P_h(\cdot\mid s,a)}[V_{h+1}^{\widehat\pi,f}(s')]+[b_h(r_h^\star)-b_h(r_h)].
\end{align*}
Both are per-step functions of $(s,a)$. Their difference satisfies the \emph{$r$-free identity}
\begin{equation}\label{eq:gen-bellman-split}
  (\cT_{r^\star}^{\widehat\pi}f)_h(s,a)-f_h(s,a)=\Delta_h^r(s,a)-\delta_h(s,a),
\end{equation}
whose right-hand side does not depend on $r$ even though each summand does. Identity~\cref{eq:gen-bellman-split} is the key replacement, in the generalized setting, of the additive identity $\cT^\pi f-f=r^\star-\tilde r$ used in \Cref{lemma:bound-diff-Q-and-f_pi}.

The next lemma propagates this one-step identity through the generalized Bellman recursion.

\begin{lemma}[Generalized simulation lemma]\label{lemma:gen-simulation}
  Let $e_h(s)\coloneqq V_h^{\widehat\pi}(s)-V_h^{\widehat\pi,f}(s)=\EE_{a\sim\widehat\pi_h(\cdot\mid s)}[Q_h^{\widehat\pi}(s,a)-f_h(s,a)]$ with the convention $e_{H+1}\equiv 0$. Then for every $h\in[H]$, $(s,a)\in\cS\times\cA$,
  \begin{equation}\label{eq:gen-sim-Q}
    Q_h^{\widehat\pi}(s,a)-f_h(s,a)=\Delta_h^r(s,a)-\delta_h(s,a)+a_h(r_h^\star(s,a))\,\EE_{s'\sim P_h(\cdot\mid s,a)}[e_{h+1}(s')].
  \end{equation}
\end{lemma}
\begin{proof}
  Since $Q^{\widehat\pi}$ is the fixed point of $\cT_{r^\star}^{\widehat\pi}$, expanding via $\sigma_h(u,v)=a_h(u)v+b_h(u)$ gives
  \[
    Q_h^{\widehat\pi}(s,a)-(\cT_{r^\star}^{\widehat\pi}f)_h(s,a)=a_h(r_h^\star(s,a))\,\EE_{s'}[V_{h+1}^{\widehat\pi}(s')-V_{h+1}^{\widehat\pi,f}(s')]=a_h(r_h^\star(s,a))\,\EE_{s'}[e_{h+1}(s')],
  \]
  since the $b_h(r_h^\star)$ terms cancel. Adding $\bigl[(\cT_{r^\star}^{\widehat\pi}f)_h-f_h\bigr]=\Delta_h^r-\delta_h$ from \cref{eq:gen-bellman-split} yields \cref{eq:gen-sim-Q}.
\end{proof}

\subsection{Generalized Regret Decomposition}
\label{subsec:gen-regret-decomp-fake}

\begin{lemma}[Generalized regret decomposition]\label{lemma:gen-regret-sharp}
  For any $\pi,\widehat\pi\in\Pi$, $r\in\cR$, and $f\in\cF$, with $\delta_h$, $\Delta_h^r$ from \cref{eq:gen-bellman-split} and $W_h^{r^\star}$ from \cref{eq:gen-all-losses-emp},
  \begin{align}
    J(\pi)-J(\widehat\pi)
     & =\underbrace{\sum_{h=1}^H\EE_\pi\bigl[W_h^{r^\star}(\tau)\bigl(\Delta_h^r-\delta_h\bigr)(s_h,a_h)\bigr]}_{(\mathrm{I}^\sharp)}+\underbrace{\sum_{h=1}^H\EE_\pi\bigl[W_h^{r^\star}(\tau)\bigl(f_h(s_h,a_h)-f_h(s_h,\widehat\pi)\bigr)\bigr]}_{(\mathrm{II}^\sharp)}+\underbrace{f_1(s_1,\widehat\pi)-J(\widehat\pi)}_{(\mathrm{III}^\sharp)}.
    \label{eq:gen-regret-sharp}
  \end{align}
  The right-hand side does not depend on the choice of $r\in\cR$, by \cref{eq:gen-bellman-split}.
\end{lemma}
\begin{proof}
  By the generalized PDL (\Cref{lemma:gen-pdl}) at reward $r^\star$,
  $J(\pi)-J(\widehat\pi)=\sum_h\EE_\pi\bigl[W_h^{r^\star}(\tau)\,A_h^{\widehat\pi}(s_h,a_h)\bigr]$
  with $A_h^{\widehat\pi}(s,a)=Q_h^{\widehat\pi}(s,a)-V_h^{\widehat\pi}(s)$. Using $V_h^{\widehat\pi,f}(s)=f_h(s,\widehat\pi)$, decompose
  \[
    A_h^{\widehat\pi}(s,a)=\bigl[Q_h^{\widehat\pi}(s,a)-f_h(s,a)\bigr]+\bigl[f_h(s,a)-f_h(s,\widehat\pi)\bigr]-e_h(s).
  \]
  The simulation lemma \cref{eq:gen-sim-Q} gives $Q_h^{\widehat\pi}-f_h=(\Delta_h^r-\delta_h)+a_h(r_h^\star)\,\EE_{s'}[e_{h+1}]$, and the identity $W_h^{r^\star}(\tau)\cdot a_h(r_h^\star(s_h,a_h))=W_{h+1}^{r^\star}(\tau)$ together with the tower rule yields
  \[
    \sum_{h=1}^H\EE_\pi\bigl[W_h^{r^\star}\cdot a_h(r_h^\star)\,\EE_{s'}[e_{h+1}(s')]\bigr]
    =\sum_{h=2}^H\EE_\pi[W_h^{r^\star}\,e_h(s_h)],
  \]
  using $e_{H+1}\equiv 0$. Combining the displays and using $W_1^{r^\star}=1$, $e_1(s_1)=J(\widehat\pi)-f_1(s_1,\widehat\pi)$, the only surviving $e_h$ term is $-\EE_\pi[W_1^{r^\star}e_1(s_1)]=f_1(s_1,\widehat\pi)-J(\widehat\pi)$, which is $(\mathrm{III}^\sharp)$. The remaining two terms are $(\mathrm{I}^\sharp)$ and $(\mathrm{II}^\sharp)$.
\end{proof}

\subsection{Auxiliary Lemmas}
\label{subsec:gen-aux-lemmas}

We collect the structural lemmas used to bound the three terms of \cref{eq:gen-regret-sharp}. Throughout, $\delta_h$ and $\Delta_h^r$ are as in \cref{eq:gen-bellman-split}.

\paragraph{Fake reward}
Given a critic $f\in\cF$ and a policy $\widehat\pi\in\Pi$, we introduce a \emph{fake reward} $\tilde r=(\tilde r_1,\ldots,\tilde r_H)$ by forcing $f$ to be Bellman-consistent with $\widehat\pi$:
\begin{equation}\label{eq:fake-reward-gen}
  \sigma_h\bigl(\tilde r_h(s,a),\,\EE_{s'\sim P_h(\cdot\mid s,a)}[V_{h+1}^{\widehat\pi,f}(s')]\bigr)=f_h(s,a),\qquad h\in[H].
\end{equation}
The existence of a measurable fake reward is guaranteed by \Cref{assumption:sigma-regularity}: for each $f,\widehat\pi,h,s,a$, the equation above has a solution $\tilde r_h(s,a)$.
Thus $f$ is exactly the value function of $\widehat\pi$ in the auxiliary MDP with reward $\tilde r$: $f=Q^{\widehat\pi,\tilde r}$ and $J_{\tilde r}(\widehat\pi)=f_1(s_1,\widehat\pi)$. The fake reward is not estimated by the algorithm; it is a proof device that converts value-function inconsistency into a reward-level discrepancy. We write
\[
  \tilde W_h(\tau)\coloneqq\prod_{j=1}^{h-1}a_j(\tilde r_j(s_j,a_j)),\qquad
  G(\tau)\coloneqq R(\tau;r^\star)-R(\tau;\tilde r).
\]

The proof below uses the Bellman inverse coefficient only through the consequence
\begin{equation}\label{eq:assump-inv-L2}
  \sum_{h=1}^H\EE_\mu\bigl[(\tilde r_h-r_h)^2\bigr]\le \chi\,\cL_\mu^{\mathrm{BE}}(\widehat\pi,r,f),\qquad\forall(\widehat\pi,r,f)\in\Pi\times\cR\times\cF.
\end{equation}
This follows from \Cref{def:chi} by applying the coefficient to $(r,\tilde r)$, since $(\cT_{\tilde r}^{\widehat\pi}f)_h=f_h$.

\paragraph{Trajectory return telescoping}
The next lemma gives an exact expression for the difference between two generalized trajectory returns, by interpolating one reward at a time.

\begin{lemma}[Trajectory return telescoping]\label{lemma:traj-return-telescope}
  For any rewards $r,r'$ and trajectory $\tau=(s_1,a_1,\ldots,s_H,a_H)$,
  \begin{equation}\label{eq:traj-telescope-lemma}
    R(\tau;r')-R(\tau;r)
    =\sum_{k=1}^H\Bigl(\prod_{j=1}^{k-1}a_j(r'_j(s_j,a_j))\Bigr)\bigl\{[a_k(r'_k)-a_k(r_k)]\,R_{k+1}(\tau_{k+1};r)+[b_k(r'_k)-b_k(r_k)]\bigr\},
  \end{equation}
  where $R_{k+1}(\tau_{k+1};r)$ is the generalized return from step $k+1$ onward under reward $r$.
\end{lemma}
\begin{proof}
  Define interpolated rewards $r^{(k)}_h=r'_h$ for $h\le k$ and $r^{(k)}_h=r_h$ for $h>k$, so that $r^{(0)}=r$ and $r^{(H)}=r'$. Then $R(\tau;r')-R(\tau;r)=\sum_k[R(\tau;r^{(k)})-R(\tau;r^{(k-1)})]$. Since $r^{(k)}$ and $r^{(k-1)}$ agree on $h>k$, $R_{k+1}(\tau_{k+1};r^{(k)})=R_{k+1}(\tau_{k+1};r^{(k-1)})=R_{k+1}(\tau_{k+1};r)$, so at step~$k$,
  \[
    R_k(\tau_k;r^{(k)})-R_k(\tau_k;r^{(k-1)})=[a_k(r'_k)-a_k(r_k)]\,R_{k+1}(\tau_{k+1};r)+[b_k(r'_k)-b_k(r_k)].
  \]
  For $h<k$ both use $r'_h$, so the difference propagates by the multiplicative factor $a_h(r'_h)$. Unrolling from $h=1$ to $h=k-1$ gives \cref{eq:traj-telescope-lemma} for the $k$-th summand; sum over $k$.
\end{proof}

\begin{lemma}[Telescoping upper bound, contractive case]\label{lemma:telescope-bound}
  Suppose $a_h(u)\in[0,1]$, $a_h,b_h$ are $L$-Lipschitz, and $|R_h(\tau_h;r)|\le V_{\max}$ for all $h,r,\tau$. Then, for any reward functions $r,r'$ and any trajectory~$\tau$,
  \[
    |R(\tau;r')-R(\tau;r)|\le L(V_{\max}+1)\sum_{h=1}^H|r'_h-r_h|(s_h,a_h).
  \]
\end{lemma}
\begin{proof}
  Apply \Cref{lemma:traj-return-telescope}. Since $a_h(u)\in[0,1]$, $\prod_{j<k}a_j(r'_j)\in[0,1]$; since $a_h,b_h$ are $L$-Lipschitz, $|a_k(r'_k)-a_k(r_k)|\le L|r'_k-r_k|$ and similarly for $b_k$; and by boundedness, $|R_{k+1}(\tau_{k+1};r)|\le V_{\max}$.
\end{proof}

\begin{lemma}[Weight Lipschitz bound]\label{lemma:weight-lip}
  Suppose $a_h(u)\in[0,1]$ and $a_h$ is $L$-Lipschitz. Then, for any reward functions $r,r'$ and any trajectory~$\tau$,
  \[
    |W_h^{r'}(\tau)-W_h^r(\tau)|   \le   L\sum_{j=1}^{h-1}|r'_j-r_j|(s_j,a_j).
  \]
\end{lemma}
\begin{proof}
  Write $p_j\coloneqq a_j(r_j(s_j,a_j))$ and $p'_j\coloneqq a_j(r'_j(s_j,a_j))$. The standard telescoping identity gives
  \[
    \prod_{j=1}^{h-1}p'_j-\prod_{j=1}^{h-1}p_j
    =\sum_{k=1}^{h-1}\Bigl(\prod_{j<k}p'_j\Bigr)(p'_k-p_k)\Bigl(\prod_{k<j\le h-1}p_j\Bigr).
  \]
  Since $a_h(u)\in[0,1]$, $p_j,p'_j\in[0,1]$, so both surrounding products are bounded by~$1$. Since $a_h$ is $L$-Lipschitz, $|p'_k-p_k|\le L|r'_k-r_k|(s_k,a_k)$. Summing over $k$ yields the claim.
\end{proof}

\begin{lemma}[Trajectory $L^2$ bound on $G$]\label{lemma:G-L2-bound}
  Suppose $a_h(u)\in[0,1]$, $a_h,b_h$ are $L$-Lipschitz, and $|R_h(\tau_h;r)|\le V_{\max}$ for all $h,r,\tau$. Then, for any $r\in\cR$, $f\in\cF$, $\widehat\pi\in\Pi$,
  \[
    \EE_\mu[G(\tau)^2]   \le   2   \cL_\mu^{\mathrm{RM}}(r)+2   L^2(V_{\max}+1)^2 \chi   H   \cL_\mu^{\mathrm{BE}}(\widehat\pi,r,f).
  \]
\end{lemma}
\begin{proof}
  Write $G=(R(\tau;r^\star)-R(\tau;r))+(R(\tau;r)-R(\tau;\tilde r))$ and use $(x+y)^2\le 2x^2+2y^2$. The first part is $\cL_\mu^{\mathrm{RM}}(r)$. For the second, \Cref{lemma:telescope-bound} with $r'=\tilde r$ gives $|R(\tau;r)-R(\tau;\tilde r)|\le L(V_{\max}+1)\sum_h|r_h-\tilde r_h|$; squaring and applying Cauchy--Schwarz,
  \[
    (R(\tau;r)-R(\tau;\tilde r))^2\le L^2(V_{\max}+1)^2   H\sum_h(r_h-\tilde r_h)^2.
  \]
  Taking $\EE_\mu$ and applying \cref{eq:assump-inv-L2} yields $\EE_\mu[(R(\tau;r)-R(\tau;\tilde r))^2]\le L^2(V_{\max}+1)^2 \chi H   \cL_\mu^{\mathrm{BE}}(\widehat\pi,r,f)$.
\end{proof}

\subsection[Proof of theorem: generalized outcome-based upper bound]{Proof of \Cref{thm:main-gen-outcome}}
\label{subsec:gen-main-result}

We use the inverse-temperature parameter $\eta=V_{\max}\sqrt{{K}/{(8\log|\cA|)}}$,
as in \Cref{lemma:hedge-regret}.

\begin{proof}[Proof of \Cref{thm:main-gen-outcome}]
  Apply \Cref{lemma:gen-regret-sharp} with comparator $\pi=\pi^\star$, current policy $\widehat\pi=\pi_k$, and the pair $(r,f)=(r_k,f_k)$ chosen by Generalized OPAC. For each $k\in[K]$,
  \begin{align}
    J(\pi^\star)-J(\pi_k)
     & =
    \underbrace{\sum_{h=1}^H\EE_{\pi^\star}\bigl[W_h^{r^\star}(\tau)(\Delta_h^{r_k}-\delta_{k,h})(s_h,a_h)\bigr]}_{(\mathrm{I}^\sharp)_k}
    +\underbrace{\sum_{h=1}^H\EE_{\pi^\star}\bigl[W_h^{r^\star}(\tau)\bigl(f_{k,h}(s_h,a_h)-f_{k,h}(s_h,\pi_k)\bigr)\bigr]}_{(\mathrm{II}^\sharp)_k}\notag \\
     & \quad
    +\underbrace{f_{k,1}(s_1,\pi_k)-J(\pi_k)}_{(\mathrm{III}^\sharp)_k},
    \label{eq:gen-regret-sharp-k}
  \end{align}
  where
  \[
    \delta_{k,h}\coloneqq f_{k,h}-(\cT_{r_k}^{\pi_k}f_k)_h,\qquad
    \Delta_h^{r_k}\coloneqq(\cT_{r^\star}^{\pi_k}f_k)_h-(\cT_{r_k}^{\pi_k}f_k)_h.
  \]
  We average the three terms in \cref{eq:gen-regret-sharp-k} over $k\in[K]$ on a single high-probability event on which the concentration bounds below hold uniformly over $\Pi\times\cR\times\cF$.

  \medskip
  \noindent\textbf{Step 1: bound $(\mathrm{I}^\sharp)$.}
  Since $a_h(u)\in[0,1]$, $W_h^{r^\star}\in[0,1]$. Thus
  \[
    |(\mathrm{I}^\sharp)_k|
    \le\sum_{h=1}^H\EE_{d_h^{\pi^\star}}\bigl[|\delta_{k,h}|+|\Delta_h^{r_k}|\bigr].
  \]
  The pointwise change of measure $d_h^{\pi^\star}(s,a)\le C_{sa}(\pi^\star)d_h^\mu(s,a)$ and Cauchy--Schwarz give
  \[
    \sum_h\EE_{d_h^{\pi^\star}}[|\delta_{k,h}|]\le\sqrt{HC_{sa}(\pi^\star)\cL_\mu^{\mathrm{BE}}(\pi_k,r_k,f_k)}.
  \]
  Since $a_h,b_h$ are $L$-Lipschitz and $|V_{h+1}^{\pi_k,f_k}|\le V_{\max}$,
  \[
    |\Delta_h^{r_k}(s,a)|
    \le L(V_{\max}+1)|r_h^\star-r_{k,h}|(s,a).
  \]
  Applying the same change of measure and using the definition of $\kappa=\kappa_\mu(\sigma)$ from \Cref{def:kappa},
  \begin{align}
    \sum_h\EE_{d_h^{\pi^\star}}[|\Delta_h^{r_k}|]\le L(V_{\max}+1)\sqrt{\kappa HC_{sa}(\pi^\star)\cL_\mu^{\mathrm{RM}}(r_k)}.\label{eq:kappa-enter-bound}
  \end{align}
  Combining,
  \begin{equation}\label{eq:I-summary}
    |(\mathrm{I}^\sharp)_k|
    \le\phi_{\mathrm{BE}}^{\mathrm{I}}\sqrt{\cL_\mu^{\mathrm{BE}}(\pi_k,r_k,f_k)}+\phi_{\mathrm{RM}}^{\mathrm{I}}\sqrt{\cL_\mu^{\mathrm{RM}}(r_k)},
  \end{equation}
  with $\phi_{\mathrm{BE}}^{\mathrm{I}}\coloneqq\sqrt{HC_{sa}(\pi^\star)}$ and $\phi_{\mathrm{RM}}^{\mathrm{I}}\coloneqq L(V_{\max}+1)\sqrt{\kappa HC_{sa}(\pi^\star)}$.

  \begin{remark}\label{remark:kappa-enter-bound}
    \cref{eq:kappa-enter-bound} is the only place where $\kappa_\mu(\sigma)$ enters the final bound; in the classical cumulative-return case, the same estimate reduces to one application of change of trajectory measure (\Cref{lemma:change-of-trajectory}). For this generalized analysis, one may alternatively use
    \begin{align*}
      \sum_h\EE_{d_h^{\pi^\star}}[|\Delta_h^{r_k}|]\le L(V_{\max}+1)\sqrt{H\kappa_{\pi^\star}C_{\tau}(\pi^\star) \cL_\mu^{\mathrm{RM}}(r_k)}, \quad\kappa_{\pi^\star} \coloneqq \sup_{r\ne r'}\frac{\sum_{h=1}^H\EE_{\pi^\star}\bigl[(r_h-r'_h)^2(s_h,a_h)\bigr]}{\EE_{\pi^\star}\bigl[(R(\tau;r)-R(\tau;r'))^2\bigr]},
    \end{align*}
    which inserts $\kappa_{\pi^\star}C_{\tau}(\pi^\star)$ in place of $\kappa_\mu C_{sa}(\pi^\star)$. This trajectory-level form can be tighter when $\mu$ places sufficient mass on trajectories visited under $\pi^\star$. For consistency with the rest of the paper, which is stated under state--action concentrability $C_{sa}(\pi^\star)$ rather than trajectory concentrability $C_{\tau}(\pi^\star)$, we retain the present analysis.
  \end{remark}

  \medskip
  \noindent\textbf{Step 2: bound $(\mathrm{II}^\sharp)$.}
  Let $\mathcal F_{h-1}$ denote the trajectory prefix before action $a_h$ is drawn. Conditional on $\mathcal F_{h-1}$, the state $s_h$ and the weight
  \[
    W_h^{r^\star}(\tau)=\prod_{j<h}a_j(r_j^\star(s_j,a_j))
  \]
  are fixed, and $W_h^{r^\star}(\tau)\in[0,1]$. Therefore
  \[
    \EE_{\pi^\star}\!\left[
      W_h^{r^\star}(\tau)
      \bigl(f_{k,h}(s_h,a_h)-f_{k,h}(s_h,\pi_k)\bigr)
      \middle| \mathcal F_{h-1}
      \right]
    =
    W_h^{r^\star}(\tau)
    \bigl(f_{k,h}(s_h,\pi^\star)-f_{k,h}(s_h,\pi_k)\bigr).
  \]
  For every fixed $h$ and state $s$, the pointwise bound \cref{eq:pointwise-regret} in the proof of \Cref{lemma:hedge-regret}, with $p=\pi_h^\star(\cdot\mid s)$ and $q_k=\pi_{k,h}(\cdot\mid s)$, gives
  \[
    \frac1K\sum_{k=1}^K
    \bigl(f_{k,h}(s,\pi^\star)-f_{k,h}(s,\pi_k)\bigr)
    \le
    V_{\max}\sqrt{\frac{\log|\cA|}{2K}},
  \]
  since $f_{k,h}(s,\cdot)\in[0,V_{\max}]$. Multiplying this pointwise inequality by the nonnegative prefix weight $W_h^{r^\star}(\tau)\le1$, taking expectation over the prefix, and summing over $h$ yields
  \begin{align}
    \frac{1}{K}\sum_{k=1}^K(\mathrm{II}^\sharp)_k
    \le HV_{\max}\sqrt{\frac{\log|\cA|}{2K}}
    =\widetilde O\bigl(HV_{\max}/\sqrt K\bigr).\label{eq:II-summary}
  \end{align}

  \medskip
  \noindent\textbf{Step 3: bound $(\mathrm{III}^\sharp)$.}
  Fix $k$ and write $\widehat\pi=\pi_k$, $r=r_k$, and $f=f_k$.

  \emph{(3a) Express $(\mathrm{III}^\sharp)$ via population losses.} Apply the generalized PDL (\Cref{lemma:gen-pdl}) with the first policy $\mu$ and the comparator policy $\widehat\pi$. Under reward $\tilde r$, using $f=Q^{\widehat\pi,\tilde r}$ and hence $J_{\tilde r}(\widehat\pi)=f_1(s_1,\widehat\pi)$,
  \[
    f_1(s_1,\widehat\pi)-J_{\tilde r}(\mu)
    =\sum_h\EE_\mu\bigl[\tilde W_h(\tau)(f_h(s_h,\widehat\pi)-f_h(s_h,a_h))\bigr]
    =\cL_\mu^{W^{\tilde r}}(\widehat\pi,\tilde r,f),
  \]
  and under reward $r^\star$,
  \[
    J(\widehat\pi)-J(\mu)
    =\sum_h\EE_\mu\bigl[W_h^{r^\star}(\tau)(V_h^{\widehat\pi}(s_h)-Q_h^{\widehat\pi}(s_h,a_h))\bigr]
    =\cL_\mu^{W^{r^\star}}(\widehat\pi,r^\star,Q^{\widehat\pi}),
  \]
  where both last equalities use the definition of the weighted population loss in \cref{eq:gen-pop-losses}. Subtracting and using $J_{\tilde r}(\mu)-J(\mu)=-\EE_\mu[G(\tau)]$ with $G(\tau)\coloneqq R(\tau;r^\star)-R(\tau;\tilde r)$,
  \begin{equation}\label{eq:gen-pessimism-via-mu}
    (\mathrm{III}^\sharp)=\cL_\mu^{W^{\tilde r}}(\widehat\pi,\tilde r,f)-\cL_\mu^{W^{r^\star}}(\widehat\pi,r^\star,Q^{\widehat\pi})-\EE_\mu[G].
  \end{equation}

  \emph{(3b) Pivot via the in-class realizer.} Let $f_{\widehat\pi}\in\cF$ satisfy
  \[
    \sup_{h,\nu}
    \Vert f_{\widehat\pi,h}-(\cT_{r^\star}^{\widehat\pi}f_{\widehat\pi})_h\Vert_{2,\nu}^2
    \le \varepsilon_\cF,
  \]
  as guaranteed by the generalized realizability assumption. Insert $\pm\cL_\mu^{W^r}(\widehat\pi,r,f)$ and $\pm\cL_\mu^{W^{r^\star}}(\widehat\pi,r^\star,f_{\widehat\pi})$:
  \begin{align}
    (\mathrm{III}^\sharp)
     & =\underbrace{\cL_\mu^{W^{\tilde r}}(\widehat\pi,\tilde r,f)-\cL_\mu^{W^r}(\widehat\pi,r,f)}_{\text{(weight replacement)}}
    +\underbrace{\cL_\mu^{W^r}(\widehat\pi,r,f)-\cL_\mu^{W^{r^\star}}(\widehat\pi,r^\star,f_{\widehat\pi})}_{\text{(algorithm gap)}}\notag                                               \\
     & \quad+\underbrace{\cL_\mu^{W^{r^\star}}(\widehat\pi,r^\star,f_{\widehat\pi})-\cL_\mu^{W^{r^\star}}(\widehat\pi,r^\star,Q^{\widehat\pi})}_{\text{(realizability gap)}}-\EE_\mu[G].
    \label{eq:gen-III-split}
  \end{align}

  \emph{(3c) Concentration estimates.} On the high-probability event, the following generalized concentration bounds hold uniformly over $(\pi,r,f)\in\Pi\times\cR\times\cF$. They use the same Bernstein/Hoeffding arguments as in \Cref{subsection:concentration}, but with three generalized ingredients: the Bellman target is $y_h^{r,f,\pi}=\sigma_h(r_h,f_{h+1}(\cdot,\pi))$, the reward-model residual is the trajectory return $R(\tau;r)-R(\tau;r^\star)$, and the policy-loss summand carries the trajectory weight $W_h^r$.
  \begin{itemize}
    \item \emph{BE concentration}. Conditional on $(s_h,a_h)$,
          \[
            \EE[y_h^{r,f,\pi}\mid s_h,a_h]=(\cT_r^\pi f)_h(s_h,a_h),
          \]
          so the same bias--variance cancellation behind \Cref{lemma:bellman-error-concentration} applies with $\cT_r^\pi$ in place of the additive Bellman operator. The only change is that boundedness is now supplied by the generalized assumptions, giving per-step squared residuals of order $V_{\max}^2$ and hence
          \[
            \cL_\mu^{\mathrm{BE}}(\pi,r,f)\le 2\cL_\cD^{\mathrm{BE}}(\pi,r,f)+\varepsilon_{\mathrm{BE}}+4\varepsilon_{\cF,\cF},\qquad \varepsilon_{\mathrm{BE}}=\widetilde O(HV_{\max}^2/n).
          \]
          The $4\varepsilon_{\cF,\cF}$ term is the generalized Bellman-completeness slack.
    \item \emph{RM concentration, $r^\star$-centred}. Apply the Bernstein self-bounding proof of \Cref{lem:reward-model-concentration-bernstein} to
          \[
            U_i(r)\coloneqq(R(\tau_i;r)-Y_i)^2-(R(\tau_i;r^\star)-Y_i)^2.
          \]
          Since $\EE[Y_i\mid\tau_i]=R(\tau_i;r^\star)$, $\EE[U_i(r)]=\cL_\mu^{\mathrm{RM}}(r)$; since $|R(\tau;r)|\le V_{\max}$ and $|Y_i-R(\tau_i;r^\star)|\le 2V_{\max}$, the range and variance terms are of order $V_{\max}^2$:
          \[
            \cL_\mu^{\mathrm{RM}}(r)\le 2[\cL_\cD^{\mathrm{RM}}(r)-\cL_\cD^{\mathrm{RM}}(r^\star)]+\varepsilon_{\mathrm{RM}},\qquad \varepsilon_{\mathrm{RM}}=\widetilde O(V_{\max}^2/n).
          \]
    \item \emph{Weighted policy-loss concentration}. This is Hoeffding's inequality as in \Cref{lemma:concentration-L}, but union-bounded over $(\pi,r,f)$ rather than only $(\pi,f)$. The extra reward argument only changes the summand through the measurable weight $W_h^r(\tau)$; contractivity gives $W_h^r\in[0,1]$, so the per-trajectory summand remains bounded by $O(HV_{\max})$:
          \[
            \sup_{(\pi,r,f)}\bigl|\cL_\mu^{W^r}(\pi,r,f)-\cL_\cD^{W^r}(\pi,r,f)\bigr|\le\sqrt{\varepsilon_{\mathrm{Perf}}},\qquad\sqrt{\varepsilon_{\mathrm{Perf}}}=\widetilde O(HV_{\max}/\sqrt n).
          \]
  \end{itemize}

  \emph{(3d) Algorithm gap.} Since $r^\star\in\cR$, the pessimism step \cref{eq:gen-alg-pessimism} can be compared with $(f_{\widehat\pi},r^\star)$. After subtracting $\beta\cL_\cD^{\mathrm{RM}}(r^\star)$ from both sides,
  \begin{equation}\label{eq:gen-stepC-centred}
    \cL_\cD^{W^r}(\widehat\pi,r,f)+\beta\cL_\cD^{\mathrm{BE}}(\widehat\pi,r,f)+\beta\bigl[\cL_\cD^{\mathrm{RM}}(r)-\cL_\cD^{\mathrm{RM}}(r^\star)\bigr]\le\cL_\cD^{W^{r^\star}}(\widehat\pi,r^\star,f_{\widehat\pi})+\beta\varepsilon_{\mathrm{apx}},
  \end{equation}
  where the generalized analogue of \Cref{lemma:bounded-em-Bellman-error} gives
  \[
    \cL_\cD^{\mathrm{BE}}(\widehat\pi,r^\star,f_{\widehat\pi})
    \le \varepsilon_{\mathrm{apx}},
    \qquad
    \varepsilon_{\mathrm{apx}}=\widetilde O(HV_{\max}^2/n+\varepsilon_\cF).
  \]
  Combining \cref{eq:gen-stepC-centred} with the three concentration bounds, the centred noisy term $\cL_\cD^{\mathrm{RM}}(r^\star)$ cancels exactly, and
  \begin{equation}\label{eq:alg-gap-bound}
    \cL_\mu^{W^r}(\widehat\pi,r,f)-\cL_\mu^{W^{r^\star}}(\widehat\pi,r^\star,f_{\widehat\pi})
    \le-\frac{\beta}{2}(\cL_\mu^{\mathrm{BE}}+\cL_\mu^{\mathrm{RM}})+\frac{\beta}{2}(\varepsilon_{\mathrm{BE}}+\varepsilon_{\mathrm{RM}}+4\varepsilon_{\cF,\cF})+\beta\varepsilon_{\mathrm{apx}}+2\sqrt{\varepsilon_{\mathrm{Perf}}}.
  \end{equation}

  \emph{(3e) Realizability gap (third bracket).} Using $W_h^{r^\star}\in[0,1]$ and the triangle inequality,
  \[
    \bigl|\cL_\mu^{W^{r^\star}}(\widehat\pi,r^\star,f_{\widehat\pi})-\cL_\mu^{W^{r^\star}}(\widehat\pi,r^\star,Q^{\widehat\pi})\bigr|
    \le\sum_{h=1}^H\bigl(\EE_\mu[|f_{\widehat\pi,h}-Q_h^{\widehat\pi}|(s_h,\widehat\pi)]+\EE_\mu[|f_{\widehat\pi,h}-Q_h^{\widehat\pi}|(s_h,a_h)]\bigr).
  \]
  Apply \Cref{lemma:gen-simulation} with $r=r^\star$ and $f=f_{\widehat\pi}$, so $\Delta_h^{r^\star}\equiv 0$. Since $a_h(u)\in[0,1]$, unrolling \cref{eq:gen-sim-Q} gives
  \[
    |f_{\widehat\pi,h}-Q_h^{\widehat\pi}|(s,a)\le\EE^{\widehat\pi}\Bigl[\textstyle\sum_{h'\ge h}|\delta_{h'}^{\widehat\pi}|\,\Big|\,s_h=s,a_h=a\Bigr],
    \qquad \delta_{h'}^{\widehat\pi}\coloneqq f_{\widehat\pi,h'}-(\cT_{r^\star}^{\widehat\pi}f_{\widehat\pi})_{h'}.
  \]
  Taking expectations in the preceding display and using the tower rule, the two terms in the first display are bounded by sums of $L^1$ norms of $\delta_{h'}^{\widehat\pi}$ under rollout distributions obtained either from $(s_h,\widehat\pi(s_h))$ or from $(s_h,a_h)$ and then following $\widehat\pi$. These are state--action distributions at step $h'$, so the uniform realizability bound above and Jensen's inequality give
  \[
    \EE_{\nu}\bigl[|\delta_{h'}^{\widehat\pi}|\bigr]
    \le
    \|\delta_{h'}^{\widehat\pi}\|_{2,\nu}
    \le
    \sqrt{\varepsilon_\cF}.
  \]
  Summing the at-most $H^2$ residual terms for each of the two parts in the first display yields
  \begin{equation}\label{eq:gen-realizability-gap}
    \bigl|\cL_\mu^{W^{r^\star}}(\widehat\pi,r^\star,f_{\widehat\pi})-\cL_\mu^{W^{r^\star}}(\widehat\pi,r^\star,Q^{\widehat\pi})\bigr|\le 2H^2\sqrt{\varepsilon_\cF}.
  \end{equation}

  \emph{(3f) Weight-replacement bracket.} Since $|f_h|\le V_{\max}$,
  \[
    \bigl|\cL_\mu^{W^{\tilde r}}(\widehat\pi,\tilde r,f)-\cL_\mu^{W^r}(\widehat\pi,r,f)\bigr|\le 2V_{\max}\sum_{h=1}^H\EE_\mu[|\tilde W_h-W_h^r|].
  \]
  By the weight Lipschitz bound (\Cref{lemma:weight-lip}) and the trivial estimate $\sum_{h=1}^H\sum_{j<h}\le H\sum_{j=1}^H$, then Cauchy--Schwarz over $j$ followed by \cref{eq:assump-inv-L2},
  \[
    \sum_h\EE_\mu[|\tilde W_h-W_h^r|]\le LH\sum_j\EE_\mu[|\tilde r_j-r_j|]\le L\sqrt{\chi}\,H^{3/2}\sqrt{\cL_\mu^{\mathrm{BE}}(\widehat\pi,r,f)},
  \]
  hence
  \begin{equation}\label{eq:weight-replace-bound}
    \bigl|\cL_\mu^{W^{\tilde r}}(\widehat\pi,\tilde r,f)-\cL_\mu^{W^r}(\widehat\pi,r,f)\bigr|\le 2V_{\max}L\sqrt{\chi}\,H^{3/2}\sqrt{\cL_\mu^{\mathrm{BE}}(\widehat\pi,r,f)}.
  \end{equation}

  \emph{(3g) The $-\EE_\mu[G]$ piece.} By Cauchy--Schwarz, $|\EE_\mu[G]|\le\sqrt{\EE_\mu[G^2]}$. Substituting the bound on $\EE_\mu[G^2]$ from \Cref{lemma:G-L2-bound} and using $\sqrt{a+b}\le\sqrt a+\sqrt b$,
  \begin{equation}\label{eq:G-bound}
    |\EE_\mu[G]|\le\sqrt 2\sqrt{\cL_\mu^{\mathrm{RM}}(r)}+L(V_{\max}+1)\sqrt{2\chi H\,\cL_\mu^{\mathrm{BE}}(\widehat\pi,r,f)}.
  \end{equation}

  \emph{(3h) Combine.} Adding \cref{eq:alg-gap-bound}--\cref{eq:G-bound} (the sign of $-\EE_\mu[G]$ is absorbed into $|\EE_\mu[G]|$) yields
  \begin{equation}\label{eq:III-summary}
    (\mathrm{III}^\sharp)
    \le-\frac{\beta}{2}(\cL_\mu^{\mathrm{BE}}+\cL_\mu^{\mathrm{RM}})+\frac{\beta}{2}(\varepsilon_{\mathrm{BE}}+\varepsilon_{\mathrm{RM}}+4\varepsilon_{\cF,\cF})+\beta\varepsilon_{\mathrm{apx}}+2\sqrt{\varepsilon_{\mathrm{Perf}}}+2H^2\sqrt{\varepsilon_\cF}+\phi_{\mathrm{BE}}^{\mathrm{III}}\sqrt{\cL_\mu^{\mathrm{BE}}}+\phi_{\mathrm{RM}}^{\mathrm{III}}\sqrt{\cL_\mu^{\mathrm{RM}}},
  \end{equation}
  with $\phi_{\mathrm{BE}}^{\mathrm{III}}\coloneqq 2V_{\max}L\sqrt{\chi}\,H^{3/2}+L(V_{\max}+1)\sqrt{2\chi H}$ and $\phi_{\mathrm{RM}}^{\mathrm{III}}\coloneqq\sqrt 2$. All BE/RM losses are evaluated at $(\widehat\pi,r,f)$.

  \medskip
  \noindent\textbf{Step 4: combine and tune $\beta$.}
  Adding \cref{eq:I-summary} and \cref{eq:III-summary}, with all BE/RM losses evaluated at $(\widehat\pi,r,f)$,
  \begin{equation}\label{eq:I-III-combined}
    (\mathrm{I}^\sharp)+(\mathrm{III}^\sharp)
    \le-\frac{\beta}{2}(\cL_\mu^{\mathrm{BE}}+\cL_\mu^{\mathrm{RM}})+\frac{\beta}{2}\bar\varepsilon+2\sqrt{\varepsilon_{\mathrm{Perf}}}+2H^2\sqrt{\varepsilon_\cF}+\Phi_{\mathrm{BE}}\sqrt{\cL_\mu^{\mathrm{BE}}}+\Phi_{\mathrm{RM}}\sqrt{\cL_\mu^{\mathrm{RM}}},
  \end{equation}
  where $\bar\varepsilon\coloneqq\varepsilon_{\mathrm{BE}}+\varepsilon_{\mathrm{RM}}+4\varepsilon_{\cF,\cF}+2\varepsilon_{\mathrm{apx}}$ aggregates the slacks (mirroring $\varepsilon$ in \cref{eq:delta-k-final}) and
  \[
    \Phi_{\mathrm{BE}}\coloneqq\phi_{\mathrm{BE}}^{\mathrm{I}}+\phi_{\mathrm{BE}}^{\mathrm{III}},\qquad
    \Phi_{\mathrm{RM}}\coloneqq\phi_{\mathrm{RM}}^{\mathrm{I}}+\phi_{\mathrm{RM}}^{\mathrm{III}},\qquad
    \Phi^2\coloneqq\Phi_{\mathrm{BE}}^2+\Phi_{\mathrm{RM}}^2.
  \]
  With the coefficients from \cref{eq:I-summary} and \cref{eq:III-summary},
  \[
    \Phi
    =
    \widetilde O\bigl(
    L(V_{\max}+1)\sqrt{\kappa HC_{sa}(\pi^\star)}
    +V_{\max}L\sqrt{\chi H^3}
    +L(V_{\max}+1)\sqrt{\chi H}
    +\sqrt{HC_{sa}(\pi^\star)}
    +1
    \bigr).
  \]
  Under the usual normalization in which $H,V_{\max},L,\kappa,\chi,C_{sa}(\pi^\star)\ge1$, the $L(V_{\max}+1)\sqrt{\chi H}$ term is dominated by $V_{\max}L\sqrt{\chi H^3}$ and the constant term is dominated by $\sqrt{HC_{sa}(\pi^\star)}$, giving the simplified order for $\Phi$ stated in \Cref{thm:main-gen-outcome}.
  Young's inequality $X\sqrt Y\le X^2/(2\beta)+\beta Y/2$ cancels the negative BE/RM terms and leaves
  \[
    (\mathrm{I}^\sharp)+(\mathrm{III}^\sharp)\le\frac{\Phi^2}{2\beta}+\frac{\beta\bar\varepsilon}{2}+2\sqrt{\varepsilon_{\mathrm{Perf}}}+2H^2\sqrt{\varepsilon_\cF}.
  \]
  Choosing
  \[
    \beta
    =
    \frac{\Phi}{\sqrt{\bar\varepsilon}},
    \qquad
    \bar\varepsilon
    =
    \widetilde O\bigl(HV_{\max}^2/n+\varepsilon_\cF+\varepsilon_{\cF,\cF}\bigr).
  \]
  Hence, using $\sqrt{x+y+z}\le\sqrt x+\sqrt y+\sqrt z$,
  \begin{align}
    \frac{\Phi^2}{2\beta}+\frac{\beta\bar\varepsilon}{2}
     & =
    \Phi\sqrt{\bar\varepsilon} \notag \\
     & =
    \widetilde O\Bigl(
    \Phi V_{\max}\sqrt{\frac{H}{n}}
    +\Phi\sqrt{\varepsilon_\cF}
    +\Phi\sqrt{\varepsilon_{\cF,\cF}}
    \Bigr). \label{eq:gen-balanced-expanded}
  \end{align}
  Expanding the statistical part with the above bound on $\Phi$ gives
  \begin{align}
    \Phi V_{\max}\sqrt{\frac{H}{n}}
     & =
    \widetilde O\Bigl(
    V_{\max}^2L\sqrt{\frac{\kappa H^2C_{sa}(\pi^\star)}{n}}
    +V_{\max}^2L\sqrt{\frac{\chi H^4}{n}}
    +V_{\max}\sqrt{\frac{H^2C_{sa}(\pi^\star)}{n}}
    \Bigr). \label{eq:gen-stat-expanded}
  \end{align}
  The concentration remainder satisfies
  \[
    2\sqrt{\varepsilon_{\mathrm{Perf}}}
    =
    \widetilde O\!\left(V_{\max}H/\sqrt n\right).
  \]
  Under the same normalization, the last term in \cref{eq:gen-stat-expanded} and the concentration remainder are lower order than the two displayed statistical terms in \cref{eq:main-text-rate}. Therefore,
  \[
    (\mathrm{I}^\sharp)+(\mathrm{III}^\sharp)
    \le
    \widetilde O\!\left(
    V_{\max}^2L\sqrt{\frac{\kappa H^2C_{sa}(\pi^\star)}{n}}
    +V_{\max}^2L\sqrt{\frac{\chi H^4}{n}}
    +\Phi\sqrt{\varepsilon_{\cF,\cF}}
    +(\Phi+H^2)\sqrt{\varepsilon_\cF}
    \right).
  \]
  Averaging over $k$ and adding the no-regret contribution \cref{eq:II-summary} gives \cref{eq:main-text-rate}.
\end{proof}

\section{Examples of Generalized Objectives}\label{sec:examples}

This section illustrates the generalized objective framework of \Cref{section:main-result-generalized}. For a policy $\pi$ and per-step reward tuple $r=(r_1,\ldots,r_H)$, write the trajectory return as
\[
  R(\tau;r)=\sigma\bigl(r_1(s_1,a_1),\ldots,r_H(s_H,a_H)\bigr),
  \qquad
  J(\pi)=\EE_{\tau\sim\pi}[R(\tau;r^\star)].
\]
\Cref{thm:main-lb-allsuccess} shows that without structure on $\sigma$, outcome-based learning can require exponentially many trajectories. We therefore quantify tractability through two complementary complexity coefficients. First, because the trajectory outcome aggregates per-step rewards, we use the \emph{reward-process coefficient} $\kappa_\mu(\sigma)$ (\Cref{def:kappa}) to measure information loss from observing only a scalar return:
\[
  \kappa_\mu(\sigma)
  =
  \sup_{r\ne r'\in\cR}
  \frac{\sum_{h=1}^H\EE_\mu\bigl[(r_h-r'_h)^2(s_h,a_h)\bigr]}{\EE_{\tau\sim\mu}\bigl[(R(\tau;r)-R(\tau;r'))^2\bigr]}.
\]
Large $\kappa_\mu(\sigma)$ means distinct reward profiles are nearly indistinguishable from scalar outcomes alone.
Second, generalized objectives need not admit a standard Bellman dynamic-programming decomposition \citep{zhang2020variational,barakat2023reinforcement}. We therefore restrict attention to Bellman-learnable aggregations $\sigma_h(u,v)=a_h(u)v+b_h(u)$ and use the \emph{Bellman inverse coefficient} $\chi_\mu(\sigma)$ (\Cref{def:chi}) to measure how much one-step generalized Bellman targets can compress per-step reward differences:
\[
  \chi_\mu(\sigma)
  =
  \sup_{\pi\in\Pi,\,f\in\cF,\,r\ne r'\in\cR}
  \frac{\sum_{h=1}^H\EE_\mu\bigl[(r_h-r'_h)^2(s_h,a_h)\bigr]}{\sum_{h=1}^H\EE_\mu\bigl[\bigl((\cT_r^\pi f)_h-(\cT_{r'}^\pi f)_h\bigr)^2(s_h,a_h)\bigr]}.
\]

Under \Cref{assumption:sigma-regularity}, generalized OPAC satisfies the guarantee of \Cref{thm:main-gen-outcome}; informally, the dominant statistical terms scale as
\begin{align}
  \widetilde O\!\left(
  V_{\max}^2L\sqrt{\frac{\kappa_\mu(\sigma)H^2C_{sa}(\pi^\star)}{n}}
  +
  V_{\max}^2L\sqrt{\frac{\chi_\mu(\sigma)H^4}{n}}
  \right),\label{eq:gen-bound-appendix}
\end{align}
up to the no-regret term in $K$ and approximation errors.

In the following examples, we take each objective in turn and indicate when it aligns with the generalized Bellman framework behind \Cref{thm:main-gen-outcome}, as well as when sample-complexity blow-ups may arise: either from reward-process compression (large $\kappa_\mu$), or because the generalized Bellman operator is unavailable or compresses per-step information too aggressively (large $\chi_\mu$).

\paragraph{Cumulative return}
Take $a_h(u)=1$, $b_h(u)=u$, and $g=0$, then
$
  \sigma_h(u,v)=u+v,
  R(\tau)=\sum_{h=1}^H   r_h(s_h,a_h).
$
Here $J(\pi)=\EE_\pi[\sum_h r_h]$ is the standard cumulative return.
The generalized Bellman operator \cref{eq:def-gen-bellman} reduces to the usual Bellman backup, and \Cref{assumption:sigma-regularity} holds with  $L=1$.
Moreover, for every $\pi,f$ and every $h,(s,a)$,
\[
  (\cT_r^\pi f)_h(s,a)-(\cT_{r'}^\pi f)_h(s,a)=r_h(s,a)-r'_h(s,a),
\]
since the continuation term $\EE_{s'}[f_{h+1}(s',\pi)]$ cancels in the difference. Summing squares over $h$ shows that the numerator and denominator in \Cref{def:chi} coincide for every $(\pi,f,r,r')$, hence $\chi_\mu(\sigma)=1$.
For $\kappa_\mu(\sigma)$, one could follow the argument in \Cref{remark:kappa-enter-bound} to obtain an upper bound identical to setting $\kappa_\mu(\sigma)=1$ in \cref{eq:gen-bound-appendix}.

\paragraph{All-success aggregation}
For the all-success objective in \Cref{ex:allsuccess-outcome}, take $\sigma_h(u,v)=uv$ and $g=1$, so that $R(\tau;r)=\prod_{h=1}^H r_h(s_h,a_h)$.
This objective is Bellman-recursive and affine in the continuation value, but it can still be statistically hard because the multiplicative aggregation strongly attenuates per-step reward differences.
To see this concretely, consider the binary-tree construction underlying \Cref{thm:main-lb-allsuccess}: under the uniform behavior policy $\mu$, each root-to-leaf trajectory has probability $2^{-H}$, and for each $\btheta\in\{0,1\}^H$ define $r_{\btheta,h}(s_h,a)=\mathbf 1\{a=\theta_h\}$.
For any $\btheta\ne\btheta'$, let $S(\btheta,\btheta')=\{h:\theta_h\ne\theta_h'\}$ and $k=|S(\btheta,\btheta')|$.
Then the per-step reward profiles differ exactly on these $k$ stages, so
\[
  \sum_{h=1}^H\EE_\mu\bigl[(r_{\btheta,h}-r_{\btheta',h})^2(s_h,a_h)\bigr]=k.
\]
On the other hand, $R(\tau;r_\btheta)=1$ iff $a_h=\theta_h$ for all $h$, an event of probability $2^{-H}$ under $\mu$; the analogous event for $\btheta'$ is disjoint and also has probability $2^{-H}$.
Hence
\[
  \EE_\mu\bigl[(R(\tau;r_\btheta)-R(\tau;r_{\btheta'}))^2\bigr]
  =2\cdot 2^{-H}=2^{-(H-1)}.
\]
Therefore
\[
  \frac{
  \sum_{h=1}^H\EE_\mu[(r_{\btheta,h}-r_{\btheta',h})^2(s_h,a_h)]
  }{
  \EE_\mu[(R(\tau;r_\btheta)-R(\tau;r_{\btheta'}))^2]
  }
  = k\,2^{H-1},
\]
which is maximized at $k=H$, giving $\kappa_\mu(\sigma)\ge H2^{H-1}$ on this class.
Thus scalar all-success outcomes can hide per-step reward differences exponentially well, matching the exponential hardness in \Cref{thm:main-lb-allsuccess}.
A similar intuition applies to the Bellman inverse coefficient $\chi_\mu(\sigma)$.
For $\sigma_h(u,v)=uv$,
\[
  (\cT_r^\pi f)_h(s,a)-(\cT_{r'}^\pi f)_h(s,a)
  =
  (r_h-r_h')(s,a)\,
  \EE_{s'\sim P_h(\cdot\mid s,a)}[f_{h+1}(s',\pi)].
\]
Thus a reward difference at stage $h$ is multiplied by the continuation all-success value after $(s,a)$.
If this continuation value is uniformly lower bounded by $v_{\min}>0$, then $\chi_\mu(\sigma)\le v_{\min}^{-2}$; but in all-success problems this continuation value is typically the probability of completing all remaining stages, which may scale like $p_0^{H-h}$ when each remaining correct action is chosen with probability at least $p_0$.
Consequently $\chi_\mu(\sigma)$ can scale as large as $p_0^{-2H}$, and can even become infinite if some supported state-action pair has zero continuation success probability.
Therefore, all-success aggregation fits the generalized Bellman form, unlike max-reward RL, but its multiplicative structure can still induce exponential sample complexity through both $\kappa_\mu(\sigma)$ and $\chi_\mu(\sigma)$.

\paragraph{Conditional continuation (game-over)}
Take $a_h(u)=\mathbf 1\{u>0\}$, $b_h(u)=u$, and $g=0$:
$\sigma_h(u,v)=u+\mathbf 1\{u>0\}v,
    R(\tau)=\sum_{h=1}^{H_{\mathrm{stop}}} r_h(s_h,a_h)$,
where $H_{\mathrm{stop}}=\min\{h: r_h(s_h,a_h)\le 0\}\wedge H$.
Once a non-positive reward is received, all subsequent rewards are discarded.
This objective is well suited for per-step analysis, since observing the reward process directly reveals both the stopping time and the cause of failure.
However, it can be very hard for outcome-only learning because the discontinuous gate $a_h(u)=\mathbf 1\{u>0\}$ is not Lipschitz and can hide too much information about the reward process.
For example, when $H=10$, both reward sequences
\[
  (1,1,-1,\star,\ldots,\star)
  \qquad\text{and}\qquad
  (0.1,0.1,\ldots,0.1)
\]
produce the same scalar outcome $R(\tau)=1$.
The first trajectory has a catastrophic failure at stage $3$, after which all downstream rewards are discarded, while the second makes small positive progress at every stage.
Thus per-step feedback distinguishes these cases immediately, but the scalar outcome collapses them into the same observation.
Although the recursion is affine in the continuation value and contractive, the exact game-over objective lies outside \Cref{assumption:sigma-regularity} because of the hard, non-Lipschitz gate.
A smoothed continuation gate can satisfy the assumption, but the resulting complexity is governed by the conditioning of the smoothed gate and by how much information the scalar outcome preserves about the stopping event.

\paragraph{Exponential Utility}
Take $a_h(u)=e^{\lambda u}$ for a risk parameter $\lambda\in\RR$, $b_h(u)=0$, and $g=1$:
$\sigma_h(u,v)=e^{\lambda u}v,
  \qquad
  R(\tau)=\prod_{h=1}^H e^{\lambda r_h(s_h,a_h)}
  =e^{\lambda\sum_{h=1}^H r_h(s_h,a_h)}$.
The objective becomes
$J(\pi)=\EE_{\tau\sim\pi}\left[e^{\lambda\sum_h r_h}\right]$,
which is the moment generating function of the cumulative return evaluated at $\lambda$.
The quantity $\frac{1}{\lambda}\log J(\pi)$ is the \emph{entropic risk measure}, a standard criterion in risk-sensitive RL \citep{howard1972risk,wang2024reductions,han2025risk}.
The Bellman recursion is affine in the continuation value, and monotonicity in the continuation holds because $e^{\lambda u}>0$, so the generalized RL theory, including the Bellman equation and Bellman optimality equation in \Cref{sec:generalized_rl}, applies.
The contractive condition used in \Cref{thm:main-gen-outcome}, however, requires $e^{\lambda u}\le 1$ on the reward range, equivalently $\lambda u\le 0$.
Thus, risk-averse utilities with nonnegative costs, or risk-seeking utilities over nonpositive rewards, fit the contractive template after bounding the return; the opposite-sign regime is Bellman-recursive but not covered by the present theorem.
However, by carrying out the same calculation as in the all-success example, this exponential utility can also lead to exponential sample complexity.

\paragraph{Max-reward RL}
\citet{quah2006maximum,gottipati2020maximum,cui2023reinforcement,veviurko2402max} propose a paradigm in which an agent optimizes the maximum (or minimum) reward rather than the cumulative reward.
Consider $R^{\max}(\tau;r)=\max_{h\in[H]} r_h(s_h,a_h)$.
This objective can be represented recursively as $\sigma_h(u,v)=\max\{u,v\}$, but the recursion is not affine in $v$, so our generalized Bellman theory in \Cref{sec:generalized_rl} and \Cref{thm:main-gen-outcome} does not apply.
We can also illustrate a separate information-theoretic obstruction arising from reward-process aggregation.
Take a simple example: suppose that, for some finite-horizon MDP, every candidate reward satisfies $r_H(s_H,a_H)=1$ and $r_h(s_h,a_h)<1$ for all $h<H$ throughout the $\mu$-support.
Then $R^{\max}(\tau;r)=1$ holds $\mu$-almost surely for every $r\in\cR$, so $\EE_\mu[(R^{\max}(\tau;r)-R^{\max}(\tau;r'))^2]=0$ for all pairs $r,r'$, while earlier-stage rewards may still differ.
If $\cR$ contains such distinct pairs with
$\sum_{h=1}^H\EE_\mu[(r_h-r'_h)^2(s_h,a_h)]>0$,
then the denominator in \Cref{def:kappa} vanishes and $\kappa_\mu(\sigma)=+\infty$.
This means that trajectory-level observations make it impossible to recover per-step signals efficiently, and thus one cannot guarantee a finite-sample complexity analysis in offline RL under this objective.
Compared with our work, existing works mainly focus on convergence rather than finite-sample complexity analysis.

\section[The Underlying Statistical Problem: sigma-Composition Regression]{The Underlying Statistical Problem: $\sigma$-Composition Regression}
\label{app:sigma-regression}

This appendix records a self-contained regression abstraction: i.i.d.\ draws $x_i=(x_{i,1},\ldots,x_{i,H})\sim\mu$, scalar outcomes $Y_i$, and the reward-process coefficient $\kappa_\mu(\sigma)$ from \Cref{def:kappa}, read with $R(x;r)$ below in place of $R(\tau;r)$ and with $r_h$ evaluated at the coordinate~$x_h$. We argue that the complexity of learning per-step reward from trajectory outcomes is exactly what $\kappa_\mu(\sigma)$ captures.

\subsection{Regression formulation}
\label{app:subsec-sigma-regression}

Fix measurable spaces $\cX_1,\ldots,\cX_H$, a law $\mu$ on $\cX_1\times\cdots\times\cX_H$, a known $\sigma:\RR^H\to\RR$, and classes $\cR_h$ of measurable maps $r_h:\cX_h\to\RR$ with $\cR\subseteq\cR_1\times\cdots\times\cR_H$.
For $x=(x_1,\ldots,x_H)$ and $r=(r_1,\ldots,r_H)\in\cR$, define the composed outcome $R(x;r)\coloneqq\sigma\bigl(r_1(x_1),\ldots,r_H(x_H)\bigr)$.
Assume $|R(x;r)|\le V_{\max}$ for every $r\in\cR$ and $\mu$-a.e.\ $x$.
Let $r^\star\in\cR$ and suppose we observe
\[
  (x_i,Y_i),\qquad x_i\stackrel{\text{i.i.d.}}{\sim}\mu,\quad i\in[n],
\]
with $\EE[Y_i\mid x_i]=R(x_i;r^\star)$ and $|Y_i|\le V_{\max}$ almost surely.
Write
\[
  \|r-r'\|^2\coloneqq\sum_{h=1}^H\EE_\mu\bigl[(r_h(x_h)-r'_h(x_h))^2\bigr].
\]
The coefficient $\kappa_\mu(\sigma)$ is as in \Cref{def:kappa}; equivalently, with $\kappa=\kappa_\mu(\sigma)$, it is the smallest $\kappa\ge0$ such that
\[
  \|r-r'\|^2\le \kappa\,\EE_{x\sim\mu}\bigl[(R(x;r)-R(x;r'))^2\bigr]\quad\text{for all distinct }r,r'\in\cR,
\]
adopting the ratio conventions $a/0=+\infty$ for $a>0$ and $0/0=0$.
For distinct $r,r'\in\cR$ with $\EE_{x\sim\mu}\bigl[(R(x;r)-R(x;r'))^2\bigr]>0$, define the pair-specific ratio
\begin{equation}\label{eq:kappa-pair}
  \kappa_{r,r'}(\sigma)\coloneqq \frac{\|r-r'\|^2}{\EE_{x\sim\mu}\bigl[(R(x;r)-R(x;r'))^2\bigr]},
\end{equation}
so $\kappa_\mu(\sigma)=\sup_{r\ne r'}\kappa_{r,r'}(\sigma)$ under the same conventions.
The ERM estimator is
\begin{equation}\label{eq:erm-appendix}
  \hat r\in\argmin_{r\in\cR}\frac{1}{n}\sum_{i=1}^n\bigl(Y_i-R(x_i;r)\bigr)^2.
\end{equation}

\subsection[Minimax Rate for sigma-Composition Regression]{Minimax Rate for $\sigma$-Composition Regression}

The next theorem records a high-probability upper bound for the empirical risk minimizer from~\cref{eq:erm-appendix}.
\begin{theorem}[\texorpdfstring{Upper bound for $\sigma$-Composition Regression}{Upper bound for sigma-Composition Regression}]
  \label{thm:sigma-regression-upper}
  Suppose $\cR$ is finite, $r^\star\in\cR$, and $(x_i,Y_i)_{i=1}^n$ are independent with $x_i\sim\mu$ and
  \[
    \EE[Y_i\mid x_i]=R(x_i;r^\star).
  \]
  Assume moreover that
  \[
    |R(x;r)|\le V_{\max}\quad\text{for all }r\in\cR\text{ and }\mu\text{-a.e. }x,
    \qquad
    |Y_i|\le V_{\max}\quad\text{a.s.}
  \]
  Then there exists a universal constant $C>0$ such that, with probability at least $1-\delta$, any empirical risk minimizer $\hat r$ from~\cref{eq:erm-appendix} satisfies
  \[
    \|\hat r-r^\star\|^2
    \le
    C\,\kappa_\mu(\sigma)\,
    \frac{V_{\max}^2\log(|\cR|/\delta)}{n}.
  \]
\end{theorem}
The proof of \Cref{thm:sigma-regression-upper} is given in \Cref{proof:sigma-regression-upper}.
For a minimax lower bound, we isolate the hardness coming from compressing per-step structure into a single trajectory outcome: we take two candidates $r,r'$ at $\|\cdot\|$-separation.
Under i.i.d.\ Gaussian outcome noise at scale~$V_{\max}$, \Cref{thm:sigma-regression-lower} shows that $\Omega\bigl(V_{\max}^2\kappa_{r,r'}(\sigma)/\Delta^2\bigr)$ samples can be necessary even for distinguishing this pair.

\begin{theorem}[\texorpdfstring{Lower bound for $\sigma$-Composition Regression}{Lower bound for sigma-Composition Regression}]
  \label{thm:sigma-regression-lower}
  Suppose $\cD=\{(x_i,Y_i)\}_{i=1}^n$ with $x_i\stackrel{\text{i.i.d.}}{\sim}\mu$ and
  \[
    Y_i=R(x_i;r^\star)+\xi_i,
    \qquad
    \xi_1,\ldots,\xi_n\stackrel{\text{i.i.d.}}{\sim}\cN(0,V_{\max}^2),
  \]
  with $\xi_i$ independent of $x_i$.
  Fix distinct $r,r'\in\cR$ with $\|r-r'\|^2=\Delta^2>0$ and finite $\kappa_{r,r'}(\sigma)$ as in~\cref{eq:kappa-pair}.
  Then whenever $n\le V_{\max}^2\kappa_{r,r'}(\sigma)/\Delta^2$,
  \[
    \inf_{\hat r}\max_{r^\star\in\{r,r'\}}\EE\bigl[\|\hat r-r^\star\|^2\bigr]
    \ge
    \frac{\Delta^2}{16}.
  \]
  In particular, for any $\varepsilon<\Delta^2/16$, every estimator $\hat r$ with $\max_{r^\star\in\{r,r'\}}\EE[\|\hat r-r^\star\|^2]\le\varepsilon$ requires $n=\Omega\bigl(V_{\max}^2\kappa_{r,r'}(\sigma)/\Delta^2\bigr)$.
\end{theorem}

The proof of \Cref{thm:sigma-regression-lower} is given in \Cref{proof:sigma-regression-lower}.
Together, \Cref{thm:sigma-regression-upper,thm:sigma-regression-lower} make the role of $\kappa$ explicit in $\sigma$-composition regression: outcome aggregation compresses per-step differences into $R(x;r)$, so controlling $\|\hat r-r^\star\|^2$ through outcome residuals incurs an extra factor of order $\kappa$ in the minimax rate.

\subsection[Proof of theorem: upper bound for sigma-composition regression]{Proof of \Cref{thm:sigma-regression-upper}}\label{proof:sigma-regression-upper}
For $r\in\cR$ set $g_r(x)\coloneqq R(x;r)$ and $g_\star(x)\coloneqq R(x;r^\star)$.
Let $(x,Y)$ be an independent copy from the same distribution as each sample pair, so $x\sim\mu$ and $\EE[Y\mid x]=g_\star(x)$.
Let $\cD=\{(x_i,Y_i)\}_{i=1}^n$ and write
\[
  \EE_{\cD}[\phi]\coloneqq \frac{1}{n}\sum_{i=1}^n \phi(x_i,Y_i)
\]
for measurable $\phi$.
Define the squared loss $\ell_r(x,Y)\coloneqq (Y-g_r(x))^2$ and the excess loss
\[
  Z_r(x,Y)\coloneqq \ell_r(x,Y)-\ell_{r^\star}(x,Y).
\]
Then
\begin{align*}
  \EE[Z_r(x,Y)]
   & =
  \EE\bigl[(Y-g_r(x))^2-(Y-g_\star(x))^2\bigr] \\
   & =
  \EE_{x\sim\mu}\bigl[(g_r(x)-g_\star(x))^2\bigr]
  \eqqcolon L_r .
\end{align*}
In particular, $L_r$ is the population squared prediction error of $g_r$ under~$\mu$.

Fix $r\in\cR$. Expanding the square gives
\[
  Z_r(x,Y)
  =
  (Y-g_r(x))^2-(Y-g_\star(x))^2
  =
  (g_\star(x)-g_r(x))\bigl(2Y-g_r(x)-g_\star(x)\bigr).
\]
Since $|Y|,|g_r(x)|,|g_\star(x)|\le V_{\max}$ almost surely,
\[
  |Z_r(x,Y)|\le 4V_{\max}^2
\]
and
\[
  Z_r(x,Y)^2
  \le
  16V_{\max}^2(g_r(x)-g_\star(x))^2.
\]
Therefore
\[
  \Var(Z_r(x,Y))
  \le
  \EE[Z_r(x,Y)^2]
  \le
  16V_{\max}^2 L_r .
\]
Also $L_r=\EE_{x\sim\mu}[(g_r(x)-g_\star(x))^2]\le 4V_{\max}^2$, hence $|Z_r(x,Y)-L_r|\le 8V_{\max}^2$ almost surely.
Apply one-sided Bernstein to the i.i.d.\ variables $Z_r(x_1,Y_1),\ldots,Z_r(x_n,Y_n)$: for any fixed $r\in\cR$ and any $\eta\in(0,1)$, with probability at least $1-\eta$,
\[
  \EE_{\cD}[Z_r]-L_r
  \ge
  -\sqrt{\frac{32V_{\max}^2L_r\log(1/\eta)}{n}}
  -\frac{16V_{\max}^2\log(1/\eta)}{3n}.
\]
Set $\eta=\delta/|\cR|$ and take a union bound over $r\in\cR$.
With probability at least $1-\delta$, the preceding display holds simultaneously for every $r\in\cR$.
On this event, if $\EE_{\cD}[Z_r]\le 0$ then
\[
  L_r
  \le
  \sqrt{\frac{32V_{\max}^2L_r\log(|\cR|/\delta)}{n}}
  +\frac{16V_{\max}^2\log(|\cR|/\delta)}{3n}.
\]
Using $\sqrt{32V_{\max}^2L_r\log(|\cR|/\delta)/n}\le L_r/2+16V_{\max}^2\log(|\cR|/\delta)/n$ and rearranging yields
\[
  L_r
  \le
  C\frac{V_{\max}^2\log(|\cR|/\delta)}{n}
\]
for a universal constant $C>0$.

If $\hat r$ is an empirical risk minimizer and $r^\star\in\cR$, then $\EE_{\cD}[Z_{\hat r}]\le 0$.

On this same event,
\[
  \EE_{x\sim\mu}\bigl[(R(x;\hat r)-R(x;r^\star))^2\bigr]
  =
  L_{\hat r}
  \le
  C\frac{V_{\max}^2\log(|\cR|/\delta)}{n}.
\]

Finally, by \Cref{def:kappa},
\[
  \|\hat r-r^\star\|^2
  \le
  \kappa_\mu(\sigma)\,
  \EE_{x\sim\mu}\bigl[(R(x;\hat r)-R(x;r^\star))^2\bigr].
\]
Combining the last two displays completes the proof.

\subsection[Proof of theorem: lower bound for sigma-composition regression]{Proof of \Cref{thm:sigma-regression-lower}}\label{proof:sigma-regression-lower}

For $r^\star\in\{r,r'\}$, let $\PP_{r^\star}$ be the law of one draw $(x,Y)$ under that truth, and let $\EE_{r^\star}$ denote expectation when $\cD$ consists of $n$ i.i.d.\ draws from $\PP_{r^\star}$.
Write $\delta_R^2:=\EE_{x\sim\mu}\bigl[(R(x;r)-R(x;r'))^2\bigr]=\Delta^2/\kappa_{r,r'}(\sigma)$.
Conditional on $x$, the two models differ only in the mean of a Gaussian with variance $V_{\max}^2$, so
\[
  \KL(\PP_r\|\PP_{r'})
  =
  \frac{\delta_R^2}{2V_{\max}^2}.
\]
By tensorization, $\KL(\PP_r^{\otimes n}\|\PP_{r'}^{\otimes n})=n\,\KL(\PP_r\|\PP_{r'})$, and Pinsker's inequality gives
\[
  \mathrm{TV}\bigl(\PP_r^{\otimes n},\PP_{r'}^{\otimes n}\bigr)
  \le
  \sqrt{\frac{n}{2}\,\KL(\PP_r^{\otimes n}\|\PP_{r'}^{\otimes n})}
  =
  \sqrt{\frac{n\delta_R^2}{4V_{\max}^2}}
  \le
  \frac12
\]
when $n\le V_{\max}^2/\delta_R^2$.

The pair $(r,r')$ is $\Delta$-separated in $\|\cdot\|$ because $\|r-r'\|=\Delta$.
By the two-point Le Cam method \citep{yu1997assouad}, when $n\le V_{\max}^2/\delta_R^2$,
\[
  \inf_{\hat r}\max_{r^\star\in\{r,r'\}}\EE_{r^\star}\bigl[\|\hat r-r^\star\|^2\bigr]
  \ge
  \frac{\Delta^2}{8}\Bigl(1-\mathrm{TV}\bigl(\PP_r^{\otimes n},\PP_{r'}^{\otimes n}\bigr)\Bigr)
  \ge
  \frac{\Delta^2}{16}.
\]

\end{document}